\documentclass[letterpaper,11pt]{article}
\usepackage{etex}
\usepackage[margin=1in]{geometry}

\usepackage[dvipsnames,table,xcdraw]{xcolor}
\usepackage{amssymb,amsfonts,amsmath,amstext,amsthm,mathrsfs}
\usepackage{graphics,latexsym,epsfig,psfrag,wrapfig,comment,paralist}
\usepackage{xspace}
\usepackage{subcaption,bm,multirow}
\usepackage{booktabs}
\usepackage{mathdots}
\usepackage{bbold}
\usepackage{centernot}
\usepackage{algorithm}
\usepackage{algorithmicx}
\usepackage[noend]{algpseudocode}
\usepackage{prettyref}
\usepackage{listings}
\usepackage{tikz}
\usetikzlibrary{decorations,calligraphy}
\usepackage{pgfplots}
\usetikzlibrary{decorations.pathreplacing}
\usetikzlibrary{shapes}
\usetikzlibrary{positioning, arrows, matrix, calc, patterns, fit}
\usepackage{caption}
\usepackage{array}
\usepackage{mdwmath}
\usepackage{multirow}
\usepackage{mdwtab}
\usepackage{eqparbox}
\usepackage{multicol}
\usepackage{amsfonts}
\usepackage{multirow,bigstrut,threeparttable}
\usepackage{amsthm}
\usepackage{array}
\usepackage{bbm}
\usepackage{epstopdf}
\usepackage{mdwmath}
\usepackage{mdwtab}
\usepackage{eqparbox}
\usepackage{tikz}
\usepackage{latexsym}
\usepackage{amssymb}
\usepackage{bm}
\usepackage{graphicx}
\usepackage{mathrsfs}
\usepackage{epsfig}
\usepackage{psfrag}
\usepackage{setspace}
\definecolor{linkColor}{HTML}{E74C3C}
\definecolor{pearcomp}{HTML}{B97E29}
\definecolor{citeColor}{HTML}{2980B9}
\definecolor{urlColor}{HTML}{1D2DEC}
\definecolor{conjColor}{HTML}{9ab569}
\usepackage[CJKbookmarks=true,
            bookmarksnumbered=true,
            bookmarksopen=true,
            colorlinks=true,
            citecolor=citeColor,
            linkcolor=linkColor,
            anchorcolor=red,
            urlcolor=urlColor,
            ]{hyperref}
\usepackage{natbib}
\setcitestyle{authoryear,round}
\newtheoremstyle{break}
  {\topsep}{\topsep}%
  {\itshape}{}%
  {\bfseries}{}%
  {\newline}{}%

\newtheorem{question}{Question}

\usepackage{pgfplots}
\usepackage{amsmath,amssymb}
\usepackage{mathtools}
\usepackage{stmaryrd}
\pgfmathdeclarefunction{gauss}{3}{%
  \pgfmathparse{#3/(#2*sqrt(2*pi))*exp(-((x-#1)^2)/(2*#2^2))}%
}
\pgfmathdeclarefunction{mingauss}{4}{%
  \pgfmathparse{#4/(#3*sqrt(2*pi))*exp(-max((x-#1)^2, (x-#2)^2)/(2*#3^2))}%
}

\tikzset{
  invisible/.style={opacity=0},
  visible on/.style={alt={#1{}{invisible}}},
  alt/.code args={<#1>#2#3}{%
    \alt<#1>{\pgfkeysalso{#2}}{\pgfkeysalso{#3}}
  },
}

\newtheorem{definition}{\textbf{Definition}}%[section]

\newtheorem{lemma}{\textbf{Lemma}}%[section]
\newtheorem{theorem}{\textbf{Theorem}}%[section]

\newtheorem*{insight*}{\textbf{Observation}}

\newtheorem{prop}{\textbf{Proposition}}

\newtheorem*{lemmai*}{\textbf{Lemma (informal)}}
\newtheorem{remark}{\textbf{Remark}}

\newtheorem{conjecture}{Conjecture}

\newcommand{\cX}{\mathcal{X}}

\newcommand{\cL}{\mathcal{L}}

\newcommand{\cS}{\mathcal{S}}
\newcommand{\cT}{\mathcal{T}}
\newcommand{\cC}{\mathcal{C}}

\newcommand{\cR}{\mathcal{R}}
\newcommand{\cP}{\mathcal{P}}

\newcommand{\cM}{\mathsf{MDP}}
\newcommand{\cA}{\mathcal{A}}
\newcommand{\cB}{\mathcal{B}}

\newcommand{\cE}{\mathcal{E}}

\newcommand{\cD}{\mathcal{D}}

\newcommand{\cO}{\mathcal{O}}

\usepackage[normalem]{ulem}

\renewcommand{\cite}[1]{\citep{#1}}

\usepackage{color}
\definecolor{cm}{RGB}{0,0,200}
\definecolor{purple}{RGB}{200,0,200}

\usepackage[intoc]{nomencl}
\makenomenclature

\makeatletter
\newcommand{\vast}{\bBigg@{2.5}}
\newcommand{\Vast}{\bBigg@{5}}
\makeatother

\DeclareMathOperator{\E}{\mathbb{E}}
\DeclareMathOperator*{\R}{\mathbb{R}}
\DeclareMathOperator{\prob}{\mathbb{P}}
\DeclareMathOperator*{\argmax}{arg\,max}

\usepackage{bbm}
\DeclareMathOperator*{\ind}{\mathbbm{1}}

\usepackage[utf8]{inputenc}
\usepackage[LGR,T1]{fontenc}

\DeclareMathOperator{\Binomial}{Binomial}
\usepackage{breqn}
\usepackage{enumitem,kantlipsum}
\newcommand\blfootnote[1]{%
  \begingroup
  \renewcommand\thefootnote{}\footnote{#1}%
  \addtocounter{footnote}{-1}%
  \endgroup
}
\title{Bridging Offline Reinforcement Learning and Imitation Learning: \\A Tale of Pessimism\footnote{Part of the paper has been published at Neurips 2021.}}
\author{Paria Rashidinejad$^\dagger$ \quad
Banghua Zhu$^\dagger$ \quad
Cong Ma$^{\diamond}$ \quad
Jiantao Jiao$^{\dagger, \ddagger}$ \quad
Stuart Russell$^{\dagger}$ \blfootnote{Emails: \texttt{\{paria.rashidinejad,banghua,jiantao,russell\}@berkeley.edu, congm@uchicago.edu}}\\ { }\\
$^\dagger$ Department of Electrical Engineering and Computer Sciences, UC Berkeley\\
$^\ddagger$ Department of Statistics, UC Berkeley\\ { } \\
$\diamond$ Department of Statistics, University of Chicago
}
\date{\today}

\begin{document}

\maketitle

\begin{abstract}
Offline (or batch) reinforcement learning (RL) algorithms seek to learn an optimal policy from a fixed dataset without active data collection. Based on the {composition} of the offline dataset, two main categories of methods are used: imitation learning which is suitable for expert datasets and vanilla offline RL which often requires uniform coverage datasets. From a practical standpoint, datasets often deviate from these two extremes and the exact data composition is usually unknown a priori. To bridge this gap, we present a new offline RL framework that smoothly interpolates between the two extremes of data composition, hence unifying imitation learning and vanilla offline RL. The new framework is centered around a weak version of the concentrability coefficient that measures the deviation from the behavior policy to the expert policy alone. 

Under this new framework, we further investigate the question on algorithm design: can one develop an algorithm that achieves a minimax optimal rate and also {adapts to unknown data composition}? To address this question, we consider a lower confidence bound (LCB) algorithm developed based on pessimism in the face of uncertainty in offline RL. We study finite-sample properties of LCB as well as information-theoretic limits in three settings: multi-armed bandits, contextual bandits, and Markov decision processes (MDPs). Our analysis reveals surprising facts about optimality rates. In particular, in both contextual bandits and RL, LCB achieves a faster rate of $1/N$ for nearly-expert datasets compared to the usual rate of $1/\sqrt{N}$ in offline RL, where $N$ is the number of samples in the batch dataset. In the case of contextual bandits with at least two contexts, we prove that LCB is adaptively optimal for the entire data composition range, achieving a smooth transition from imitation learning to offline RL. We further show that LCB is {almost} adaptively optimal in MDPs.
\end{abstract}

{
  \hypersetup{linkcolor=black}
  \tableofcontents
}

\section{Introduction}

Reinforcement learning (RL) algorithms have recently achieved tremendous empirical success including beating Go champions \cite{silver2016mastering,silver2017mastering} and surpassing professionals in Atari games \cite{mnih2013playing,mnih2015human}, to name a few. Most success stories, however, are in the realm of online RL in which active data collection is necessary. This online paradigm falls short of leveraging previously-collected datasets and dealing with scenarios where online exploration is not possible~\cite{fu2020d4rl}. To tackle these issues, offline (or batch) reinforcement learning \cite{lange2012batch,levine2020offline} arises in which the agent aims at achieving competence by exploiting a batch dataset without access to online exploration. This paradigm is useful in a diverse array of application domains such as healthcare \cite{wang2018supervised,gottesman2019guidelines, nie2020learning}, autonomous driving \cite{yurtsever2020survey,bojarski2016end,pan2017agile}, and recommendation systems \cite{strehl2010learning,garcin2014offline,thomas2017predictive}.

The key component of offline RL is a pre-collected dataset from an unknown stochastic environment. Broadly speaking, there exist two types of \emph{data composition} for which offline RL algorithms have shown promising empirical and theoretical success; see Figure~\ref{fig:spectrum_offline_RL} for an illustration.

\begin{itemize}[leftmargin=*]
  \item \textbf{Expert data.} One end of the spectrum includes datasets collected by following an expert policy. For such datasets, imitation learning algorithms (e.g., behavior cloning~\cite{ross2010efficient}) are shown to be effective in achieving a small sub-optimality competing with the expert policy. In particular, it is recently shown in the work~\citet{rajaraman2020toward} that the behavior cloning algorithm achieves the minimal sub-optimality $1/N$ in episodic Markov decision processes, where $N$ is the total number of samples in the expert dataset. 
  \item \textbf{Uniform coverage data.} On the other end of the spectrum lies the datasets with uniform coverage. More specifically, such datasets are collected with an aim to cover \textit{all} states and actions, even the states never visited or actions never taken by satisfactory policies. Most vanilla offline RL algorithms are only suited in this region and are shown to diverge for \textit{narrower} datasets \cite{fu2020d4rl, koh2020wilds}, such as those collected via human demonstrations or hand-crafted policies, both empirically \cite{fujimoto2019off,kumar2019stabilizing} and theoretically \cite{agarwal2020optimality, du2020good}. In this regime, a widely-adopted requirement is the \emph{uniformly bounded concentrability coefficient} which assumes that the ratio of the state-action occupancy density induced by \textit{any policy} and the data distribution is bounded uniformly over all states and actions \cite{munos2007performance,farahmand2010error,chen2019information,xie2020batch}. Another common assumption is uniformly lower bounded data distribution on all states and actions~\cite{sidford2018near, agarwal2020model}, which ensures all states and actions are visited with sufficient probabilities. Algorithms developed for this regime 
  are demonstrated to achieve a $1/\sqrt{N}$ sub-optimality competing with the optimal policy; see for example the papers~\citet{yin2020near,hao2020sparse,uehara2021finite}.
\end{itemize}

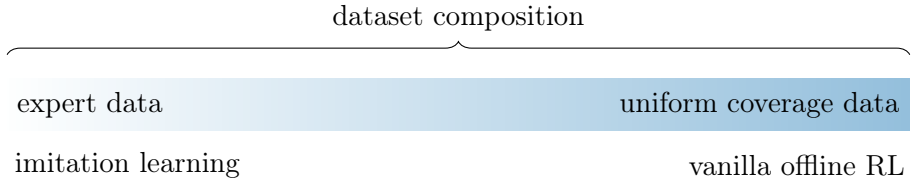
\begin{figure}[t]
    \centering
    \scalebox{1}{
    \begin{tikzpicture}
    \node[rectangle, draw=none, fill, left color=white, right color=citeColor!50, minimum width=12cm, minimum height=0.7cm] at (0,0) { };
    \draw [decorate,decoration={brace,amplitude=5pt,raise=4ex}]
    (-6,0) -- (6,0) node[midway,yshift=3em]{dataset composition};
    \node[text = black] (n1) at (-4.9,0) {expert data};
    \node[text = black] (n1) at (-4.4,-0.8) {imitation learning};
    \node[text = black] (n2) at (4,0) {uniform coverage data};
    \node[text = black] (n1) at (4.5,-0.8) {vanilla offline RL};
    \end{tikzpicture}}
    \caption{Dataset composition range for offline RL problems. On one end, we have expert data for which imitation learning algorithms are well-suited. On the other end, we have uniform exploratory data for which vanilla offline RL algorithms can be used.}
    \label{fig:spectrum_offline_RL}
\end{figure}

\subsection{Motivating questions}
Clearly, both of these two extremes impose strong assumptions on the dataset: at one extreme, we hope for a solely expert-driven dataset; at the other extreme, we require the dataset to cover every, even sub-optimal, actions. In practice, there are numerous scenarios where the dataset deviates from these two extremes, which has motivated the development of new offline RL benchmark datasets with different data compositions~\cite{fu2020d4rl,koh2020wilds}. With this need in mind, the first and foremost question is regarding offline RL formulations: 

\begin{question}[{\textbf{Formulation}}]
  {Can we propose an offline RL framework that accommodates the entire data composition range?}
\end{question}

We answer this question affirmatively by proposing a new formulation for offline RL that smoothly interpolates between two regimes: expert data and data with uniform coverage. More specifically, we characterize the data composition in terms of the ratio between the state-action occupancy density of an optimal policy\footnote{In fact, our developments can accommodate arbitrary competing policies, however, we restrict ourselves to the optimal policy for ease of presentation.} and that of the behavior distribution which we denote by $C^\star$; see Definition~\ref{def:concentrability} for a precise formulation.
In words, $C^\star$ can be viewed as a measure of the deviation between the behavior distribution and the distribution induced by the optimal policy. The case with $C^\star = 1$ recovers the setting with expert data since, by the definition of $C^\star$, the behavior policy is identical to the optimal policy. In contrast, when $C^\star > 1$, the dataset is no longer purely expert-driven: it could contain ``spurious'' samples---states and actions that are not visited by the optimal policy. As a further example, when the dataset has uniform coverage, say the behavior probability is lower bounded by $\mu_{\min}$ over all states and actions, it is straightforward to check that the new concentrability coefficient is also upper bounded by $\mu_{\min}^{-1}$.

Assuming a finite $C^\star$ is the weakest concentrability requirement \cite{scherrer2014approximate,geist2017bellman,xie2020batch} that is currently enjoyed only by some online algorithms such as CPI \cite{kakade2002approximately}. $C^\star$ imposes a much weaker assumption in contrast to other concentrability requirements which involve taking a maximum over all policies; see \citet{scherrer2014approximate} for a hierarchy of different concentrability definitions. 
We would like to immediately point out that existing works on offline RL either do not specify the dependency of sub-optimality on data coverage \cite{jin2020pessimism, yu2020mopo}, or  do not have a batch data coverage assumption that accommodates the entire data spectrum including the expert datasets \cite{yin2021near,kidambi2020morel}.

\medskip 
With this formulation in mind, a natural next step is designing offline RL algorithms that handle various data compositions, i.e., for all $C^\star \geq 1$. Recently, efforts have been made toward reducing the offline dataset requirements based on a shared intuition: the agent should act conservatively and avoid states and actions less covered in the offline dataset. Based on this intuition, a variety of offline RL algorithms are proposed that achieve promising empirical results. Examples include model-based methods that learn pessimistic MDPs \cite{yu2020mopo,kidambi2020morel, yu2021combo}, model-free methods that reduce the Q-functions on unseen state-action pairs \cite{liu2020provably, kumar2020conservative,agarwal2020optimistic}, and policy-based methods that minimize the divergence between the learned policy and the behavior policy \cite{kumar2019stabilizing,nachum2020reinforcement, fujimoto2019off, nadjahi2019safe,laroche2019safe,peng2019advantage,siegel2020keep,ghasemipour2020emaq}.

However, it is observed empirically that existing policy-based methods perform better when the dataset is nearly expert-driven (toward the left of data spectrum in Figure~\ref{fig:spectrum_offline_RL}) whereas existing model-based methods perform better when the dataset is randomly-collected (toward the right of data spectrum in Figure~\ref{fig:spectrum_offline_RL}) \cite{yu2020mopo, buckman2020importance}. It remains unclear whether a single algorithm exists that performs well regardless of data composition---an important challenge from a practical perspective~\cite{neurips_tutorial, fu2020d4rl, koh2020wilds}. More importantly, the knowledge of the dataset composition may not be available a priori to assist in selecting the right algorithm.
This motivates the second question on the algorithm design:

\begin{question}[{\textbf{Adaptive algorithm design}}]
    {Can we design algorithms that can achieve minimal sub-optimality when facing different dataset compositions (i.e., different $C^\star$)? Furthermore, can this be achieved in an adaptive manner, i.e., without knowing $C^\star$ beforehand?}
\end{question}

To answer the second question, we analyze a \textit{pessimistic} variant of a value-based method in which we first form a lower confidence bound (LCB) for the value function of a policy using the batch data and then seek to find a policy that maximizes the LCB. A similar algorithm design has appeared in the recent work~\citet{jin2020pessimism}. It turns out that such a simple algorithm---fully agnostic to the data composition---is able to achieve  \emph{almost} optimal performance in multi-armed bandits and Markov decision processes, and optimally solve the offline learning problem in contextual bandits. See the section below for a summary of our theoretical results.

\begin{table}[!b]
\centering
\caption{A summary of our theoretical results with all the log factors ignored.} 
\label{tab:results_summary}
\scalebox{0.9}{
\begin{tabular}{@{} l c c l @{}}
\toprule
\textbf{Multi-armed bandits} &
  $C^\star \in [1,2)$ &
  \multicolumn{2}{c}{$C^\star \in [2,\infty)$} \\ [0.5ex] \hline \hline \\ [-0.5ex]
Algorithm \ref{alg:MAB-LCB} (MAB-LCB) sub-optimality &
  \multirow{2}{*}{$\sqrt{\frac{C^\star}{N}}$} &
  \multicolumn{2}{c}{\multirow{2}{*}{$\sqrt{\frac{C^\star}{N}}$}} \\[0.5ex]
(Theorem \ref{thm:MAB_LCB_upper}) &
   &
  \multicolumn{2}{c}{} \\[0.5ex]
Information-theoretic lower bound &
  \multirow{2}{*}{{$\exp\left(-(2-C^\star)\cdot\log\left(\frac{2}{C^\star-1}\right)\cdot N \right)$}} &
  \multicolumn{2}{c}{\multirow{2}{*}{$\sqrt{\frac{C^\star}{N}}$}} \\[0.5ex]
(Theorem \ref{thm.bandit_lower_bound}) &
   &
  \multicolumn{2}{c}{} \\[0.5ex]
Most played arm &
  \multirow{2}{*}{$\exp\left(-N\cdot\mathsf{KL}\left(\mathrm{Bern}\left(\frac{1}{2}\right) \| \mathrm{Bern}\left(\frac{1}{C^\star}\right)\right)\right)$} &
  \multicolumn{2}{c}{\multirow{2}{*}{N/A}} \\[0.5ex]
(Proposition \ref{thm.imitation_bandit_bound}) &
   &
  \multicolumn{2}{c}{} \\ \toprule
\textbf{Contextual bandits} &
  \multicolumn{3}{c}{$C^\star \in [1, \infty)$} \\[0.5ex] \hline \hline \\ [-0.5ex]
Algorithm \ref{alg:CB-LCB} (CB-LCB) sub-optimality &
  \multicolumn{3}{c}{\multirow{2}{*}{$\sqrt{\frac{S(C^\star-1)}{N}} + \frac{S}{N}$}} \\[0.5ex]
(Theorem \ref{thm:LCB_CB_upper_bound}) &
  \multicolumn{3}{c}{} \\[0.5ex]
Information-theoretic lower bound &
  \multicolumn{3}{c}{\multirow{2}{*}{$\sqrt{\frac{S(C^\star-1)}{N}} + \frac{S}{N}$}} \\[0.5ex]
(Theorem \ref{theorem:lower_bound_offline_CB}) &
  \multicolumn{3}{c}{} \\ \toprule
\textbf{Markov decision processes} &
  $C^\star \in [1, 1+ 1/N )$ &
  \multicolumn{2}{c}{$C^\star \in [ 1+ 1/N, \infty )$} \\[0.5ex] \hline \hline \\ [-1.5ex]
Algorithm \ref{alg:OVI-LCB-DS} (VI-LCB) sub-optimality &
  {\multirow{2}{*}{$ \frac{S}{(1-\gamma)^4N}$}} &
  \multicolumn{2}{l}{\multirow{2}{*}{$\sqrt{\frac{SC^\star}{(1-\gamma)^5 N}} $}} \\[0.5ex]
(Theorem \ref{thm:MDP_upperbound_hoffding}) &
  \multicolumn{1}{l}{} &
  \multicolumn{2}{l}{} \\[0.5ex]
Information-theoretic lower bound &
  \multirow{2}{*}{$\sqrt{\frac{S(C^\star-1)}{(1-\gamma)^3N}} + \frac{S}{(1-\gamma)^2N} $} &
  \multicolumn{2}{l}{\multirow{2}{*}{$ \sqrt{\frac{S(C^\star-1)}{(1-\gamma)^3N}} + \frac{S}{(1-\gamma)^2N}$}} \\[0.5ex]
(Theorem \ref{thm:MDP_lower}) &
   &
  \multicolumn{2}{l}{} \\ \bottomrule
\end{tabular}}
\end{table}

\subsection{Main results}\label{sec:results_summary}
In this subsection, we give a preview of our theoretical results; see Table~\ref{tab:results_summary} for a summary.
Under the new framework defined via $C^\star$, we instantiate the LCB approach to three different decision-making problems with increasing complexity: 
(1) multi-armed bandits, (2) contextual bandits, and (3)  infinite-horizon discounted Markov decision processes. 
We will divide our discussions on the main results accordingly. 
Throughout the discussion, $N$ denotes the number of samples in the batch data, $S$ denotes the number of states, and we ignore the log factors.

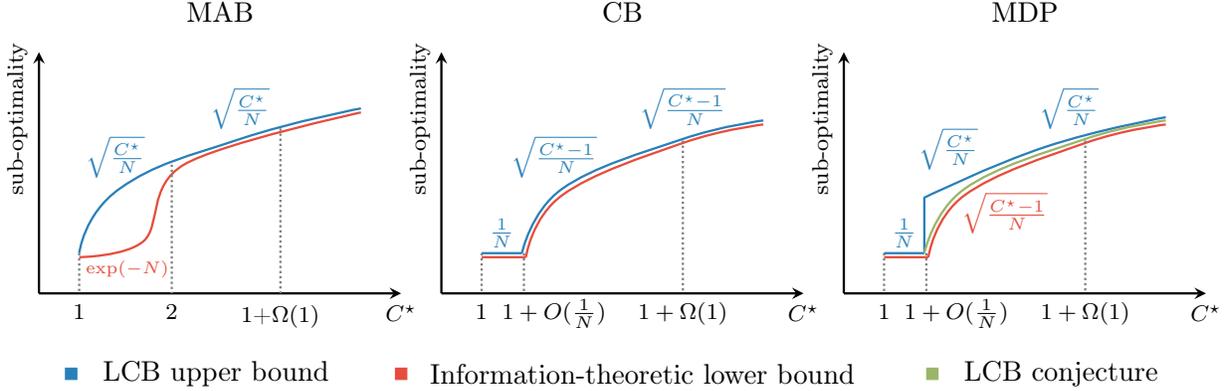
\begin{figure}[t]
    \centering
    \scalebox{1.07}{
    \begin{tikzpicture}[observed/.style={circle, draw=black, fill=black!10, thick, minimum size=10mm},
    hidden/.style={circle, draw=black, thick, minimum size=10mm},
    squarednode/.style={rectangle, draw=red!60, fill=red!5, very thick, minimum size=10mm},
    treenode/.style={rectangle, draw=none, thick, minimum size=10mm},
    rootnode/.style={rectangle, draw=none, thick, minimum size=10mm},
    squarednode/.style={rectangle, draw=none, fill=citeColor!20, very thick, minimum size=7mm},]
    \node at (2.25,3.5) {\small MAB};
    \draw[draw, ->, >=stealth, line width=0.8pt] (0, 0) -- (4.5, 0) node[ below] {\scriptsize $C^\star$};
    \draw[draw, ->, >=stealth, line width=0.8pt] (0, 0) -- (0, 3) node[above, rotate=90, xshift=-5ex] {\scriptsize sub-optimality};
    \draw[thick, citeColor] plot [smooth,tension=0.8] coordinates{(0.5,0.5) (1,1.3) (2.5,1.92) (4,2.3)};
    \draw[thick, linkColor] plot [smooth,tension=0.6] coordinates{(0.5,0.45) (1.3,0.65) (1.8,1.6) (4,2.25)};
    \node[linkColor] at (1.1,0.3) {\tiny $\exp(-N)$};
    \node[citeColor] at (0.95,1.7) {\scriptsize $\sqrt{\frac{C^\star}{N}}$};
    \draw[densely dotted, line width=0.8pt, gray] (0.5,0.5)--(0.5,0) node[black,below] {\scriptsize 1};
    \draw[densely dotted, line width=0.8pt, gray] (3,2.05)--(3,0) node[black,below] {\scriptsize 1+$\Omega(1)$};
    \node[citeColor] at (2.5,2.3) {\scriptsize $\sqrt{\frac{C^\star}{N}}$};
    \draw[densely dotted, line width=0.8pt, gray] (1.65,1.6)--(1.65,0) node[black,below] {\scriptsize 2};
    \node at (7.25,3.5) {\small CB};
    \draw[ ->, >=stealth, line width=0.8pt] (5, 0) -- (9.5, 0) node[below] {\scriptsize $C^\star$};
    \draw[ ->, >=stealth, line width=0.8pt] (5, 0) -- (5, 3) node[above, rotate=90, xshift=-5ex] {\scriptsize sub-optimality};
    \draw[line width=0.8pt,citeColor] (5.5,0.5) -- (6,0.5);
    \draw[line width=0.8pt,linkColor] (5.5,0.45) -- (6.05,0.45);
    \draw[line width=0.8pt,citeColor] plot [smooth,tension=0.7] coordinates{(6,0.5) (6.5,1.3) (8,1.92) (9,2.15)};
    \draw[line width=0.8pt,linkColor] plot [smooth,tension=0.7] coordinates{(6.05,0.45) (6.5,1.22) (8,1.87) (9,2.1)};
    \draw[densely dotted, line width=0.8pt, gray] (5.5,0.5)--(5.5,0) node[black, below] {\scriptsize 1};
    \draw[densely dotted, line width=0.8pt, gray] (6.025,0.5)--(6.025,0);
    \node at(6.4,-0.25) {\scriptsize $1+O(\frac{1}{N})$};
    \node at(8,-0.25) {\scriptsize $1+\Omega(1)$};
    \draw[densely dotted, line width=0.8pt, gray] (8,1.9)--(8,0);
    \node[citeColor] at (8,2.3) {\scriptsize
    $\sqrt{\frac{C^\star-1}{N}}$};
    \node[citeColor] at (6.4,1.7) {\scriptsize
    $\sqrt{\frac{C^\star-1}{N}}$};
    \node[citeColor] at (5.75,0.75) {\scriptsize $\frac{1}{N}$};
    %MDP
    \node at (12.25,3.5) {\small MDP};
    \draw[densely dotted, line width=0.8pt, gray] (10.5,0.5)--(10.5,0) node[black, below] {\scriptsize 1};
    \draw[->, >=stealth, line width=0.8pt] (10, 0) -- (14.5, 0) node[below] {\scriptsize $C^\star$};
    \draw[ ->, >=stealth, line width=0.8pt] (10, 0) -- (10, 3) node[above, rotate=90, xshift=-5ex] {\scriptsize sub-optimality};
    \draw[line width=0.8pt,citeColor] (10.5,0.5) -- (11,0.5);
    \draw[line width=0.8pt,citeColor] (11,0.5)--(11,1.2);
    \draw[thick, citeColor] plot [smooth,tension=0.8] coordinates{(11,1.19) (12.5,1.82) (14,2.19)};
    \draw[line width=0.8pt,linkColor] (10.5,0.45) -- (11.05,0.45);
    \draw[line width=0.8pt,conjColor] plot [smooth,tension=0.7] coordinates{(11,0.5) (11.5,1.3) (13,1.92) (14,2.15)};
    \draw[line width=0.8pt,linkColor] plot [smooth,tension=0.7] coordinates{(11.05,0.45) (11.5,1.22) (13,1.87) (14,2.1)};
    \node[citeColor] at (12.8,2.3) {\scriptsize $\sqrt{\frac{C^\star}{N}}$};
    \node[citeColor] at (11.3,1.8) {\scriptsize $\sqrt{\frac{C^\star}{N}}$};
    \node[linkColor] at (12,1) {\scriptsize $\sqrt{\frac{C^\star-1}{N}}$};
    \node[citeColor] at (10.75,0.75) {\scriptsize $\frac{1}{N}$};
    \node at(13,-0.25) {\scriptsize $1+\Omega(1)$};
    \draw[densely dotted, line width=0.8pt, gray] (13,1.9)--(13,0);
    \draw[densely dotted, line width=0.8pt, gray] (11.025,0.5)--(11.025,0);
    \node at(11.4,-0.25) {\scriptsize $1+O(\frac{1}{N})$};
    \draw [citeColor] plot [only marks, mark=square*] coordinates {(0.4,-1)};
    \node at (2.2,-1) {\small LCB upper bound};
    \draw [linkColor] plot [only marks, mark=square*] coordinates {(4.5,-1)};
    \node at (7.5,-1) {\small Information-theoretic lower bound};
    \draw [conjColor] plot [only marks, mark=square*] coordinates {(11.1,-1)};
    \node at (12.7,-1) {\small LCB conjecture};
    \end{tikzpicture}}
    \caption{The sub-optimality upper bounds and information-theoretic lower bounds for the LCB-based algorithms in MAB, CB with at least two contexts, and MDP settings. In all setting, it is assumed that the knowledge of $C^\star$ is not available to the LCB algorithm.}
    \label{fig:results_summary}
\end{figure}

\paragraph{Multi-armed bandits.} To address the offline learning problem in multi-armed bandits, LCB starts by forming a lower confidence bound---using the batch data---on the mean reward associated with each action and proceeds to select the one with the largest LCB. We show in Theorem~\ref{thm:MAB_LCB_upper} that LCB achieves a $\sqrt{C^\star/ N}$ sub-optimality competing with the optimal action for all $C^\star \geq 1$. It turns out that LCB is adaptively optimal in the regime $C^\star \in [2, \infty)$ in the sense that it achieves the minimal sub-optimality $\sqrt{C^\star / N}$ without the knowledge of the $C^\star$; see Theorem~\ref{thm.bandit_lower_bound}. We then turn to the case with $C^\star \in [1,2)$, in which the optimal action is pulled with more than probability $1/2$. In this regime, it is discovered that the optimal rate has an exponential dependence on $N$, i.e., $e^{-N}$, and is achieved by the naive algorithm of selecting the most played arm (cf.~Theorem~\ref{thm.imitation_bandit_bound}). To complete the picture, we also prove in~Theorem~\ref{thm:LCB_lower_bandit} that LCB cannot be adaptively optimal for all ranges of $C^\star\geq 1$ if the knowledge of $C^\star$ range is not available.

At first glance, it may seem that LCB for offline RL mirrors upper confidence bound (UCB) for online RL by simply flipping the sign of the bonus. However, our results reveal that the story in the offline setting is much more subtle than that in the online case. Contrary to UCB that achieves optimal regret in multi-armed bandits \cite{bubeck2011pure}, LCB is provably \textit{not} adaptively optimal for solving offline bandit problems under the $C^\star$ framework.

\paragraph{Contextual bandits.} The LCB algorithm for contextual bandits shares a similar design to that for multi-armed bandits. However, the  performance upper and lower bounds are more intricate and interesting when we consider contextual bandits with at least two states. 
With regards to the upper bound, we show in~Theorem~\ref{thm:LCB_CB_upper_bound} that LCB exhibits two different behaviors depending on the data composition $C^\star$. When $C^\star \geq 1 + S/N$, LCB enjoys a $\sqrt{S(C^\star -1) / N}$ sub-optimality, whereas when $C^\star \in [1, 1 + S/N)$, LCB achieves a sub-optimality with the rate $S/N$; see Figure~\ref{fig:results_summary}(b) for an illustration. The latter regime ($C^\star \approx 1$) is akin to the imitation learning case where the batch data is close to the expert data. LCB matches the performance of behavior cloning for the extreme case $C^\star = 1$. In addition, in the former regime ($C^\star \geq 1 + S/N$), the performance upper bound depends on the data composition through $C^\star -1$, instead of $C^\star$. This allows the rate of sub-optimality to smoothly transition from $1/ N$ to $1/\sqrt{N}$ as $C^\star$ increases. 
More importantly, both rates are shown to be minimax optimal in Theorem~\ref{thm:CB_IL_lower}, hence confirming the adaptive optimality of LCB for solving offline contextual bandits---in stark contrast to the bandit case. 
On the other hand, this showcases the advantage of the $C^\star$ framework as it provably interpolates the imitation learning regime and the (non-expert) offline RL regime. 

On a technical front, to achieve a tight dependency on $C^\star - 1$, a careful decomposition of the sub-optimality is necessary. In Section~\ref{subsec:CB_upper_anlaysis}, we present the four levels of decomposition of the sub-optimality of LCB that allow us to accomplish the goal. The key message is this: the sub-optimality is incurred by both the value difference and the probability of choosing a sub-optimal action. A purely value-based analysis falls short of capturing the probability of selecting the wrong arm and yields  a $1/\sqrt{N}$ rate regardless of $C^\star$. In contrast, the decomposition laid out in Section~\ref{subsec:CB_upper_anlaysis} delineates the cases in which the value difference (or the probability of choosing wrong actions) plays a bigger role.

\paragraph{Markov decision processes.} We combine the LCB approach with the traditional value iteration algorithm to solve the offline Markov decision processes. Ignore the dependence on the effective horizon $1/ (1- \gamma)$ for a moment. Similar behaviors to contextual bandits emerge: when $C^\star \in [1, 1+ 1/N)$, LCB achieves an $S/ N$ sub-optimality, and when (say) $C^\star \geq 1.1$, LCB enjoys a $\sqrt{SC^\star/N}$ rate; see Theorem~\ref{thm:MDP_upperbound_hoffding}. Both are shown in Theorem~\ref{thm:MDP_lower} to be minimax optimal in their respective regimes of $C^\star$, up to a $1/(1-\gamma)^2$ factor in sample complexity. And this leaves us with an interesting middle ground, i.e., the case when $C^\star \in (1+1/N, 1.1)$. Our lower bound still has a dependence $C^\star - 1$ as opposed to $C^\star$ in this regime, and we conjecture that LCB is able to close the gap in this regime. 

\begin{conjecture}[Adaptive optimality of LCB, Informal]\label{cnj:discounted_MDP}
    The LCB approach, together with value iteration is adaptively optimal for solving offline MDPs for all ranges of $C^\star$.
\end{conjecture}

We discuss the conjecture in detail in Section~\ref{sec:conjecture}, where we present an example showing that a variant of value iteration with LCB in the episodic case is able to achieve the optimal dependency on $C^\star$ and hence closing the gap between the upper and the lower bounds. A complete analysis of the LCB algorithm in the episodic MDP setting is presented in Appendix \ref{app:episodic_MDP}.

\paragraph{Notation.}\label{sec:notation}
We use calligraphy letters for sets and operators, e.g., $\cS, \cA$, and $\cT$. Given a set $\cS$, we write $|\cS|$ to represent the cardinality of $\cS$. Vectors are assumed to be column vectors except for the probability and measure vectors. The probability simplex over a set $\cS$ is denoted by $\Delta(\cS)$. For two $n$-dimensional vectors $x$ and $y$, we use $x \cdot y = x^\top y$ to denote
their inner product and $x\leq y$ to denote an element-wise inequality $x_i \leq y_i$ for all $i \in \{1, \dots, n\}$. We write $x \lesssim y$ when there exists a  constant $c>0$ such that $x \leq c y$. We use the notation $x \asymp y$ if constants $c_1, c_2 > 0$ exist such that $c_1 |x| \leq |y| \leq c_2 |x|$. We write $x \vee y$ to denote the supremum of $x$ and $y$. We write $f(x) = O(g(x))$ if   there exists some positive real number $M$ and some $x_0$   such that $|f(x)|\leq M g(x)$ for all $x\geq x_0$.  We use $\widetilde O(\cdot)$ to be the big-$O$ notation ignoring logarithmic factors. We write $f(x) = \Omega(g(x))$ if   there exists some positive real number $M$ and some $x_0$   such that $|f(x)|\geq  M g(x)$ for all $x\geq x_0$.

\section{Background and problem formulation}\label{sec:preliminaries}
We begin with reviewing some core concepts in Markov decision processes in Section~\ref{sec:MDP_model}.
Then we introduce the data collection model and the learning objective for offline RL in Section~\ref{sec:offline-RL}.
In the end, Section~\ref{sec:assumption} is devoted to the formalization and discussions of the weaker concentrability 
coefficient assumption that notably allows us to bridge offline RL with imitation learning.

\subsection{Markov decision processes} \label{sec:MDP_model}

\paragraph{Infinite-horizon discounted Markov decision processes.}
 We consider an infinite-horizon discounted Markov decision process (MDP) described by a tuple $M = (\cS, \cA, P, R, \rho, \gamma)$, where $\cS = \{1, \dots, S\}$ is a finite state space, $\cA = \{1, \dots, |\cA|\}$ is a finite action space, $P: \cS \times \cA \mapsto \Delta(\cS)$ is a probability transition matrix, $R: \cS \times \cA \mapsto \Delta([0,1])$ encodes a family of reward distributions with $r: \cS \times \cA \mapsto [0,1]$ as the expected reward function, $\rho: \cS \mapsto \Delta(\cS)$ is the initial state distribution, and $\gamma \in [0,1)$ is a discount factor. Upon executing action $a$ from state $s$, the agent receives a (random) reward  distributed according to $R(s,a)$ and transits to the next state $s'$ with probability $P(s'|s,a)$.

\paragraph{Policies and value functions.}
A stationary deterministic policy $\pi: \cS \mapsto \cA$ is a function that maps a state to an action. 
Correspondingly, the value function $V^\pi: \cS \mapsto \mathbb{R}$ of the policy $\pi$ is defined as the expected sum of discounted rewards starting at state $s$ and following policy $\pi$. More precisely, we have
\begin{align}\label{def:value_fn}
    V^\pi(s) \coloneqq \E \left[\sum_{t=0}^\infty \gamma^t r_t \;\middle|\; s_0 = s , a_t = \pi(s_t) \text{ for all } t\geq 0\right], \qquad \forall s \in \cS,
\end{align}
where the expectation is taken over the trajectory generated according to the transition kernel $s_{t+1} \sim P(\cdot \mid s_t, a_t)$ and reward distribution $r_t \sim R(\cdot \mid s_t, a_t)$. Similarly, the quality function (Q-function or action-value function) $Q^\pi: \cS \times \cA \rightarrow \mathbb{R}$ of policy $\pi$ is defined analogously:
\begin{align}\label{def:q_fn}
    Q^\pi(s,a) \coloneqq \E \left[\sum_{t=0}^\infty \gamma^t r_t \;\middle|\; s_0 = s, a_0 = a, a_t = \pi(s_t) \text{ for all } t\geq 1 \right] \qquad \forall s \in \cS, a \in \cA.
\end{align}
Denote 
\begin{equation}\label{defn:V-max}
    V_{\max} \coloneqq (1-\gamma)^{-1}.
\end{equation} 
It is easily seen that for any $(s,a)$, one has $0 \leq V^\pi(s) \leq V_{\max}$ and $0 \leq Q^\pi(s,a) \leq V_{\max}$.

Oftentimes, it is convenient to define a scalar summary of the performance of a policy $\pi$. 
This can be achieved by defining the expected value of a policy $\pi$:
\begin{align}\label{def:J_pi}
    J(\pi) \coloneqq \E_{s\sim \rho}[V^\pi(s)] = \sum_{s \in \cS} \rho(s)V^\pi(s).
\end{align}
It is well known that there exists a stationary deterministic policy $\pi^\star$ that simultaneously maximizes $V^\pi(s)$ for all $s \in \cS$, and hence maximizing the expected value $J(\pi)$; see e.g., \citet[Chapter 6.2.4]{puterman1990markov}. We use shorthands $V^\star \coloneqq V^{\pi^\star}$ and $Q^\star \coloneqq Q^{\pi^\star}$ to denote the optimal value function and the optimal Q-function.

\paragraph{Discounted occupancy measures.}
The (normalized) state discounted occupancy measures $d_\pi: \cS \mapsto [0,1]$ and state-action discounted occupancy measures $d^\pi: \cS \times \cA \mapsto [0,1]$ are respectively defined as
\begin{subequations}
\begin{alignat}{2}\label{def:state_occupancy}
    d_\pi(s) & \coloneqq (1-\gamma) \sum_{t=0}^\infty \gamma^t \prob_t(s_t = s; \pi), \qquad && \forall s \in \cS,\\ \label{def:state_action_occupancy}
    d^\pi(s,a) & \coloneqq (1-\gamma)\sum_{t=0}^\infty \gamma^t \prob_t(s_t = s, a_t = a; \pi), \qquad && \forall s \in \cS, a \in \cA,
\end{alignat}
\end{subequations}
where we overload notation and write $\prob_t(s_t = s; \pi)$ to denote the probability of visiting state $s_t = s$ (and similarly $s_t = s, a_t = a$) at step $t$ after executing policy $\pi$ and starting from $s_0 \sim \rho(\cdot)$.

\subsection{Offline data and offline RL}\label{sec:offline-RL}
\paragraph{Batch dataset.} The current paper focuses on offline RL, where the agent cannot interact with the MDP and instead is given a \textit{batch dataset} $\cD$ consisting of tuples $(s,a,r,s')$, where $r \sim R(s,a)$ and $s' \sim P(\cdot \mid s,a)$. For simplicity, we assume $(s,a)$ pairs are generated i.i.d.~according to a data distribution $\mu$ over the state-action space $\cS \times \cA$, which is \emph{unknown} to the agent.\footnote{The i.i.d. assumption is motivated by the data randomization performed in experience replay \cite{mnih2015human}.}
Throughout the paper, we denote by $N(s,a) \geq 0$ the number of times a pair $(s,a)$ is observed in $\cD$ and by $N = |\cD|$ the total number of samples.

\subsection{Assumptions on the dataset coverage} \label{sec:assumption} 
\begin{definition}[Single policy concentrability]\label{def:concentrability} Given a policy $\pi$, define $C^\pi$ to be the smallest constant that satisfies
\begin{align}
    \frac{{d}^\pi(s,a)}{\mu(s,a)} \leq C^\pi, \qquad \forall s \in \cS, a \in \cA.
\end{align}
\end{definition}
In words, $C^\pi$ characterizes the \textit{distribution shift} between the normalized occupancy measure induced by $\pi$ and data distribution $\mu$. For a stationary deterministic\footnote{Throughout the paper, when we talk about optimal policies, we restrict ourselves to deterministic stationary policies.} optimal policy, $C^\star \coloneqq C^{\pi^\star}$ is the ``best” \textit{concentrability coefficient} definition which is often much smaller than the widely-used uniform concentrability coefficient $C \coloneqq \max_{\pi} C^\pi$ which takes the maximum over all policies $\pi$. A small $C^\pi$ implies that data distribution covers $(s,a)$ pairs visited by policy $\pi$, whereas a small $C$ requires the coverage of $(s,a)$ visited by all policies.
Further discussion on different assumptions imposed on batch datasets in prior works is postponed to Section \ref{sec:related_work}.

\section{A warm-up: LCB in multi-armed bandits}\label{sec:bandit}

In this section, we focus on the simplest example of an MDP, the multi-armed bandit model (MAB), to motivate and explain the LCB approach. 
More specifically, the multi-armed bandit model is a special case of the MDP described in Section~\ref{sec:MDP_model} 
with $S = 1$ and $\gamma = 0$. 

In the MAB setting, the offline dataset $\cD$ is a set of tuples $\{(a_i, r_i)\}_{i=1}^{N}$ sampled independently from some joint distribution. Denote the marginal distribution of action $a_i$ as $\mu$. 
Let $r(a) \coloneqq \mathbb{E}[r_i \mid a_i = a]$ be the expectation of the reward distribution for action $a$. Competing with the optimal policy that chooses action $a^\star$, the data coverage assumption simplifies to 
\begin{align}\label{assumption:MAB_general}
    \frac{1}{\mu(a^\star)} \leq C^\star.
\end{align}
The goal of offline learning in MAB is to select an arm $\hat{a}$ that minimizes the expected sub-optimality
\begin{equation*}
    \E_\cD[J(\pi^\star) - J(\hat{\pi})] = \mathbb{E}_{\cD}[r(a^\star) - r(\hat{a})]. 
\end{equation*}
\subsection{Why does the empirical best arm fail?}
A natural choice for identifying the optimal action is to select the arm with the highest empirical mean reward. Mathematically,  for all $a \in \cA$, let $N(a) \coloneqq \sum_{i=1}^N \ind\{a_i = a\}$ and
\begin{equation*}
    \hat{r}(a)  \coloneqq\begin{dcases}
    0, & \text{if }N(a)=0,\\
    \frac{1}{N(a)}\sum_{i=1}^{N}r_{i} \ind \{a_{i} = a\}, & \text{otherwise}.
    \end{dcases}
\end{equation*}
The empirical best arm is then given by $\hat{a} \coloneqq \argmax_{a} \hat{r}(a)$. 

Though intuitive, the empirical best arm is quite \emph{sensitive} to the arms which have small observation counts $N(a)$: a less-explored sub-optimal arm might have high empirical mean just by chance (due to large variance) and overwhelm the true optimal arm. To see this, let us consider the following scenario. 

\paragraph{A failure instance for the empirical best arm.} Let $a^\star = 1$ be the optimal arm with a deterministic reward $1/2$. For the remaining sub-optimal arms, we set the reward distribution to be a Bernoulli distributions on $\{0,1\}$ with mean $1/4$.  
Consider even the benign case in which the optimal arm is drawn with dominant probability while the sub-optimal ones are sparsely drawn. 
Under such circumstances, there is a decent chance that one of the sub-optimal arms (say $a=2$) is drawn for very few times (say just one time) and unfortunately the observed reward is $1$, which renders $a = 2$ the empirical best arm. This clearly fails to achieve a low sub-optimality. 

Indeed, this intuition about the failure of the empirical best arm can be formalized in the following proposition. 
\begin{prop}[Failure of the empirical best arm]\label{lem:lower_bd_largest_empirical}
For any $\epsilon < 0.05$, $N\geq 500$, there exists a bandit problem with two arms  such that for $\hat a = \argmax_a \hat{r}(a)$, one has 
\begin{align*}
    \mathbb{E}_\cD[r(a^\star) - r(\hat a)]\geq  \epsilon.
\end{align*}
\end{prop}

\noindent It is worth pointing out that the above lower bound holds for any $\frac{1}{\mu(a^\star)}\leq C^\star$ with $C^\star - 1$ being a constant. See Appendix~\ref{app:proof_lower_bd_largest_empirical} for the proof of Proposition~\ref{lem:lower_bd_largest_empirical}. 

\medskip 
Proposition~\ref{lem:lower_bd_largest_empirical} reveals that even in the favorable case when $C^\star \approx 1$, returning the best empirical arm will have a constant error due to the high sensitivity to the less-explored sub-optimal arms. In contrast, the LCB approach, which we will introduce momentarily, will secure a sub-optimality of $\widetilde{O}(\sqrt{1/N})$ in this regime, hence reaching a drastic improvement over the vanilla empirical best arm approach. 

\subsection{LCB: The benefit of pessimism}
Revisiting the failure instance for the best empirical arm approach, one soon realizes that it is not sensible to put every action on an equal footing: for the arms that are pulled less often, one should tune down the belief on its empirical mean and be pessimistic on its true reward. Strategically, this principle of pessimism can be deployed with the help of a penalty function $b(a)$ that shrinks as the number of counts $N(a)$ increases. Instead of returning an arm maximizing the empirical reward, the pessimism principle leads us to the following approach: return
\begin{equation*}
    \hat{a} \in \argmax_{a} \; \hat{r}(a) - b(a).
\end{equation*}
Intuitively, one could view the right hand side $\hat{r}(a) - b(a)$ as a lower confidence bound (LCB) on the true mean reward $r(a)$. This LCB approach stands on the conservative side and seeks to find an arm with the largest lower confidence bound. 

Algorithm~\ref{alg:MAB-LCB} shows one instance of the LCB approach for MAB, in which the penalty function originiates from Hoeffding's inequality. We have the following performance guarantee for the LCB approach of Algorithm~\ref{alg:MAB-LCB}, whose proof can be found in Appendix~\ref{app:proof_MAB_LCB_upper}.

\begin{algorithm}[t]
\caption{LCB for multi-armed bandits}
\label{alg:MAB-LCB}
\begin{algorithmic}[1]
\State \textbf{Input:} Batch dataset $\cD = \{(a_i, r_i)\}_{i=1}^N$, and a confidence level $\delta \in (0,1)$.
\State Set $N(a) = \sum_{i=1}^N \ind\{a_i = a\}$ for all $a \in \cA$.
\For{$a \in \cA$}
\If{$N(a) = 0$} 
\State Set the empirical mean reward $\hat{r}(a) \mapsfrom 0$.
\State Set the penalty $b(a) \mapsfrom 1$.
\Else
\State Compute the empirical mean reward $\hat{r}(a) \mapsfrom \frac{1}{N(a)}\sum_{i=1}^{N}r_{i} \ind \{a_{i} = a\}$.
\State Compute the penalty $b(a) \mapsfrom  \sqrt{\frac{\log(2|\mathcal{A}|/\delta)}{2N(a)}}$.
\EndIf
\EndFor
\State \textbf{Return:} $\hat{a} = \argmax_a \hat{r}(a) - b(a)$.
\end{algorithmic}
\end{algorithm}

\begin{theorem}[LCB sub-optimality, MAB]\label{thm:MAB_LCB_upper}
Consider a multi-armed bandit and assume that 
\begin{align*}
    \frac{1}{\mu(a^\star)}\leq C^\star,
\end{align*}
for some $C^\star \geq 1$. Suppose that the sample size obeys $N \geq 8 C^\star \log N$. Setting $\delta = 1/N$, then action $\hat{a}$ returned by Algorithm~\ref{alg:MAB-LCB} obeys 
\begin{equation}\label{eq:expect-MAB}
    \E_\cD[r(a^\star) - r(\hat a)] \lesssim \min\left(1, \sqrt{\frac{C^\star\log(2N|\mathcal{A}|)}{N}}\right). 
\end{equation}
\end{theorem}
Applying the performance guarantee~\eqref{eq:expect-MAB} of LCB on the failure instance used in Proposition~\ref{lem:lower_bd_largest_empirical}, one sees that
LCB achieves a sub-optimality on the order of $\sqrt{(\log N)/ N}$, which clearly beats the best empirical arm. This demonstrates the benefit of pessimism over the vanilla approach. Intuitively, the LCB approach applies larger penalties to the actions that are observed only a few times.  Even if we have actions with huge fluctuations in their respective empirical rewards due to a small number of samples, the penalty term helps to rule them out. 

In fact, our proof yields a stronger high probability performance bound for $\hat{a}$ returned by Algorithm~\ref{alg:MAB-LCB}: for any $\delta \in (0,1)$, as long as $N \geq 8 C^\star \log(1/ \delta)$, we have with probability at least $1-2\delta$ that
\begin{equation}\label{eq:high-prob-MAB}
    r(a^\star) - r(\hat a)\leq\min\left(1, 2\sqrt{\frac{C^\star\log(2|\mathcal{A}|/\delta)}{N}}\right).
\end{equation}
Furthermore, for policy $\pi$ that selects a fixed action $a$, if $\frac{1}{\mu(a)}\leq C^\pi$ for some $C^\pi$, the same analysis gives the following guarantee: 
\begin{equation*}
    \E_\cD[\max\{0, r(a) - r(\hat a)\}] \lesssim \min\left(1, \sqrt{\frac{C^\pi\log(2N|\mathcal{A}|)}{N}}\right). 
\end{equation*}
This result shows that the LCB algorithm can compete with \textit{any covered} target policy that is not necessarily optimal, i.e., the output policy of the LCB algorithm performs nearly as well as the covered target policy.

\subsection{Is LCB optimal for solving offline multi-armed bandits?}
Given the performance upper bound~\eqref{eq:expect-MAB} of the LCB approach, it is a natural to ask whether LCB is optimal for solving the bandit problem using offline data. To address this question, we resort to the usual minimax criterion. Since we are dealing with lower bounds, without loss of generality, we assume that the expert always takes the optimal action. Consequently, we can define the following family of multi-armed bandits:
\begin{equation}
    \mathsf{MAB}(C^\star) = \{(\mu, R) \mid \frac{1}{\mu(a^\star)}\leq C^\star\}.
\end{equation}
$\mathsf{MAB}(C^\star)$ includes all possible pairs of behavior distribution $\mu$ and reward distribution $R$ such that the data coverage assumption $1/ \mu(a^\star) \leq C^\star$ holds. It is worth noting that the optimal action $a^\star$ implicitly depends on the reward distribution $R$. With this definition in place, we define the worst-case risk of any estimator $\hat{a}$ to be 
\begin{equation}
    \sup_{(\mu, R) \in \mathsf{MAB}(C^\star)} \mathbb{E}_\cD [ r(a^\star) - r(\hat{a})].
\end{equation}
 Here an estimator $\hat{a}$ is simply a measureable function of the data $\{(a_i, r_i)\}_{i=1}^{N}$ collected under the MAB instance $\mu$ and $R$. 

It turns out that LCB is optimal up to a logarithmic factor when $C^\star \geq 2$, as shown in the following theorem. 

\begin{theorem}[Information-theoretic limit, MAB]\label{thm.bandit_lower_bound}
For
$C^\star \geq 2$, one has
\begin{align}
    \inf_{\hat a}\sup_{(\mu, R) \in \mathsf{MAB}(C^\star)} \mathbb{E}_\cD [ r(a^\star) - r(\hat{a})] \gtrsim \min \left(1,  \sqrt{\tfrac{C^\star}{N}}\right). 
\end{align}
For $C^\star\in(1, 2)$, one has
\begin{align}\notag 
    \inf_{\hat a}\sup_{(\mu, R) \in \mathsf{MAB}(C^\star)} \mathbb{E}_\cD [ r(a^\star) - r(\hat{a})] \gtrsim  \exp\left(-(2-C^\star)\cdot\log\left(\frac{2}{C^\star-1}\right)\cdot N\right). 
\end{align}

\end{theorem}

\noindent See Appendix~\ref{app:proof_bandit_lower_bound} for the proof.

\subsection{Imitation learning in bandit: The most played arm achieves a better rate}
From the above analysis, we know that when $C^\star\geq 2$, the best possible expected sub-optimality is $\sqrt{{C^\star}/{N}}$, which is achieved by LCB. On the other hand, 
if we know that $1/\mu(a^\star)\leq C^\star$ where $ C^\star\in [1, 2)$, we can use imitation learning to further improve the rate. The algorithm for bandit is straightforward: pick the arm most frequently selected in dataset, i.e., $\hat a= \argmax_{a} N(a)$. The performance guarantee of the most played arm is stated in the following proposition. 
\begin{prop}[Sub-optimality of the most played arm]\label{thm.imitation_bandit_bound}
Assume that $\frac{1}{\mu(a^\star)}\leq C^\star$ for some $C^\star \in[1, 2)$. For $\hat a= \argmax_{a} N(a)$, we have
\begin{align}
     \mathbb{E}_\cD[r(a^\star) - r(\hat a)] \leq \exp\left(-N\cdot\mathsf{KL}\left(\mathrm{Bern}\left(\tfrac{1}{2}\right) \; \Big\| \; \mathrm{Bern}\left(\tfrac{1}{C^\star}\right)\right)\right).
\end{align}
\end{prop}

\noindent The proof is deferred to Appendix~\ref{app:proof_IL_bandit}. 

\medskip
When $C^\star\in[1, 2)$, 
one can see that the rate for the most played arm achieves an exponential dependence on $N$, whereas the upper bound for LCB is only $1/\sqrt{N}$. On the other hand, the most played arm algorithm completely fails when $C^\star>2$, while LCB still keeps the rate $1/\sqrt{N}$.

In terms of the dependence on $C^\star$,  the KL divergence above evaluates to $\log(C^\star/2) + \log(1/(C^\star-1))/2$ when the expert policy is optimal. One can see that as $C^\star\rightarrow 1$, the rate increases to the order of $1/(C^\star-1)^N$, which matches the lower bound in Theorem~\ref{thm.bandit_lower_bound} in terms of the dependence on $C^\star-1$. 

\subsection{Non-adaptivity of LCB}
One may ask whether LCB can achieve optimal rate under both cases of $C^\star\in[1, 2)$ and $C^\star\geq 2$. Unfortunately,
we show in the following theorem that  no matter how we set the parameter $\delta$ in Algorithm~\ref{alg:MAB-LCB},  LCB cannot be optimally adaptive in both regimes.
\begin{theorem}[Non-adaptivity of LCB, MAB]\label{thm:LCB_lower_bandit}
Let $C^\star =1.5$. There exists a two-armed bandit instance $(\mu_{0}, R_{0}) \in \mathsf{MAB}(C^\star)$ such that 
Algorithm~\ref{alg:MAB-LCB} with $L\coloneqq \sqrt{\log (2 |\cA| / \delta)/ 2 }$ satisfies
\begin{align*}
     \mathbb{E}_\cD[r(a^\star) - r(\hat a)] \gtrsim   \min\left(\frac{\sqrt{L}}{N}, \frac{1}{\sqrt{N}}\right)\cdot \exp(-32L).
\end{align*}
On the other hand, when $C^\star=6$, there exists $(\mu_{1}, R_{1}) \in \mathsf{MAB}(C^\star)$ such that
\begin{align*}
     \mathbb{E}_\cD[r(a^\star) - r(\hat a)] \gtrsim   \min\left(1, \sqrt{\frac{L}{N}}\right).
\end{align*}
\end{theorem}

\noindent The proof is deferred to Appendix~\ref{app:proof_LCB_lower_bandit}. 

\medskip
The theorem above can be understood in the following way: intuitively, a larger $L$ means that we put higher weight on the penalty of the arm instead of the empirical average of the arm. As $L\rightarrow \infty$, the LCB algorithm recovers the most played arm algorithm; while as $L\rightarrow 0$, the LCB algorithm recovers the empirical best  arm algorithm. 

When $C^\star\in(1, 2)$, we know from Theorem~\ref{thm.imitation_bandit_bound} that the most played arm achieves an exponential rate in $N$. In order to match the rate, we need to select $\delta$ such that $L\gtrsim N^\alpha$ for some $\alpha>0$. However, under this choice of $L$, the algorithm fails to achieve $1/\sqrt{N}$ rate when $C^\star\geq 6$, which can be done by setting $\delta = 1/N$ (and thus $L = \log(2|\mathcal{A}|N)$) according to Theorem~\ref{thm:MAB_LCB_upper}. This shows that it is impossible for LCB to achieve optimal rate under both cases of $C^\star\in(1, 2)$ and $C^\star\geq 2$ simultaneously.  

\section{LCB in contextual bandits}\label{sec:contextual_bandits}
In this section, we take the analysis one step further by studying offline learning in contextual bandits (CB). As we will see shortly, simply going beyond one state turns the tables in favor of the minimax optimality of LCB.

Formally, contextual bandits can be viewed as a special case of MDP described in Section \ref{sec:MDP_model} with $\gamma = 0$.
In CB setting, the batch dataset $\cD$ is a set of tuples $\{(s_i,a_i,r_i)\}_{i=1}^N$ sampled independently according to $(s_i,a_i) \sim \mu$, and $r_i \sim R(\cdot \mid s_i, a_i)$. Competing with an optimal policy, the data coverage assumption in the CB case simplifies to
\begin{align*}
    \max_s \frac{\rho(s)}{\mu(s,\pi^\star(s))} \leq C^\star. 
\end{align*}

The offline learning objective in CB turns into finding a policy $\hat{\pi}$ based on the batch dataset that minimizes the expected sub-optimality 
\begin{align*}
    \E_\cD[J(\pi^\star) - J(\hat{\pi})] = \E_{\cD, \rho}\left[r(s,\pi^\star(s)) - r(s,\hat{\pi}(s)) \right]. 
\end{align*}
\subsection{Algorithm and its performance guarantee}
The pessimism principle introduced in the MAB setting can be naturally extended to CB. First, the empirical expected reward is computed for all state-action pairs $(s,a) \in \cS \times \cA$ according to 
\begin{equation*}
    \hat{r}(s,a)  \coloneqq\begin{dcases}
    0, & \text{if }N(s,a)=0,\\
    \frac{1}{N(s,a)}\sum_{i=1}^{N}r_{i} \ind \{(s_i,a_i) = (s, a)\}, & \text{otherwise}.
    \end{dcases}
\end{equation*}
Pessimism is then applied through a penalty function $b(s,a)$ and for every state $s$ the algorithm returns
\begin{align*}
    \hat{\pi}(s) \in \argmax_a \hat{r}(s,a) - b(s,a).
\end{align*}
Algorithm \ref{alg:CB-LCB} generalizes the LCB instance given in Algorithm \ref{alg:MAB-LCB} to the CB setting.

\begin{algorithm}[t]
\caption{LCB for contextual bandits}
\label{alg:CB-LCB}
\begin{algorithmic}[1]
\State \textbf{Input:} Batch dataset $\cD = \{(s_{i},a_{i},r_{i})\}_{i=1}^{N}$, and confidence level $\delta$.
\State Set $N(s,a) = \sum_{i=1}^N \ind\{(s_i,a_i) = (s,a)\}$ for all $a \in \cA, s \in \cS$.
\For{$s \in \cS, a \in \cA$}
\If{$N(s,a) = 0$}
\State Compute the empirical reward $\hat{r}(s,a) \mapsfrom 0$.
\State Compute the penalty $b(s,a) = 1$.
\Else
\State Compute the empirical reward $\hat{r}(s,a) \mapsfrom \frac{1}{N(s,a)}\sum_{i=1}^{N}r_{i} \ind \{(s_{i},a_{i})=(s,a)\}$.
\State Compute the penalty $b(s,a) = \sqrt{\frac{2000\log(2S|\cA|/\delta)}{N(s,a)}}$.
\EndIf
\EndFor
\State \textbf{Return:} $\hat{\pi}(s) \in \argmax_{a} \hat{r}(s,a) - b(s,a)$ for each $s \in \cS$.
\end{algorithmic}
\end{algorithm}

The following theorem establishes an upper bound on the expected sub-optimality of the policy returned by Algorithm \ref{alg:CB-LCB}; see Appendix \ref{app:proof_LCB_CB_upper_bound} for a complete proof.
\begin{theorem}[LCB sub-optimality, CB]
\label{thm:LCB_CB_upper_bound} 
Consider a contextual bandit with $S \geq 2$ and assume that \[
\max_{s}\frac{\rho(s)}{\mu(s,\pi^\star(s))}\leq C^\star,
\] 
for some $C^\star \geq 1$. Setting $\delta = 1/N$, the policy $\hat{\pi}$ returned by Algorithm \ref{alg:CB-LCB} obeys
\begin{align*}
    \E_\cD[J(\pi^\star)-J(\hat{\pi})] \lesssim \min \left( 1, \widetilde{O} \left(\sqrt{\frac{S (C^\star-1) }{N}} + \frac{S}{N}\right) \right).
\end{align*}
\end{theorem}

It is interesting to note that the sub-optimality bound in Theorem \ref{thm:LCB_CB_upper_bound} consists of two terms. The first term is the usual statistical estimation rate of $1/\sqrt{N}$. The second term is due to \textit{missing mass}, which captures the suboptimality incurred in states for which an optimal arm is never observed in the batch dataset. More importantly, the dependency of the first term on data composition is $C^\star - 1$ instead of $C^\star$. When $C^\star$ is close to one, LCB enjoys a faster rate of $1/N$, reminiscent of the rates achieved by behavioral cloning in imitation learning, without the knowledge of $C^\star$ or the behavior policy. Furthermore, the convergence rate smoothly transitions from $1/N$ to $1/\sqrt{N}$ as $C^\star$ increases.

\subsection{Optimality of LCB for solving offline contextual bandits}
In this section, we establish an information-theoretic lower bound for the contextual bandit setup described above.
Define the following family of contextual bandits problems
\begin{equation*}
    \mathsf{CB}(C^\star) \coloneqq \{(\rho, \mu, R) \mid \max_{s}\frac{\rho(s)}{\mu(s,\pi^\star(s))} \leq C^\star \}.
\end{equation*}
Note that the optimal policy $\pi^\star$ implicitly depends on the reward distribution $R$.

Let $\hat \pi: \mathcal{S} \mapsto \mathcal{A}$ be an arbitrary estimator of the best arm $\pi(s)$ for any state $s$, which is a measurable function of the data $\{(s_i, a_i, r_i)\}_{i=1}^{N}$. The worst-case risk of $\hat{\pi}$ is defined as 
\begin{align*}
    \sup_{(\rho, \mu,R)\in\mathsf{CB}(C^\star)}  \E_\cD[J(\pi^{\star})-J(\hat{\pi})].
\end{align*}
We have the following minimax lower bound for offline learning in contextual bandits with $S \geq 2$; see Appendix~\ref{app:proof_CB_lower} for a proof. Note that the case of $S = 1$ is already addressed in Theorem \ref{thm.bandit_lower_bound}.
\begin{theorem}[Information-theoretic limit, CB] \label{theorem:lower_bound_offline_CB}
Assume that $S \geq 2$. For any $C^\star \geq 1$, one has
\begin{align*}
    \inf_{\hat \pi} \sup_{(\rho, \mu,R)\in\mathsf{CB}(C^\star)}  \E_\cD[J(\pi^{\star})-J(\hat{\pi})]\gtrsim \min\left(1, \sqrt{\frac{S(C^\star-1)}{N}} + \frac{S}{N}\right).
\end{align*}
\end{theorem}

Comparing Theorem~\ref{theorem:lower_bound_offline_CB} with Theorem~\ref{thm:LCB_CB_upper_bound}, one readily sees that the LCB approach enjoys a near-optimal rate in contextual bandits with $S \geq  2$, regardless of the data composition parameter $C^\star$. This is in stark contrast to the MAB case.

On a closer inspection, in the $C^\star \in [1,2)$ regime, there is a clear separation between the information-theoretic difficulties of offline learning in MAB, which has an exponential rate in $N$, and CB with at least 2 states, which has a $1/N$ rate. The reason behind this separation is the possibility of missing mass when $S \geq 2$. Informally, when there is only one state, the probability that an optimal action is never observed in the dataset decays exponentially. On the other hand, when there are more than one states, the probability that an optimal action is never observed for at least one state decays with the rate of $1/N$.
\medskip

Assume hypothetically that we are provided with the knowledge that $C^\star \in (1,2)$. Recall that with such a knowledge, the most played arm achieves a faster rate in the MAB setting. Under this circumstance, one might wonder whether simply picking the most played arm in every state also achieves a fast rate in the CB setting. Strikingly, the answer is negative as the following proposition shows that the most played arm fails to achieve a vanishing rate when $C^\star\in(1, 2)$. The proof of this theorem is deferred to Appendix~\ref{app:proof_CB_IL_lower}.
\begin{prop}[Failure of the most played arm, CB]\label{thm:CB_IL_lower}
For any $C^\star\in(1, 2)$, there exists a contextual bandit problem $(\rho, \mu,R)\in\mathsf{CB}(C^\star)$  such that for the policy $\hat \pi(s) = \argmax_{a} N(s, a)$,  
\begin{align*}  \lim_{N\rightarrow \infty}\E_\cD[J(\pi^{\star})-J(\hat{\pi})] \geq C^\star-1.
\end{align*}
\end{prop}
We briefly describe the intuition here. Under concentrability assumption, we can move at most $C^\star-1$ mass from $d^\star$ to  sub-optimal actions. Thus we can design a specific contextual bandit instance such that a $C^\star-1$ fraction of the states pick wrong actions by choosing the most played arm instead. This shows that even when $C^\star \in (1, 2)$, the most played arm approach for CB does not have a decaying rate in $N$, whereas in the MAB case it converges exponentially fast.

\subsection{Architecture of the proof}\label{subsec:CB_upper_anlaysis}
We pause to lay out the main steps to prove the upper bound in Theorem~\ref{thm:LCB_CB_upper_bound}. 
It is worth pointing out that following the MAB sub-optimality analysis as detailed in Appendix~\ref{app:proof_MAB_LCB_upper} only yields a crude upper bound of $\sqrt{C^\star S/N}+S/N$ on the sub-optimality of $\hat \pi$. When $C^\star$ is close to one, i.e., when we have access to a nearly-expert dataset, such analysis only gives a $\sqrt{S/N}$ rate. This rate is clearly worse than the rate $S/N$ achieved by the imitation learning algorithms. Therefore, special considerations are required for analyzing the sub-optimality of LCB in contextual bandits in order to establish the tight dependence of $\sqrt{(C^\star-1) S/N} + S/N$ instead of $\sqrt{C^\star S/N}$.

We achieve this goal by directly analyzing the policy sub-optimality via a gradual decomposition of the sub-optimality of $\hat{\pi}$ as illustrated in Figure~\ref{fig:sub_opt_decomposition}. The decomposition steps are described below.

\begin{figure}[t]
    \centering
    \scalebox{0.95}{
    \begin{tikzpicture}[observed/.style={circle, draw=black, fill=black!10, thick, minimum size=10mm},
    hidden/.style={circle, draw=black, thick, minimum size=10mm},
    squarednode/.style={rectangle, draw=red!60, fill=red!5, very thick, minimum size=10mm},
    treenode/.style={rectangle, draw=none, thick, minimum size=10mm},
    rootnode/.style={rectangle, draw=none, thick, minimum size=10mm},
    squarednode/.style={rectangle, draw=none, fill=citeColor!20, very thick, minimum size=7mm},decoration={brace},]
    \def \x {0}
    \def \h {0.7}
    \def \l {0.5}
    \def \w {3}
    \def \r {2.5}
    \node[rootnode]   (n1)   at(\x-10pt,0) {$\E_{\cD}[J(\pi^\star) - J(\hat{\pi})]$};
    \node[treenode]     (n2)    at(\x+\w,0+\h)   {$N(s,\pi^\star(s)) = 0$};
    \node[citeColor]     (n3)    at(\x+\w+1.9, 0+\h)   {\large $\rightarrow T_1$};
    \node[treenode]     (n4)    at(\x+\w,0-\h)   {$N(s,\pi^\star (s)) \geq 1$};
    \draw [very thick,decorate,decoration={calligraphic brace,amplitude=10pt}] 
    ([yshift=0]n4.west) -- ([yshift=0]n2.west);
    \node[treenode]     (n5)    at(\x+2*\w-0.65,0)   {$\ind \{\cE^c\}$}; 
    \node[citeColor]       (n6)    at([xshift=12]n5.east)   {\large $\rightarrow T_2$};
    \node[treenode]     (n7)    at(\x+2*\w-0.7,0-2*\h)    {$\ind \{\cE\}$}; 
    \draw [very thick,decorate,decoration={calligraphic brace,amplitude=10pt}] 
    ([yshift=0]n7.west) -- ([yshift=0]n5.west);
    \node[treenode]     (n8)    at(\x+2*\w+1.4,0-\h) {$\rho(s) < \frac{2C^\star L}{N}$};  
    \node[citeColor]    (m8)    at([xshift=12]n8.east) {\large $\rightarrow T_3$};
    \node[treenode]     (n9)    at(\x+2*\w+1.4,0-3*\h)  {$\rho(s) \geq \frac{2C^\star L}{N}$};
    \draw [very thick,decorate,decoration={calligraphic brace,amplitude=10pt}] 
    ([yshift=0]n9.west) -- ([yshift=0]n8.west);
    \node[treenode]     (n14)    at(\x+3.35*\w+0.5,0-3*\h+0.5*\h) {$\mu(s,\pi^\star(s)) < 10   \overline{\mu}(s)$};  
    \node[citeColor]      (n15)    at(\x+3.35*\w+2.7,0-3*\h+0.5*\h)   {\large $\rightarrow T_4$};
    \node[treenode]     (n16)    at(\x+3.35*\w+0.5,0-3*\h-0.5*\h)  {$\mu(s,\pi^\star(s)) \geq 10  \overline{\mu}(s)$};
    \node[citeColor]       (m18)    at(\x+3.35*\w+2.7,0-3*\h-0.5*\h)   {\large $\rightarrow T_5$};
    \draw [very thick,decorate,decoration={calligraphic brace,amplitude=5pt}] 
    ([yshift=0]n16.west) -- ([yshift=0]n14.west);
    \end{tikzpicture}}
    \caption{Decomposition of the sub-optimality of the policy $\hat \pi$ returned by Algorithm \ref{alg:CB-LCB}.}
    \label{fig:sub_opt_decomposition}
\end{figure}

\paragraph{First level of decomposition.}
In the first level of decomposition, we separate the error based on whether $N(s,\pi^\star(s))$ is zero for a certain state $s$. 
When $N(s,\pi^\star(s)) = 0$, there is absolutely no basis for the LCB approach to figure out the correct action $\pi^\star(s)$. 
Fortunately, this type of error, incurred by \emph{missing mass}, can be bounded by 
\begin{align}
    T_1 \lesssim \frac{C^\star S}{N}.
\end{align}
From now on, we focus on the case in which the expert action $\pi^\star(s)$ is seen for every state $s$. 

\paragraph{Second level of decomposition.}
The second level of decomposition hinges on the following clean/good event:
\begin{equation}\label{eq:event-CB_main_text}
\mathcal{E}\coloneqq\{\forall s,a: \; |r(s,a)-\hat{r}(s,a)|\leq b(s,a)\}.    
\end{equation}
In words, the event $\cE$ captures the scenario in which the penalty function provides valid confidence bounds for every state-action pair.
Standard concentration arguments tell us that $\cE$ takes place with high probability, i.e.,  the term $T_2$ in the figure is no larger than $\delta$. By setting $\delta$ small, say $1/N$, we are allowed to concentrate on the case when $\cE$ holds. 

\paragraph{Third level of decomposition.}
The third level of decomposition relies on the observation that states with small weights (i.e., $\rho(s)$ is small) have negligible effects on the sub-optimality $J(\pi^\star) - J(\hat \pi)$. More specifically, the aggregated contribution $T_3$ from the states with $\rho(s) \lesssim \frac{C^\star L}{N}$ is upper bounded by
\begin{align}
    T_3 \lesssim \frac{C^\star SL}{N}.
\end{align}
This allows us to focus on the states with large weights. We record an immediate consequence of large $\rho(s)$ and the data coverage assumption, that is $\mu(s, \pi^\star(s)) \geq \rho(s) / C^\star \asymp L / N$.

\paragraph{Fourth level of decomposition.}
Now comes the most important part of the error decomposition, which is not present in the MAB analysis. 
We decompose the error based on whether the optimal action has a higher data probability $\mu(s,\pi^\star(s))$ than the total  probability of sub-optimal actions $\overline{\mu}(s) \coloneqq \sum_{a \neq \pi^\star(s)} \mu(s, a)$.
In particular, when $\mu(s,\pi^\star(s)) < 10 \overline{\mu}(s)$, we can repeat the analysis of MAB and show that 
\begin{align*}
    T_4 \lesssim \sqrt{\frac{S(C^\star - 1)L}{N}}.
\end{align*}
Here, the appearance of $C^\star -1$, as opposed to $C^\star$ is due to the restriction $\mu(s,\pi^\star(s)) < 10 \overline{\mu}(s)$. 
One can verify that $\mu(s,\pi^\star(s)) < 10 \overline{\mu}(s)$ together with the data coverage assumption ensures that 
\[
\sum_{s:\rho(s)\geq {2C^\star L} / {N}, \mu(s,\pi^\star(s)) < 10   \overline{\mu}(s)} \rho (s) \lesssim C^\star - 1.
\]
On the other hand, when $\mu(s,\pi^\star(s)) \geq 10 \overline{\mu}(s)$, i.e., when the optimal action is more likely to be seen in the dataset, the penalty function $b(s,\pi^\star(s))$ associated with the optimal action would be much smaller than those of the sub-optimal actions. Thanks to the LCB approach, the optimal action will be chosen with high probability, i.e., $T_5 \lesssim 1/ N^{10}$.

Putting the pieces together, we arrive at the desired rate $O(\sqrt{\frac{S(C^\star - 1)}{N}} + \frac{S}{N})$.

\section{LCB in Markov decision processes}\label{sec:LCB_MDP}
Now we are ready to instantiate the LCB principle to the full-fledged Markov decision process. 
We propose a variant of value iteration with LCB (VI-LCB) in Section~\ref{sec:VI-LCB} and present its performance guarantee in Section~\ref{sec:VI-LCB-upper}. 
Section \ref{sec:MDP-lower} is devoted to the information-theoretic lower bound for offline learning in MDPs, which leaves us with a regime in which 
it is currently unclear whether LCB for MDP is optimal or not. 
However, we conjecture that VI-LCB is optimal for all ranges of $C^\star$. We conclude our discussion in Section \ref{sec:conjecture} with an explanation about the technical difficulty of closing the gap and a preview to a simple episodic example where we manage to prove the optimality of LCB with a rather \emph{intricate} analysis.

\paragraph{Additional notation.} We present the algorithm and results in this section with the help of some matrix notation for MDPs. For a function $f: \cX \mapsto \R$, we overload the notation and write $f \in \R^{|\cS|}$ to denote a vector with elements $f(x)$, e.g., $V, Q,$ and $r$. We write $P \in \R^{S|\cA| \times S}$ to represent the probability transition matrix whose $(s,a)$-th row denoted by $P_{s,a}$ is a probability vector representing $P(\cdot \mid s,a)$. We use $P^\pi \in \R^{S|\cA| \times S|\cA|}$ to denote a transtion matrix induced by policy $\pi$ whose $(s,a) \times (s',a')$ element is equal to $P(s'|s,a) \pi(a'|s')$. We write $\rho^\pi \in \R^{S|\cA|}$ to denote the initial distribution induced by policy $\pi$ whose $(s,a)$ element is equal to $\rho(s) \pi(a|s)$. 

\subsection{Offline value iteration with LCB}\label{sec:VI-LCB}
Our algorithm design builds upon the classic value iteration algorithm. 
In essence, value iteration updates the value function $V \in \mathbb{R}^{S}$ using
\begin{align*}
    Q(s,a) &\mapsfrom r(s,a) + \gamma P_{s,a} \cdot V, \qquad \text{for all }(s,a), \\
    V(s) &\mapsfrom \max_{a} Q(s,a), \qquad \text{for all }s.
\end{align*}
Note, however, with offline data, we do not have access to the expected reward $r(s,a)$ and the true transition dynamics $P_{s,a}$. 
One can naturally replace them with the empirical counterparts $\hat{r}(s,a)$ and $\hat{P}_{s,a}$ estimated from offline data $\cD$, and arrive at the empirical value iteration:
\begin{align*}
    Q(s,a) &\mapsfrom \hat{r}(s,a) + \gamma \hat{P}_{s,a} \cdot V, \qquad \text{for all }(s,a), \\
    V(s) &\mapsfrom \max_{a} Q(s,a), \qquad \text{for all }s.
\end{align*}
Mimicking the algorithmic design for MABs and CBs, we can subtract a penalty function $b(s,a)$ from the $Q$ update as the finishing touch, which yields the value iteration algorithm with LCB:
\begin{align}\label{eq:base-VI-LCB}
    Q(s,a) &\mapsfrom \hat{r}(s,a) - b(s,a) + \gamma \hat{P}_{s,a} \cdot V, \qquad \text{for all }(s,a), \\
    V(s) &\mapsfrom \max_{a} Q(s,a), \qquad \text{for all }s.
\end{align}

\begin{algorithm}[t]
\caption{Offline value iteration with LCB (VI-LCB)}
\begin{algorithmic}[1]
\State \textbf{Inputs:} Batch dataset $\cD$, discount factor $\gamma$, and confidence level $\delta$.
\State Set $T \coloneqq \frac{\log N}{1-\gamma}$.
\State Randomly split $\cD$ into $T+1$ sets $\cD_t = \{(s_i, a_i, r_i, s_i')\}_{i=1}^m$ for $t \in \{0, 1, \dots, T\}$ with $m \coloneqq N / (T+1)$.
\State Set $m_0(s,a) \coloneqq \sum_{i=1}^m \ind \{(s_i, a_i) = (s,a)\}$ based on dataset $\cD_0$.
\State For all  $a \in \cA$ and $s \in \cS$, initialize $Q_0(s,a) = 0 $, $V_0(s) = 0$ and set $\pi_0(s)=\argmax_{a}m_0(s, a)$. 
\For{$t = 1, \dots, T$}
\State Initialize $r_t(s,a) = 0$ and set $P^t_{s,a}$ to be a random probability vector.
\State Set $m_t(s,a) \coloneqq \sum_{i=1}^m \ind \{(s_i, a_i) = (s,a)\}$ based on dataset $\cD_t$.
\State Compute penalty $b_t(s,a)$ for $L = 2000\log(2 (T+1)S|\cA|/\delta)$
\begin{align}\label{eq:reward_adjustment_hoeffding}
    b_t(s,a) \coloneqq V_{\max} \cdot \sqrt{\frac{L}{m_t(s,a) \vee 1}}.
\end{align}

\For{$(s,a) \in (\cS, \cA)$}
    \If{$m_t(s,a) \geq 1$}
        \State Set $P^t_{s,a}$ to be empirical transitions and $r_t(s,a)$ be empirical average of rewards.
    \EndIf
    \State Set $Q_{t}(s,a) \mapsfrom r_{t}(s,a) - b_t(s,a) + \gamma P_{s,a}^t \cdot V_{t-1}$.
\EndFor
\State Compute $V^{\text{mid}}_{t} \leftarrow \max_a Q_{t}(s,a)$ and $\pi^{\text{mid}}_{t}(s) \in \argmax_a Q_{t}(s,a)$.
\For{$s \in \cS$}
\label{line:imposing_monotonicity}\If{$V^{\text{mid}}_{t}(s) \leq V_{t-1}(s)$} $V_{t}(s) \leftarrow  V_{t-1}(s)$ and $ \pi_{t}(s) \leftarrow \pi_{t-1}(s)$.
\Else \; $V_{t}(s) \leftarrow V^{\text{mid}}_{t}(s)$ and $\pi_{t}(s) \leftarrow \pi^{\text{mid}}_{t}(s)$.
\EndIf 
\EndFor
\EndFor
\State \textbf{Return} $\hat{\pi} \coloneqq \pi_T$.
\end{algorithmic}
\label{alg:OVI-LCB-DS}
\end{algorithm}

Algorithm~\ref{alg:OVI-LCB-DS} uses the update rule~\eqref{eq:base-VI-LCB} as its key component as well as a few other tricks:

\begin{itemize}[leftmargin=*]
    \item Data splitting: Instead of using the full offline data $\cD = \{ (s_i, a_i, r_i, s_i')\}_{i=1}^N$ to form the empirical estimates $\hat{r}(s,a)$ and $\hat{P}_{s,a}$, Algorithm~\ref{alg:OVI-LCB-DS} deploys data splitting where each iteration~\eqref{eq:base-VI-LCB} uses different samples to perform the update. This procedure is not needed in practice, however it is helpful in alleviating the dependency issues in the analysis, resulting in the removal of an extra factor of $S$ in the sample complexity.
    
    \item Monotonic update: Unlike traditional value iteration methods, Algorithm~\ref{alg:OVI-LCB-DS} involves a monotonic improvement step, in which the value function $V$ and the policy $\pi$ are updated only when the corresponding value function is larger than that in the previous iteration. This extra step was first proposed in the work~\citet{sidford2018near} for reinforcement learning with access to a generative model. In a nutshell, the key benefit of the monotonic update is to shave a $1/(1-\gamma)$ factor in the sample complexity; we refer the interested reader to the original work~\citet{sidford2018near} for further discussions on this step.
\end{itemize}

\subsection{Performance guarantees of VI-LCB}\label{sec:VI-LCB-upper}

Now we turn to the performance guarantee for the VI-LCB algorithm (cf.~Algorithm~\ref{alg:OVI-LCB-DS}).
\begin{theorem}[LCB sub-optimality, MDP]\label{thm:MDP_upperbound_hoffding} 
Consider a Markov decision process and assume that 
\begin{align*}
    \max_{s,a} \frac{{d}^\star(s,a)}{\mu(s,a)} \leq C^\star. 
\end{align*}
Then, for all $C^\star \geq 1$, policy $\hat{\pi}$ returned by Algorithm \ref{alg:OVI-LCB-DS} with $\delta = 1/N$ achieves
\begin{align}
    \E_\cD \left[J(\pi^\star) - J(\hat{\pi})\right]\lesssim \min \left(\frac{1}{1-\gamma}, \sqrt{\frac{SC^\star}{(1-\gamma)^5 N}}\right).
    \label{eq:MDP-upper-bound-case-1}
\end{align}
In addition, if $1 \leq C^\star \leq 1+\frac{L\log(N)}{200(1-\gamma)N}$, we have a tighter performance upper bound
\begin{align}\label{eq:MDP-upper-bound-case-2}
    \E_\cD \left[J(\pi^\star) - J(\hat{\pi})\right] \lesssim \min \left(\frac{1}{1-\gamma}, \frac{S}{(1-\gamma)^4N} \right).
\end{align}
\end{theorem}
We will shortly provide a proof sketch of Theorem \ref{thm:MDP_upperbound_hoffding}; a complete proof is deferred to Appendix~\ref{app:proof_MDP_upperbound}. The upper bound shows that for all regime of $C^\star\geq 1$, we can guarantee a rate of $\widetilde O(\sqrt{SC^\star/((1-\gamma)^5N)})$, which is similar to the rate of contextual bandit when the $C^\star=1+\Omega(1)$ by taking $\gamma=0$. When $C =1+O(\log(N)/N)$, we can show a rate $S/((1-\gamma^4)N)$, which also recovers the result in contextual bandit case. However, in the regime of $C^\star\in[1+\Omega(\log(N)/N), 1+O(1)]$, contextual bandit gives $\sqrt{S(C^\star-1)/N}$, while we fail to give the same dependence on $C^\star$ in this case. We defer the further discussion on the sub-optimality of this regime to Section \ref{sec:conjecture}.

\begin{remark} Relaxation of the concentrability assumption is possible by allowing the ratio to hold only for a subset $\cC$ of state-action pairs and characterizing the sub-optimality incurred by $(s,a) \in \cC$ via a missing mass analysis dependent on a constant $\xi$ such that $\sum_{(s,a) \not \in \cC} d^\star(s,a) \leq \xi$.
\end{remark}

\paragraph{Proof sketch for Theorem \ref{thm:MDP_upperbound_hoffding}.}

For the general case of $C^\star\geq 1$, we first define the clean event of interest as below. 
\begin{align}\label{def:clean_event}
    \cE_{\text{MDP}} \coloneqq \left\{ \forall s, a, t: \; \; \middle| r(s,a) - r_t(s,a) + \gamma  \left( P_{s,a} - P^t_{s,a} \right) \cdot V_{t-1} \middle| \leq b_t(s,a)  \right\}.
\end{align}
In words, on the event $\cE_{\text{MDP}}$, the  penalty function $b_t(s,a)$ well captures the statistical fluctuations of the Q-function estimate $r_t(s,a) + \gamma P_{s,a}^{t}\cdot V_{t-1}$. The following lemma
shows that this event happens with high probability. The proof is postponed to Appendix \ref{app:proof_penalty_LCB}.
\begin{lemma}[Clean event probability, MDP]\label{lemma:penalty_LCB}
One has $\prob(\cE_{\text{MDP}}) \geq 1-\delta.$
\end{lemma}
In the above lemma, concentration of $V_t$ is only needed instead of any value function $V$ such as required in the work \citet{yu2020mopo}. For the latter to hold, one needs to introduce another factor of $\sqrt{S}$ by taking a union bound. We avoid a union bound by exploiting the independence of $P^t_{s,a}$ and $V_t$ obtained by randomly splitting the dataset. This is key to obtaining an optimal dependency on the state size $S$. 

Under the clean event, we can show that the  monotonically increasing value function $V_t$  always lower bounds the value of the corresponding policy $\pi_t$, along with  a recursive inequality on the sub-optimality of $Q_{t+1}$ w.r.t. $Q^\star$ to penalty and sub-optimality of the previous step.
\begin{prop}[Contraction properties of Algorithm \ref{alg:OVI-LCB-DS}]\label{prop:monotonicity_data_splitting}  Let $\pi$ be an arbitrary policy. On the event $\mathcal{E}_{MDP}$, one has for all $s \in \cS, a \in \cA$, and $t \in \{1, \dots, T\}$:
\begin{align*}
    V_{t-1} \leq V_t \leq V^{\pi_t} \leq V^\star, \quad 
    \; Q_t \leq r + \gamma P V_{t-1}, \quad 
    \mathrm{ and } \quad Q^\pi - Q_{t} \leq \gamma P^\pi (Q^\pi - Q_{t-1}) + 2 b_t.
\end{align*}
\end{prop}
By recursively applying the last inequality, we can derive a value difference lemma. The following lemma relates the sub-optimality to the penalty term $b_t$, of which we have good control:
\begin{lemma} [Value difference for Algorithm \ref{alg:OVI-LCB-DS}] \label{lemma:value_difference} Let $\pi$ be an arbitrary policy. On the event $\cE_{\text{MDP}}$, one has for all $t \in \{1,\dots, T\}$
\begin{align*}
    J(\pi) - J(\pi_{t}) \leq \frac{\gamma^t}{1-\gamma} + 2 \sum_{i=1}^{t} \E_{\nu^\pi_{t-i}}[b_i(s,a)].
\end{align*}
Here, $\nu^\pi_k \coloneqq \gamma^k \rho^\pi (P^\pi)^k$ for $k \geq 0$.
\end{lemma}
The proof is provided in Appendix \ref{app:proof_value_difference}.
The value difference bound  has two terms: the first term is due to convergence error of value iteration and the second term is the error caused by subtracting penalties $b_i(s,a)$ in each iteration $i$ from the rewards. By plugging in $b_i$ and choosing $t$ appropriately we can get the desired performance guarantee.

For the case of $1 \leq C^\star \leq 1+\frac{L\log(N)}{200(1-\gamma)N}$, we adopt a similar decomposition as the contextual bandit analysis sketched in Section~\ref{subsec:CB_upper_anlaysis}. The only difference is that since $C^\star$ is small enough, we know that all the sub-optimal actions have very small mass in the $\mu$. Thus LCB enjoys a rate of $1/N$ as the imitation learning case.

\subsection{Information-theoretic lower bound for offline RL in MDPs}\label{sec:MDP-lower}
In this section, we focus on the statistical limits of offline learning in MDPs. 

Define the following family of MDPs
\begin{equation*}
    \mathsf{MDP}(C^\star) = \{(\rho, \mu, P, R) \mid \max_{s,a} \frac{{d}^\star(s,a)}{\mu(s,a)} \leq C^\star \}.
\end{equation*}
Note that here the normalized discounted occupancy measure ${d}^\star$ depends implicitly on the specification of the MDP, i.e., $\rho$, $P$, and $R$.

We have the following minimax lower bound for offline policy learning in MDPs, with the proof deferred to Appendix~\ref{app:proof_MDP_lower}.
\begin{theorem}[Information-theoretic limit, MDP]\label{thm:MDP_lower}
For any $C^\star\geq 1, \gamma\geq 0.5$, one has  
 \begin{align*}
\inf_{\hat \pi} \sup_{(\rho, \mu, P, R) \in \mathsf{MDP}(C^\star)}  \E_\cD[J(\pi^{\star})-J(\hat{\pi})]  \gtrsim \min\left(\frac{1}{1-\gamma}, \frac{S}{(1-\gamma)^2N} + \sqrt{\frac{S(C^\star-1)}{(1-\gamma)^3N}}\right).
 \end{align*}
\end{theorem}

\medskip

Several remarks are in order. 

\paragraph{Imitation learning and offline learning.}
It is interesting to note that similar to the lower bound for contextual bandits, the statistical limit involves two separate terms $\frac{S}{(1-\gamma)^2N}$ and $\sqrt{\frac{S(C^\star-1)}{(1-\gamma)^3N}}$. The first term captures the imitation learning regime under which a fast rate $1/ N$ is expected, while the second term deals with the large $C^\star$ regime with a parametric rate $1/ \sqrt{N}$. 
    More interestingly, the dependence on $C^\star$ appears to be $C^\star -1$, which is different from the performance upper bound of VI-LCB in Theorem~\ref{thm:MDP_upperbound_hoffding}. We will comment more on this in the coming section. 
    
\paragraph{Dependence on the effective horizon $1/(1-\gamma)$.}
Comparing the upper bound in Theorem~\ref{thm:MDP_upperbound_hoffding} with the lower bound in Theorem~\ref{thm:MDP_lower}, one sees that the sample complexity of VI-LCB for all regimes of $C^\pi$ is loose by an extra $1/(1-\gamma)^2$ factor in sample complexity. We believe that this extra factor can be shaved by replacing the Hoeffding-based penalty function to a Bernstein-based one and using variance reduction similar to the technique used in the work~\citet{sidford2018near}.

\subsection{What happens when $C^\star\in[1+\Omega(1/N), 1+O(1)]$?}\label{sec:conjecture}
Now we return to the discussion on the dependency on $C^\star$. Ignore the dependency on $1/(1-\gamma)$ for the moment. By comparing Theorems~\ref{thm:MDP_upperbound_hoffding} and~\ref{thm:MDP_lower}, one realizes that VI-LCB is optimal both when $C^\star \geq 1+ \Theta(1)$ and when $C^\star \leq 1 + \Theta(1/N)$.   
However, in the middling regime when $C^\star \in[1+\Omega(1/N), 1+O(1)]$, the upper and lower bounds differ in their dependency on $C^\star$. More specifically, the upper bound presented in Theorem~\ref{thm:MDP_upperbound_hoffding} is $\sqrt{SC^\star/N}$, while the lower bound in Theorem~\ref{thm:MDP_lower} is $S/N + \sqrt{S(C^\star-1)/N}$. 

\paragraph{Technical hurdle.}
We \emph{conjecture} that VI-LCB is optimal even this regime and the current gap is an artifact of our analysis. However, we would like to point out that, although we manage to close the gap in contextual bandits, the case with MDPs is significantly more challenging due to error propagation. Naively applying the decomposition in the contextual bandit case fails to achieve the $C^\star -1$ dependence in this regime. Take the term $T_{5}$ in Figure~\ref{fig:sub_opt_decomposition} as an example. For contextual bandits, given the selection rule is
\begin{equation}
    \hat{\pi}(s) \mapsfrom \argmax_{a} \hat{r}(s,a) - \sqrt{\frac{L}{N(s,a)}},
\end{equation}
it is straightforward to check that as long as the optimal action is taken with much higher probability than the sub-optimal ones, i.e., $\mu(s,\pi^\star(s)) \gg \sum_{a \neq \pi^\star(s)}\mu(s,a)$, the LCB approach will pick the right action regardless of the value gap $r(s,\pi^\star(s)) - r(s,a)$. 
In contrast, due to the recursive update $Q(s,a) \mapsfrom r_t(s,a) - \sqrt{\frac{L}{N_t(s,a)}} + \gamma P^t_{s,a} \cdot V_{t-1}$, LCB picks the right action if 
\begin{equation*}
    r_t(s,\pi^\star(s)) - \sqrt{\frac{L}{N_t(s,\pi^\star(s))}} + \gamma P^t_{s,\pi^\star(s)} \cdot V_{t-1} > r_t(s,a) - \sqrt{\frac{L}{N_t(s,a)}} + \gamma P^t_{s,a} \cdot V_{t-1},\quad  \text{for all }a \neq \pi^\star(s).
\end{equation*}
The presence of the value estimate from the previous step, i.e., $V_{t-1}$ (which is absent in CBs) drastically changes the picture: 
even if we know that $\mu(s,\pi^\star(s)) \gg \sum_{a \neq \pi^\star(s)}\mu(s,a)$ and hence $N_t(s,\pi^\star(s)) \gg N_t(s,a)$, the current analysis does not guarantee the above inequality to hold. It is likely that for the value gap $Q^\star (s, \pi^\star(s)) - Q^\star(s,a)$ to affect whether the LCB algorithm chooses the optimal action. How to study the interplay between the value gap and the policy chosen by LCB forms the main obstacle to obtaining tight performance guarantees when $C^\star\in[1+\Omega(1/N), 1+O(1)]$. 

\paragraph{A confirmation from an episodic MDP.} In Appendix \ref{app:episodic_example} we present an episodic example with the intention to demonstrate that (1) a variant of VI-LCB in the episodic case is able to achieve the optimal dependency on $C^\star$ and hence closing the gap between the upper and the lower bounds, and (2) the tight analysis of the sub-optimality is rather intricate and depends on a delicate decomposition based on the value gap $Q^\star (s, \pi^\star(s)) - Q^\star(s,a)$.

As a preview, we illustrate the episodic MDP with $H=3$ in Figure~\ref{fig:episodic_MDP_example}. 
It turns out that when tackling the term similar to $T_{5}$ in Figure~\ref{fig:sub_opt_decomposition}, a further decomposition based on the value gap is needed. In a nutshell, we decompose the error into two cases: (1) when $Q^\star (s, 1) - Q^\star(s,2)$ is large, and (2) when $Q^\star (s, 1) - Q^\star(s,2)$ is small. Intuitively, in the latter case, the contribution to the sub-optimality is well controlled, and in the former one, we manage to show that VI-LCB selects the right action with high probability. What is more interesting and suprising is that the right threshold for value gap is given by $\sqrt{(C^\star -1)/N}$. Ultimately, this allows us to achieve the optimal dependency on $C^\star$.
\begin{figure}[t]
    \centering
    \scalebox{1}{
    \begin{tikzpicture}[observed/.style={circle, draw=black, fill=black!10, thick, minimum size=10mm},
    hidden/.style={circle, draw=black, thick, minimum size=5mm},
    squarednode/.style={rectangle, draw=red!60, fill=red!5, very thick, minimum size=10mm},
    treenode/.style={rectangle, draw=none, thick, minimum size=10mm},
    rootnode/.style={rectangle, draw=none, thick, minimum size=10mm},
    squarednode/.style={rectangle, draw=none, fill=citeColor!20, very thick, minimum size=7mm},]
    \node[hidden] (s1) at (0,1.5) {$1$};
    \node[hidden] (s2) at (0,0) {$2$};
    \node[hidden] (s3) at (2,1.5) {$3$};
    \node[hidden] (s4) at (2,0) {$4$};
    \node[hidden] (s5) at (4,1.5) {$5$};
    \node[hidden] (s6) at (4,0) {$6$};
    \draw[->, thick, >=stealth] (s1.east) -- (s3.west);
    \draw[->, thick, >=stealth] (s1.south east) -- (s4.north west);
    \draw[->, thick, >=stealth] (s2.north east) -- (s3.south west);
    \draw[->, thick, >=stealth] (s3.east) -- (s5.west);
    \draw[->, thick, >=stealth] (s3.south east) -- (s6.north west);
    \draw[->, thick, >=stealth] (s4.north east) -- (s5.south west);
    \draw[->, thick, >=stealth] (s2.east) -- (s4.west);
    \draw[->, thick, >=stealth] (s4.east) -- (s6.west);
    \end{tikzpicture}}
    \caption{An episodic MDP with $H=3$, two states per level, and two actions $\cA = \{1,2\}$ available from every state. The rewards are assumed to be deterministic and bounded. Action 1 is assumed to be optimal in all states and that $\mu(s,1) \geq 9 \mu(s,2)$.}
    \label{fig:episodic_MDP_example}
\end{figure}
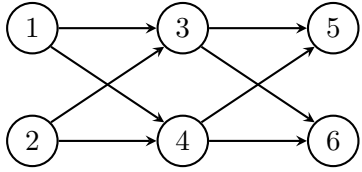

\section{Related work}\label{sec:related_work}
In this section we review additional related works. In Section~\ref{sec:batch_data_assumptions}, we discuss various assumptions on the batch dataset that have been proposed in the literature. In Section~\ref{sec:conservatism}, we review conservative methods in offline RL. We conclude this section by comparing existing lower bounds with the ones presented in this paper.

\subsection{Assumptions on batch dataset}\label{sec:batch_data_assumptions} 
One of the main challenges in offline RL is the insufficient coverage of the dataset caused by lack of online exploration \cite{wang2020statistical, zanette2020exponential,szepesvari2010algorithms} and in particular 
the \textit{distribution shift} in which the occupancy density of the behavior policy and the one induced by the learned policy are different. This effect can be characterized using concentrability coefficients \cite{munos2007performance} which impose bounds on the density ratio (importance weights).

Most concentrability requirements imposed in existing offline RL involve taking a supremum of the density ratio over all state-action pairs and all policies, i.e., $\max_\pi C^\pi$ \cite{scherrer2014approximate,chen2019information,jiang2019value,wang2019neural,liao2020batch,liu2019neural,zhang2020variational} and some definitions are more complex and stronger assuming a bounded ratio per time step \cite{szepesvari2005finite, munos2007performance, antos2008learning,farahmand2010error,antos2007fitted}. A more stringent definition originally proposed by \citet{munos2003error} also imposes exploratoriness on state marginals. This definition is recently used by \citet{xie2020batch} to develop an efficient offline RL algorithm with general function approximation and only realizability. The MABO algorithm proposed by \citet{xie2020batch} and the related algorithms by \citet{feng2019kernel} and \citet{uehara2020minimax} use a milder definition based on a \textit{weighted} norm of density ratios as opposed to the infinity norm. In contrast, to compete with an optimal policy, we only require coverage over states and actions visited by that policy, which is referred to as the ``best'' concentrability coefficient \cite{scherrer2014approximate,geist2017bellman,agarwal2020optimality,xie2020batch}.

Another related assumption is the uniformly lower bounded data distribution. For example, some works consider access to a generative model with an equal number of samples on all state-action pairs \cite{sidford2018near, sidford2018variance, agarwal2020model, li2020breaking}. As discussed before, this assumption is significantly stronger than assuming $C^\star$ is bounded. Furthermore, one can modify the analysis of the LCB algorithm to show optimal data composition dependency in this case as well. 

\subsection{Conservatism in offline RL} \label{sec:conservatism}
In practice, such high coverage assumptions on batch dataset also known as data diversity \cite{levine2020offline} often fail to hold \cite{gulcehre2020rl,agarwal2020optimistic,fu2020d4rl}. Several methods have recently emerged to address such strong data requirements. The first category involves policy regularizers or constraints to ensure closeness between the learned policy and the behavior policy  \cite{fujimoto2019off,wu2019behavior,jaques2019way, peng2019advantage, siegel2020keep, wang2020critic,kumar2019stabilizing,fujimoto2019benchmarking,ghasemipour2020emaq,nachum2019algaedice,Zhang2020GenDICE:,nachum2019dualdice,zhang2020gradientdice}. These methods are most suited when the batch dataset is nearly-expert \cite{wu2019behavior,fu2020d4rl} and sometimes require the knowledge of the behavior policy.

Another category includes the value-based methods. \citet{kumar2020conservative} propose conservative Q-learning through value regularization and demonstrate empirical success. \citet{liu2020provably} propose a variant of fitted Q-iteration with a conservative update called MSB-QI. This algorithm effectively requires the data distribution to be uniformly lower bounded on the state-action pairs visited by any competing policy. Moreover, the sub-optimality of MSB-QI has a $1/(1-\gamma)^4$ horizon dependency compared to ours which is $1/(1-\gamma)^{2.5}$.

The last category involves learning pessimistic models such as \citet{kidambi2020morel}, \citet{yu2020mopo} and \citet{yu2021combo} all of which demonstrate empirical success. From a theoretical perspective, the recent work \citet{jin2020pessimism} studies pessimism in offline RL in episodic MDPs and function approximation setting. The authors present upper and lower bounds for linear MDPs with a suboptimality gap of $dH$, where $d$ is the feature dimension and $H$ is the horizon. Specialized to the tabular case, this gap is equal to $SAH$, compared to ours which is only $H$. Furthermore, this work does not study the adaptivity of pessimism to data composition.

Another recent work by \citet{yin2021near} studies pessimism in tabular MDP setting and proves matching upper and lower bounds. However, their approach requires a uniform lower bound on the data distribution that traces an optimal policy. This assumption is stronger than ours; for example, it requires optimal actions to be included in the states not visited by an optimal policy. Furthermore, this characterization of data coverage does not recover the imitation learning setting: if the behavior policy is exactly equal to the optimal policy, data distribution lower bound can still be small.

\subsection{Information-theoretic lower bounds} \label{sec:literature_lower_bounds}
There exists a large body of literature providing information-theoretic lower bounds for RL under different settings; see e.g., \citet{dann2015sample,krishnamurthy2016pac,jiang2017contextual,jin2018q,azar2013minimax, ma2021minimax, lattimore2012pac, domingues2020episodic,duan2020minimax,zanette2020exponential,wang2020statistical}. In the generative model setting with uniform samples, \citet{azar2013minimax} proves a lower bound on value sub-optimality which is later extended to policy sub-optimality by \citet{sidford2018near}. For the offline RL setting, \citet{kidambi2020morel} prove a lower bound only considering the data and policy occupancy support mismatch without dependency on sample size. \citet{jin2020pessimism} gives a lower bound for linear MDP setting but which does not give a tight dependency on parameters when specialized to the tabular setting. In \citet{yin2020near,yin2021near}, a hard MDP is constructed with a dependency on the data distribution lower bound. In contrast, our lower bounds depend on $C^\star$, which has not been studied in the past, and holds for the entire data spectrum. In the imitation learning setting, \cite{xu2020error} considers discounted MDP setting and shows a lower bound on the performance of the behavior cloning algorithm. We instead present an information-theoretic lower bound for any algorithm for $C^\star = 1$ which is based on adapting the construction of \citet{rajaraman2020toward} to the discounted case.

\section{Discussion}
In this paper, we propose a new batch RL framework based on the single policy concentrability coefficient (e.g., $C^\star$) that smoothly interpolates the two extremes of data composition encountered in practice, namely the expert data and uniform coverage data. Under this new framework, we pursue the statistically optimal algorithms that can even be implemented without the knowledge of the exact data composition. More specifically, focusing on the lower confidence bound (LCB) approach inspired by the principle of pessimism, we find that LCB is adaptively minimax optimal for addressing the offline contextual bandit problems and the optimal rate naturally bridges the $1/N$ rate when data is close to following the expert policy and the $1/\sqrt{N}$ rate in the typical offline RL case. Here $N$ denotes the number of samples in the batch dataset. We also investigate the LCB approach in the offline multi-armed bandit problems and Markov decision processes. The message is somewhat mixed. For bandits, LCB is shown to be optimal for a wide range of data compositions, however, LCB without the knowledge of data composition, is provably non-adaptive in the near-expert data regime. When it comes to MDPs, we show that LCB is adaptively rate-optimal when $C^\star$ is extremely close to 1, and when $C^\star \geq 1 + $ constant. Contrary to bandits, we conjecture that LCB is optimal across the spectrum of data composition. 

Under the new framework, there exist numerous avenues for future study. Below, we single out a few of them. 
\begin{itemize}[leftmargin=*]
    \item \textbf{Closing the gap in MDPs.} It is certainly interesting to provide tighter upper bounds for LCB in the MDP case when $C \in (1 + \Omega(1/N), 1+o(1))$. This regime is of particular interest when we believe that a significant fraction of the data comes from the optimal policy. 
    \item \textbf{Improving the dependence on the effective horizon.}  There is a $1/(1-\gamma)^2$ gap in the sample complexity for solving infinite-horizon discounted MDPs. We believe that using a Bernstein-type penalty in conjunction with a variance reduction technique or data reuse across iterations may help address this issue.
    \item \textbf{Incorporating function approximation.} In this paper we focus on the tabular case only. It would be of particular interest and importance to extend the analysis and algorithms to function approximation setting.
    \item \textbf{Investigating other algorithms.} In this paper we study a conservative method based on lower confidence bound. Other conservative methods such as algorithms that use value regularization may also achieve adaptivity and/or minimax optimality.
\end{itemize}

\section*{Acknowledgements}
The authors are grateful to Nan Jiang, Aviral Kumar, Yao Liu,
and Zhaoran Wang for helpful discussions and suggestions. PR was partially supported by the Open Philanthropy Foundation and the Leverhulme Trust. BZ and JJ were partially supported by NSF Grants IIS-1901252, CCF-1909499, and DMS-2023505. 

\bibliographystyle{plainnat}
\bibliography{references}

\newpage

\appendix

\section{Proofs for multi-armed bandits}
In Section~\ref{app:proof_lower_bd_largest_empirical}, we prove Proposition~\ref{lem:lower_bd_largest_empirical} that demonstrates the failure of the best empirical arm when solving offline MABs.
Section~\ref{app:proof_MAB_LCB_upper} is devoted to the proof of Theorem~\ref{thm:MAB_LCB_upper}, which supplies the 
performance upper bound of the LCB approach. This upper bound is accompanied by a minimax lower bound given in Section~\ref{app:proof_bandit_lower_bound}. In the end, we provably show the lack of adaptivity of the LCB approach in Section~\ref{app:proof_IL_bandit}. 

\subsection{Proof of Proposition~\ref{lem:lower_bd_largest_empirical}}\label{app:proof_lower_bd_largest_empirical}
We start by introducing the bandit instance under consideration. Set $|\cA| = 2$, $a^\star = 1$, $\mu(1) = (N-1 )/ N$, and $\mu(2) =1/N$.  As for the reward distributions, for the optimal arm $a^\star = 1$, we let $R(1) = 2\epsilon$ almost surely. In contrast, for arm $2$ we set 
\begin{align*}
    R(2) = \begin{cases}
    2.1\epsilon, & \text{w.p. }0.5, \\ 
    0, & \text{w.p. } 0.5.
    \end{cases}
\end{align*}
It is easy to check that indeed $a^\star = 1$ is the optimal arm to choose. 
Our goal is to show that for this particular bandit problem, given $N$ offline data from $\mu$ and $R$, the empirical best arm $\hat{a}$ will perform poorly with 
high probability. 

To see this, consider the following event 
\begin{align*}
    \mathcal{E}_{1} \coloneqq \{ N(2) = 1 \}.
\end{align*}
We  have
\begin{align*}
    \mathbb{P}(\mathcal{E}_{1}) & = N\cdot \mu(1)^{N-1}\cdot \mu(2)= \left(1-1/N\right)^{N-1}.
\end{align*}
As long as $N$ is sufficiently large (say $N \geq 500$), we  have $ \mathbb{P}(\mathcal{E}_{1}) \geq 0.36$ for any $0 \leq n \leq N$, and thus  $\mathbb{P}(\mathcal{E}_{1})\geq 0.36$.

Now we are in position to develop a performance lower bound for the empirical best arm $\hat{a}$. 
By construction, we have $r(1) - r(2)= 0.95\epsilon$. 
Therefore the sub-optimality is given by 
\begin{align*}
    \mathbb{E}_\cD[r(a^\star) - r(\hat a)] & = 0.95\epsilon \cdot \mathbb{P}(\hat a \neq a^\star) \nonumber \\ 
    & \geq 0.95\epsilon \cdot  \mathbb{P}(\mathcal{E}_1\cap \hat r(2) = 2.1\epsilon) \nonumber \\
    & \geq 0.95\epsilon\cdot 0.18>0.1 \epsilon.
\end{align*}
Rescaling the value of $\epsilon$ finishes the proof.

\subsection{Proof of Theorem~\ref{thm:MAB_LCB_upper}}\label{app:proof_MAB_LCB_upper}

Before embarking on the main proof, we record two useful lemmas.
The first lemma sandwiches the true mean reward by the empirical one and the 
penalty function, which directly follows from Hoeffding's inequality and a union bound. For completeness, we provide the proof at the end of this subsection.

\begin{lemma}\label{lemma:hoeffding-MAB}
With probability at least $1-\delta$, we have
\begin{equation}\label{eq:bonus-sandwich-MAB}
    \hat{r}(a) - b(a) \leq r(a) \leq \hat{r}(a) + b(a), \quad \text{for all } 1\leq a\leq |\mathcal{A}|.
\end{equation}
\end{lemma}

The second one is a simple consequence of the Chernoff bound for binomial random variables.
\begin{lemma}\label{lem:N-lower-bound-MAB}
With probability at least $1 - \exp(-N\mu(a^\star) / 8)$, one has
\begin{equation}\label{eq:N-lower-bound-MAB}
    N(a^\star) \geq \frac{1}{2}N\mu(a^\star). 
\end{equation}
\end{lemma}

Denote by $\mathcal{E}$ the event that both relations~\eqref{eq:bonus-sandwich-MAB} and~\eqref{eq:N-lower-bound-MAB} hold.
Conditioned on $\mathcal{E}$, one has 
\[
r(a^\star)\leq\hat{r}(a^\star)+b(a^\star)=\hat{r}(a^\star)-b(a^\star)+2b(a^\star).
\]
In view of the definition of $\hat{a}$, we have $\hat{r}(a^\star)-b(a^\star)\leq\hat{r}(\hat{a})-b(\hat{a})$,
and hence 
\[
r(a^\star)\leq\hat{r}(\hat{a})-b(\hat{a})+2b(a^\star)\leq r(\hat{a})+2b(a^\star),
\]
where the last inequality holds under the event $\mathcal{E}$ (in particular the bound~\eqref{eq:bonus-sandwich-MAB} on $\hat{a}$). Now
we are left with the term $b(a^\star)$. It suffices to lower bound $N(a^\star)$. Note that the event $\mathcal{E}$ (cf.~the lower bound~\eqref{eq:N-lower-bound-MAB})
ensures that 
\[
N(a^\star)\geq\frac{1}{2}N\mu(a^\star)\geq \frac{N}{2C^\star}>0.
\]
As a result, we conclude 
\[
b(a^\star)=\sqrt{\frac{\log(2|\mathcal{A}|/\delta)}{2N(a^\star)}}\leq\sqrt{\frac{\log(2|\mathcal{A}|/\delta)}{N\mu(a^\star)}},
\]
which further implies 
\begin{equation}
r(a^\star)\leq r(\hat{a})+2\sqrt{\frac{\log(2|\mathcal{A}|/\delta)}{N\mu(a^\star)}}\label{eq:clean-event-consequence-MAB}
\end{equation}
whenever~the event $\mathcal{E}$ holds. 
It is easy to check that under the assumption $N \geq 8 C^\star \log(1/\delta)$, we have $\mathbb{P}(\mathcal{E}) \geq 1 - 2\delta $. This finishes the proof of the high probability claim. 

In the end, we can compute the expected
sub-optimality as 
\begin{align*}
\mathbb{E}_\cD[r(a^\star)-r(\hat{a})] & =\mathbb{E}_\cD[\left(r(a^\star)-r(\hat{a})\right)1\{\mathcal{E}\}]+\mathbb{E}_\cD[\left(r(a^\star)-r(\hat{a})\right)1\{\mathcal{E}^{c}\}]\\
 & \leq2\sqrt{\frac{\log(2|\mathcal{A}|/\delta)}{N\mu(a^\star)}}\mathbb{P}(\mathcal{E})+\mathbb{P}(\mathcal{E}^{c}).
\end{align*}
Here the inequality uses the bound~\eqref{eq:clean-event-consequence-MAB} and
the fact that $r(a^\star)-r(\hat{a})\leq1$. 
We continue bounding the sub-optimality by
\begin{equation*}
    \mathbb{E}_\cD[r(a^\star)-r(\hat{a})] \leq 2\sqrt{\frac{\log(2|\mathcal{A}|/\delta)}{N\mu(a^\star)}} + 2\delta \leq 2\sqrt{\frac{C^\star\log(2|\mathcal{A}|/\delta)}{N}} + 2\delta.
\end{equation*}
Here the last relation uses $\mu(a^\star) \geq 1 / C^\star$. Taking $\delta=1/N$ completes the proof.

\begin{proof}[Proof of Lemma~\ref{lemma:hoeffding-MAB}]
Consider a fixed action $a$. If $N(a)$ = 0, one trivially has $\hat{r}(a) - b(a) = -1 \leq r(a) \leq \hat{r}(a) + b(a) = 1$. 
When $N(a) > 0$, applying Hoeffding's inequality, one sees that 
\begin{equation*}
    \mathbb{P}\left( \left|\hat{r}(a)-r(a)\right|\geq\sqrt{\frac{\log(2|\mathcal{A}| / \delta)}{2N(a)}} \Large\mid N(a) \right)\leq \frac{\delta}{|\mathcal{A}|}.
\end{equation*}
Since this claim holds for all possible $N(a)$, we have for any fixed action $a$
\begin{equation*}
    \mathbb{P}\left( \left|\hat{r}(a)-r(a)\right|\geq b(a)\right)\leq \frac{\delta}{|\mathcal{A}|}.
\end{equation*}
A further union bound over the action space yields the advertised claim. 
\end{proof}

\subsection{Proof of Theorem~\ref{thm.bandit_lower_bound}}\label{app:proof_bandit_lower_bound}
We separate the proof into two cases: $C^\star \geq 2$ and $C^\star \in (1,2)$. 
For both cases, our lower bound proof relies on the classic Le Cam's two-point method \cite{yu1997assouad, le2012asymptotic}. In essence, we construct two MAB instances in the family $\mathsf{MAB}(C^\star)$ with different optimal rewards that are difficult to distinguish given the offline dataset.

\paragraph{The case of $C^\star \geq 2$.}
We consider a simple two-armed bandit. For the behavior policy, we set 
$\mu(2) = 1/C^\star$ and $\mu(1) = 1 - 1/C^\star$. Since we are constructing lower bound instances, it suffices to consider Bernoulli distributions supported on $\{0,1\}$. In particular, we consider the following two possible sets for the Bernoulli means 
\begin{align*}
    f_{1} = (\frac{1}{2}, \frac{1}{2}-\delta); \qquad    f_{2} = (\frac{1}{2}, \frac{1}{2}+\delta),
\end{align*}
with $\delta \in [0,1/4]$. Indeed, $(\mu, f_{1}), (\mu,f_{2}) \in \mathrm{MAB}(C^\star)$ with the proviso that $C^\star \geq 2$.
Denote the loss/sub-optimality of an estimator $\hat{a}$ to be 
\begin{equation}
    \mathcal{L}(\hat{a};f) \coloneqq  r(a^\star) - r(\hat{a}),
\end{equation}
where the optimal action $a^\star$ implicitly depends on the reward distribution $f$. Clearly, for any estimator $\hat{a}$, we have 
\begin{align*}
    \mathcal{L}(\hat{a};f_1) +\mathcal{L}(\hat{a};f_2) \geq \delta.     
\end{align*}
Therefore Le Cam's method tells us that 
\begin{align*}
    \inf_{\hat a}\sup_{(\mu, R) \in \mathsf{MAB}(C^\star)} \mathbb{E}_\cD [ r(a^\star) - r(\hat{a})] \geq \inf_{\hat{a}} \sup_{f\in {f_1, f_2}} \; \mathbb{E}_{\cD} [\mathcal{L}(\hat{a};f)] \geq \frac{\delta}{4}\cdot \exp(-\mathsf{KL}(\mathbb{P}_{\mu \otimes f_{1}} \| \mathbb{P}_{\mu \otimes f_{2}})).
\end{align*}
Here $\mathsf{KL}(\mathbb{P}_{\mu \otimes f_{1}} \| \mathbb{P}_{\mu \otimes f_{2}})$ denotes the KL divergence between the two MAB instances with $N$ samples. 
Direct calculations yield 
\begin{align*}
\mathsf{KL}(\mathbb{P}_{\mu \otimes f_{1}} \| \mathbb{P}_{\mu \otimes f_{2}})  \leq \frac{N \mathsf{KL}(\mathbb{P}_{f_{1}} \| \mathbb{P}_{f_{2}})}{C^\star} 
 \leq \frac{N(2\delta)^2}{C^\star(1/4-\delta^2)} 
 \leq 200N\delta^2/C^\star.
\end{align*}
Here we use the fact that for two Bernoulli distribution, $\mathsf{KL}(\mathrm{Bern}(p) \| \mathrm{Bern}(q))\leq (p-q)^2/[q(1-q)]$ and that $\delta \in [0,1/4]$. Taking
\begin{equation*}
    \delta = \min \left\{\frac{1}{4}, \sqrt{\frac{C^\star}{N}}\right\}
\end{equation*}
yields the desired lower bound for $C^\star \geq 2$.

\paragraph{The case of $C^\star\in(1, 2)$.}
Recall that when $C^\star \geq 2$, we construct the same behavior distribution $\mu$ for two different reward distributions $f_1, f_2$. In contrast, in the case of $C^\star\in[1, 2)$, we construct instances that are different in both the reward distributions as well as the behavior distribution. More specifically, let $\mu_1(1) = 1/C^\star$, $\mu_1(2) = 1-1/C^\star$, $f_1 = (\frac{1}{2}+\delta, \frac{1}{2})$ for some $\delta>0$ which will be specified later. Similarly, we let $\mu_2(1) = 1-1/C^\star$, $\mu_2(2) = 1/C^\star$, $f_2 = (\frac{1}{2}, \frac{1}{2}+\delta)$. It is straightforward to check that  
$(\mu_1, f_{1}), (\mu_2,f_{2}) \in \mathsf{MAB}(C^\star)$. 
Clearly, for any estimator $\hat{a}$, we have 
\begin{align*}
    \mathcal{L}(\hat{a};f_1) +\mathcal{L}(\hat{a};f_2) \geq \delta.     
\end{align*}
Again, applying Le Cam's method, we have
\begin{align}
    \inf_{\hat a}\sup_{(\mu, R) \in \mathsf{MAB}(C^\star)} \mathbb{E}_\cD [ r(a^\star) - r(\hat{a})] \geq \frac{\delta}{4}\cdot \exp(-\mathsf{KL}(\mathbb{P}_{\mu_1 \otimes f_{1}} \| \mathbb{P}_{\mu_2 \otimes f_{2}})).
\end{align}
Note that 
\begin{align*}
\mathsf{KL}(\mathbb{P}_{\mu_1 \otimes f_{1}} \| \mathbb{P}_{\mu_2 \otimes f_{2}}) & \leq N\cdot \Big(\frac{\frac{1}{2}+\delta}{C^\star} \log(\frac{1+2\delta}{C^\star-1}) + \frac{\frac{1}{2}-\delta}{C^\star} \log(\frac{1-2\delta}{C^\star-1})  \\
& \quad + \frac{1-\frac{1}{C^\star}}{2}\log(\frac{C^\star-1}{1+2\delta}) + \frac{1-\frac{1}{C^\star}}{2}\log(\frac{C^\star-1}{1-2\delta})\Big) \\ 
& = N \cdot \left(\left(\frac{1+\delta}{C^\star}-\frac{1}{2}\right)\log\left(\frac{1+2\delta}{C^\star-1}\right) + \left(\frac{1-\delta}{C^\star}-\frac{1}{2}\right)\log\left(\frac{1-2\delta}{C^\star-1}\right)\right).
\end{align*}
Taking $\delta = \frac{2-C^\star}{2}$, we get $\mathsf{KL}(\mathbb{P}_{\mu_1 \otimes f_{1}} \| \mathbb{P}_{\mu_2 \otimes f_{2}})\leq N\cdot \frac{2-C^\star}{C^\star}\cdot \log\left(\frac{2}{C^\star-1}\right)$. Thus we know that 
\begin{align}
    \inf_{\hat a}\sup_{(\mu, R) \in \mathsf{MAB}(C^\star)} \mathbb{E}_\cD [ r(a^\star) - r(\hat{a})] \gtrsim \exp\left(-(2-C^\star)\cdot\log\left(\frac{2}{C^\star-1}\right)\cdot N\right).
\end{align}
This finishes the proof of the lower bound for $C^\star\in (1, 2)$. 

\subsection{Proof of Proposition~\ref{thm.imitation_bandit_bound}}\label{app:proof_IL_bandit}
To begin with, we have $\mathbb{E}[r(a^\star) - r(\hat a)] \leq \mathbb{P}(\hat a \neq a^\star)$, where we have used the fact that the rewards are bounded between 0 and 1. Thus it is sufficient to control
$\mathbb{P}(\hat a \neq a^\star)$, which obeys
\begin{equation*}
    \mathbb{P}(\hat a \neq a^\star) =\mathbb{P}(\exists a\neq a^\star, N(a)\geq N({a^\star})) \leq   \mathbb{P}(N-N({a^\star})\geq N({a^\star})) =  \mathbb{P}(N({a^\star})\leq \frac{N}{2}).
\end{equation*}
Applying the Chernoff bound for binomial random variables yields 
\begin{equation*}
\mathbb{P}(N({a^\star})\leq \frac{N}{2}) \leq \exp\left(-N\cdot\mathsf{KL}\left(\mathrm{Bern}\left(\tfrac{1}{2}\right) \; \Big\| \; \mathrm{Bern}\left(\tfrac{1}{C^\star}\right)\right)\right).    
\end{equation*}
Taking the previous steps collectively to arrive at the desired conclusion.

\subsection{Proof of Theorem~\ref{thm:LCB_lower_bandit}}\label{app:proof_LCB_lower_bandit}

We prove the case when $C^\star = 1.5$ and when $C^\star = 6$ separately.

\paragraph{The case when $C^\star=1.5$.}
We begin by introducing the MAB problem. 
\subparagraph{The bandit instance.} Consider a two-armed bandit problem with the optimal arm denoted by $a^\star$ and the sub-optimal arm $a$. We set $\mu(a^\star) = 1/{C^\star}$, and $\mu(a) = 1 - 1/{C^\star}$ in accordance with the requirement $1 / \mu(a^\star) \leq C^\star$. We consider the following reward distributions: the optimal arm $a^\star$ has a deterministic reward equal to 1/2 whereas the sub-optimal arm has a reward distribution of $\mathsf{Bern}(1/2-g)$ for some $g\in(0, 1/3)$, which will be specified momentarily. It is straightforward to check that the arm $a^\star$ is indeed optimal and the MAB problem $(\mu, R)$ belongs to $\mathsf{MAB}(C^\star)$. 

\subparagraph{Lower bounding the performance of LCB.} 
For the two-armed bandit problem introduced above, we have
\begin{align}
    \mathbb{E}_\cD[r(a^\star) - r(\hat a) ]&= g \cdot \mathbb{P}(\text{LCB chooses arm }a) \nonumber \\
    &= g \sum_{k=0}^{N} \mathbb{P}(\text{LCB chooses arm }a \mid N(a) = k) \mathbb{P} (N(a) = k) \nonumber \\
    &\geq g \sum_{k=N\mu(a) / 2}^{2 N \mu(a)} \mathbb{P}(\text{LCB chooses arm }a \mid N(a) = k) \mathbb{P} (N(a) = k),
    \label{eq:LCB-lower-bound-bandit-first}
\end{align}
where we restrict ourselves to the event 
\begin{equation*}
    \mathcal{E} \coloneqq \{\frac{1}{2}N \mu(a) \leq N(a) \leq 2 N \mu(a)\}.
\end{equation*}
It turns out that when $1 \leq  k \leq 2 N \mu(a)$, one has 
\begin{equation}\label{eq:LCB-lower-bound-bandit}
    \mathbb{P}(\text{LCB chooses arm }a \mid N(a) = k) \geq \frac{1}{\sqrt{4 N \mu(a)}} \cdot \exp\left(-\frac{(g \sqrt{2 N \mu(a)}+\sqrt{L})^2}{\frac{1}{4}-g^2}\right). 
\end{equation}
Combine inequalities~\eqref{eq:LCB-lower-bound-bandit-first} and~\eqref{eq:LCB-lower-bound-bandit} to obtain
\begin{align*}
    \mathbb{E}_\cD[r(a^\star) - r(\hat a) ] \geq g \frac{1}{\sqrt{4 N \mu(a)}} \cdot \exp\left(-\frac{(g \sqrt{2 N \mu(a)}+\sqrt{L})^2}{\frac{1}{4}-g^2}\right)  \mathbb{P}(\mathcal{E}). 
\end{align*}
Setting $g = \min \{1/3, \sqrt{L / (2 N \mu(a))}\}$ yields
\begin{align*}
    \mathbb{E}_\cD[r(a^\star) - r(\hat a) ] &\geq  \frac{\min\left(\sqrt{L / (2 N \mu(a))}, \frac{1}{3}\right)}{\sqrt{4N\mu(a)}} \cdot \exp\left(-32L\right)  \mathbb{P}(\mathcal{E}) \\
    &\geq \min\left(\frac{\sqrt{L}}{8N\mu (a)}, \frac{1}{12\sqrt{N\mu(a)}}\right) \cdot \exp\left(-32L\right),
\end{align*}
where the last inequality uses Chernoff's bound, i.e., $\mathbb{P}(\mathcal{E}) \geq 1 - 2\exp(-N\mu(a) / 8) \geq \frac{1}{2}$. Substituting the definition of $L$ and $\mu(a)$ completes the proof. 

\begin{proof}[Proof of the lower bound~\eqref{eq:LCB-lower-bound-bandit}]
By the definition of LCB, we have
\begin{align*}
   \mathbb{P}(\text{LCB chooses arm }a \mid N(a) = k) & =  \mathbb{P}\left( 1/2 - \sqrt{L/N(a^\star)}\leq \hat r(a) - \sqrt{L/N(a)}\mid N(a) = k\right) \\
   & \geq \mathbb{P}\left(  \hat r(a)  \geq  1/2 +  \sqrt{L/N(a)}\mid N(a) = k\right)  \\ 
   & \geq \frac{1}{\sqrt{2k}} \cdot \exp\left(-k\cdot \mathsf{KL}\left(\frac{1}{2}-\sqrt{\frac{L}{k}} \Big\| \frac{1}{2}+g\right)\right) \\
   & \geq \frac{1}{\sqrt{2k}} \cdot \exp\left(-\frac{k(g+\sqrt{\frac{L}{k}})^2}{\frac{1}{4}-g^2}\right).
\end{align*}
Here, the penultimate inequality comes from a lower bound for Binomial tails~\cite{robert1990ash} and the last inequality uses the elementary fact that  $\mathsf{KL}(p \| q)\leq (p-q)^2/q(1-q)$. 
One can easily see that the probability lower bound is decreasing in $k$ and hence when $N(a)=k \leq 2 N \mu(a)$, we have
\begin{equation*}
    \mathbb{P}(\text{LCB chooses the arm }a \mid N(a) = k) \geq \frac{1}{\sqrt{4 N \mu(a)}} \cdot \exp\left(-\frac{(g \sqrt{2 N \mu(a)}+\sqrt{L})^2}{\frac{1}{4}-g^2}\right).
\end{equation*}
This completes the proof. 
\end{proof}

\paragraph{The case when $C^\star= 6$.}
We now prove the lower bound for the case of $C^\star=6$.
\subparagraph{The bandit instance.} Consider a two-armed bandit problem with $\mu(a^\star) = \frac{1}{C^\star}$ for the optimal arm and $\mu(a) = 1 - \frac{1}{C^\star}$ for the sub-optimal arm, which satisfies the concentrability requirement. We set the following reward distributions: the optimal arm $a^\star$ is distributed according to $\mathsf{Bern}(1/2)$ and the sub-optimal arm has a deterministic reward equal to $1/2-g$ for some $g \in (0,1/2)$, which will be specified momentarily. It is immediate that $a^\star$ is optimal in this construction and that the MAB problem $(\mu, R)$ belongs to $\mathsf{MAB}(C^\star)$.

\subparagraph{Lower bounding the performance of LCB.} 
Similar arguments as before give
\begin{align}
    \mathbb{E}_\cD[r(a^\star) - r(\hat a) ] \geq g \sum_{k=N\mu(a^\star) / 2}^{2 N \mu(a^\star)} \mathbb{P}(\text{LCB chooses arm }a \mid N(a^\star) = k) \mathbb{P} (N(a^\star) = k),
    \label{eq:LCB-lower-bound-bandit-first-C-2}
\end{align}
where we restrict ourselves to the event (with abuse of notation)
\begin{equation*}
    \mathcal{E} \coloneqq \{\frac{1}{2}N \mu(a^\star) \leq N(a^\star) \leq 2 N \mu(a^\star)\}.
\end{equation*}
By the definition of LCB, when $C^\star= 6$ and $\frac{1}{2}N \mu(a^\star)\leq  k \leq 2 N \mu(a^\star)\leq \frac{1}{3}N$, one has 
\begin{align*}
   \mathbb{P}(\text{LCB chooses arm }a \mid N(a^\star) = k) & =  \mathbb{P}\left( \hat{r}(a^\star) - \sqrt{L/N(a^\star)}\leq   \frac{1}{2}-g - \sqrt{L/N(a)}\mid N(a^\star) = k\right) \\
   & = \mathbb{P}\left(  \hat r(a^\star)  \leq  1/2 -g +\sqrt{L/k} -   \sqrt{L/(N-k)}\mid N(a^\star) = k\right)  \\ 
    & \geq \mathbb{P}\left(  \hat r(a^\star)  \leq  1/2 -g +\sqrt{3L/N} -   \sqrt{3L/(2N)}\mid N(a^\star) = k\right)  \\ 
     & > \mathbb{P}\left(  \hat r(a^\star)  \leq  1/2 -g +\sqrt{\frac{L}{4N}} \mid N(a^\star) = k\right).
\end{align*}
We set $g = \min\{\sqrt{L / (4N)}, 1/2\}$. 
Under this choice of $g$, we always have
\begin{equation}\label{eq:LCB-lower-bound-bandit-C-2}
    \mathbb{P}(\text{LCB chooses arm }a \mid N(a^\star) = k) \geq \frac{1}{2}.
\end{equation}
Combine the inequalities~\eqref{eq:LCB-lower-bound-bandit-first-C-2} and~\eqref{eq:LCB-lower-bound-bandit-C-2} to obtain
\begin{align*}
    \mathbb{E}_\cD[r(a^\star) - r(\hat a) ] \geq g \cdot \frac{1}{2}\cdot \mathbb{P}(\mathcal{E}) \geq \frac{\min(1, \sqrt{L/N})}{8}.
\end{align*} 

\section{Proofs for contextual bandits}
In Section~\ref{app:proof_LCB_CB_upper_bound}, we prove the sub-optimality guarantee of the LCB approach for contextual bandits stated in Theorem~\ref{thm:LCB_CB_upper_bound}. 
In Section~\ref{app:proof_CB_lower} we prove Theorem~\ref{theorem:lower_bound_offline_CB}---a minimax lower bound for contextual bandits. In the end, we prove the failure of the most played arm approach in Section~\ref{app:proof_CB_IL_lower}.

\subsection{Proof of Theorem \ref{thm:LCB_CB_upper_bound}}\label{app:proof_LCB_CB_upper_bound}

We prove a stronger version of Theorem~\ref{thm:LCB_CB_upper_bound}: Fix a deterministic expert policy $\pi$ that is not necessarily optimal. We 
assume that \[
\max_{s}\frac{\rho(s)}{\mu(s,\pi(s))}\leq C^\pi.
\] Setting $\delta = 1/N$, the policy $\hat{\pi}$ returned by Algorithm \ref{alg:CB-LCB} obeys
\begin{align*}
    \E_\cD[J(\pi)-J(\hat{\pi})] \lesssim \min \left( 1, \widetilde{O} \left(\sqrt{\frac{S (C^\pi-1) }{N}} + \frac{S}{N}\right) \right).
\end{align*}
The statement in Theorem~\ref{thm:LCB_CB_upper_bound} can be recovered when we take $\pi=\pi^\star$.

We begin with defining a good event 
\begin{equation}\label{eq:event-CB}
\mathcal{E}\coloneqq\{\forall s,a: \; |r(s,a)-\hat{r}(s,a)|\leq b(s,a)\},
\end{equation}
on which the penalty function $b(s,a)$ provides a valid upper bound on
the reward estimation error $r(s,a)-\hat{r}(s,a)$. With this definition
in place, we state a key decomposition of the sub-optimality of the
LCB method: 
\begin{align*}
&\mathbb{E}_\cD\left[\sum_{s}\rho(s)\left[r(s,\pi(s))-r(s,\hat{\pi}(s))\right]\right] \\
&\quad =\mathbb{E}_\cD\left[\sum_{s}\rho(s)\left[r(s,\pi(s))-r(s,\hat{\pi}(s))\right]\mathbbm{1}\{N(s,\pi(s))=0\}\right]\eqqcolon T_{1}\\
 & \quad \quad+\mathbb{E}_\cD\left[\sum_{s}\rho(s)\left[r(s,\pi(s))-r(s,\hat{\pi}(s))\right]\mathbbm{1}\{N(s,\pi(s))\geq1\}\mathbbm{1}\{\mathcal{E}\}\right]\coloneqq T_{2}\\
 & \quad \quad+\mathbb{E}_\cD\left[\sum_{s}\rho(s)\left[r(s,\pi(s))-r(s,\hat{\pi}(s))\right]\mathbbm{1}\{N(s,\pi(s))\geq1\}\mathbbm{1}\{\mathcal{E}^{c}\}\right]\coloneqq T_{3}.
\end{align*}
In words, the term $T_{1}$ corresponds to the error induced by missing
mass, i.e., when the expert action  $\pi(s)$ is not seen
in the data $\mathcal{D}$. The second term $T_{2}$ denotes the error
when the good event $\mathcal{E}$ takes place. The last term $T_{3}$ denotes the sub-optimality incurred under the complement event $\mathcal{E}^{c}$.

To avoid cluttered notation, we denote $L\coloneqq 2000\sqrt{2\log (S|\mathcal{A}|N)}$ such that 
$b(s,a) = \sqrt{L / N(s,a)}$ when $N(s,a) \geq 1$. 
These three error terms obey the following upper bounds, whose proofs
are provided in subsequent subsections:
\begin{subequations}
\begin{align}
T_{1} & \leq\frac{4SC^{\pi}}{9N};\label{eq:missing-mass-CB}\\
T_{2} & \lesssim \frac{SC^{\pi}}{N}L+\sqrt{\frac{S(C^{\pi}-1)}{N}L}+\frac{1}{N^{9}};\label{eq:hard-term-CB}\\
T_{3} & \leq\frac{1}{N}.\label{eq:event-CB-proof}
\end{align}
\end{subequations}
Combining the above three bounds together with the fact that  $\E_\cD[J(\pi)-J(\hat{\pi})]\leq 1$ yields that
\begin{align*}
     \E_\cD[J(\pi)-J(\hat{\pi})] \lesssim \min \left( 1, \widetilde{O} \left(\sqrt{\frac{S (C^\pi-1) }{N}} + \frac{SC^\pi}{N}\right) \right).
\end{align*}
Note that if $C^\pi\geq 2$, the first term $\sqrt{\frac{S (C^\pi-1) }{N}}$ always dominates. Conversely, if $C^\pi<2$, we can omit the extra $C^\pi$ in the second term $\frac{SC^\pi}{N}$. This gives
the desired claim in Theorem~\ref{thm:LCB_CB_upper_bound}.

\subsubsection{Proof of the bound~\eqref{eq:missing-mass-CB} on $T_{1}$}

Since $r(s,\pi(s))-r(s,\hat{\pi}(s))\leq1$ for any $\hat{\pi}(s)$,
one has 
\begin{align*}
T_{1} & \leq\mathbb{E}_\cD\left[\sum_{s}\rho(s)\mathbbm{1}\{N(s,\pi(s))=0\}\right] =\sum_{s}\rho(s)\mathbb{P}(N(s,\pi(s))=0)\\
 & =\sum_{s}\rho(s)(1-\mu(s,\pi(s)))^{N}.
\end{align*}
Recall the assumption that $\max_{s}\frac{\rho(s)}{\mu(s,\pi(s))}\leq C^{\pi}$.
We can continue the upper bound of $T_{1}$ to obtain 
\[
T_{1}\leq\sum_{s}C^{\pi}\mu(s,\pi(s))(1-\mu(s,\pi(s)))^{N}\leq\sum_{s}C^{\pi}\frac{4}{9N}=\frac{4}{9N}SC^{\pi}.
\]
Here, the last inequality holds since $\max_{x\in[0,1]}x(1-x)^{N}\leq4/(9N)$. 

\subsubsection{Proof of the bound~\eqref{eq:hard-term-CB} on $T_{2}$}
For any state $s \in \cS$, define the total mass on sub-optimal actions to be
\[
\bar{\mu}(s)\coloneqq\sum_{a:a\neq\pi(s)}\mu(s,a).
\]
We can then partition the state space into the following three disjoint
sets: \begin{subequations}
\begin{align}
\mathcal{S}_{1} & \coloneqq\left\{ s\mid\rho(s)<\frac{2 C^{\pi}L}{N}\right\} ,\label{eq:defn-S-1-CB}\\
\mathcal{S}_{2} & \coloneqq\left\{ s\mid\rho(s)\geq\frac{2 C^{\pi}L}{N},\mu(s,\pi(s))\geq10\bar{\mu}(s)\right\},\label{eq:defn-S-2-CB}\\
\mathcal{S}_{3} & \coloneqq\left\{ s\mid \rho(s)\geq\frac{2 C^{\pi}L}{N},\mu(s,\pi(s))<10\bar{\mu}(s)\right\} .\label{eq:defn-S-3-CB}
\end{align}
\end{subequations}
The set $\mathcal{S}_{1}$ includes the
states that are ``less important'' in evaluating the performance of LCB.
The set $\mathcal{S}_{2}$ captures the states for which the expert
action $\pi(s)$ is drawn more frequently under the behavior
distribution $\mu$. 

With this partition at hand, we can decompose
the term $T_{2}$ accordingly:
\begin{align*}
T_{2} & =\sum_{s\in\mathcal{S}_{1}}\rho(s)\mathbb{E}_\cD\left[\left[r(s,\pi(s))-r(s,\hat{\pi}(s))\right]\mathbbm{1}\{N(s,\pi(s))\geq1\}\mathbbm{1}\{\mathcal{E}\}\right]\eqqcolon T_{2,1}\\
 & \quad+\sum_{s\in\mathcal{S}_{2}}\rho(s)\mathbb{E}_\cD\left[\left[r(s,\pi(s))-r(s,\hat{\pi}(s))\right]\mathbbm{1}\{N(s,\pi(s))\geq1\}\mathbbm{1}\{\mathcal{E}\}\right]\eqqcolon T_{2,2}\\
 & \quad+\sum_{s\in\mathcal{S}_{3}}\rho(s)\mathbb{E}_\cD\left[\left[r(s,\pi(s))-r(s,\hat{\pi}(s))\right]\mathbbm{1}\{N(s,\pi(s))\geq1\}\mathbbm{1}\{\mathcal{E}\}\right]\eqqcolon T_{2,3}.
\end{align*}
The proof is completed by observing the following three upper bounds:
\[
T_{2,1}\leq\frac{2SC^{\pi}L}{N};\qquad T_{2,2}\lesssim\frac{1}{N^{9}};\qquad T_{2,3}\lesssim \sqrt{\frac{C^{\pi}SL}{N} \min\{1, 10(C^\pi - 1)\}} \lesssim \sqrt{\frac{(C^{\pi}-1)SL}{N} }.
\]

\paragraph{Proof of the bound on $T_{2,1}$. }

We again use the basic fact that 
\[
\left[r(s,\pi(s))-r(s,\hat{\pi}(s))\right]\mathbbm{1}\{N(s,\pi(s))\geq1\}\mathbbm{1}\{\mathcal{E}\}\leq1
\]
to reach
\[
T_{2,1}\leq\sum_{s\in\mathcal{S}_{1}}\rho(s)\leq\frac{2SC^{\pi}L}{N},
\]
where the last inequality hinges on the definition~\eqref{eq:defn-S-1-CB}
of $\mathcal{S}_{1}$, namely for any $s\in\mathcal{S}_{1}$, one
has $\rho(s)<\frac{2 C^{\pi}L}{N}$. 

\paragraph{Proof of the bound on $T_{2,2}$. }
Fix a state $s \in \cS_{2}$, we define the following two sets of actions:
\begin{align*}
    \cA_1(s) & \coloneqq \{ a \mid r(s, a) < r(s, \pi(s)), \mu(s,a) \leq L/(200N) \},\\
    \cA_2(s) & \coloneqq \{ a \mid r(s, a) < r(s, \pi(s)), \mu(s,a) > L/(200N) \}.
\end{align*}
Further define $A(s, a)$ to be the event that $ \hat r(s, \pi(s))-b(s, \pi(s))<\hat r(s,a) - b(s,a)$. Clearly one has $ r(s,\pi(s))-r(s,\hat{\pi}(s))\leq \ind \{\cup_{a \in \cA_1(s) \cup \cA_{2}(s)}A(s,a)\}$. Consequently, we can write the following decomposition:
\begin{align*}
    & \E_\cD\left[\left[r(s,\pi(s))-r(s,\hat{\pi}(s))\right] \ind \{N(s,\pi(s))\geq1\}\ind\{\mathcal{E}\}\right]\\
     &\quad \leq  \mathbb{P}(\exists a, r(s, a) < r(s, \pi(s)), A(s,a), N(s,\pi(s))\geq1)\\
      &\quad \leq \mathbb{P}(\exists a\in \mathcal{A}_1(s), A(s,a), N(s,\pi(s))\geq1) \eqqcolon p_1(s)\\
    & \quad \quad +\mathbb{P}(\exists a\in \mathcal{A}_2(s), A(s,a), N(s,\pi(s))\geq1) \eqqcolon p_2(s) .
\end{align*}
As a result, $T_{2,2}$ obeys 
\begin{align}\label{eq:T_4_prime_decomp_CB}
    T_{2, 2} \leq \sum_{s\in\mathcal{S}_2} \rho(s) p_1(s) + \sum_{s\in\mathcal{S}_2} \rho(s) p_2(s),
\end{align}
which satisfy the bounds
\begin{equation*}
    \sum_{s\in\mathcal{S}_2} \rho(s) p_1(s) \lesssim \frac{1}{N^{10}},\quad \text{and}\quad \sum_{s\in\mathcal{S}_2} \rho(s) p_2(s) \lesssim \frac{1}{N^9}.
\end{equation*}
Taking these two bounds collectively leads us to the desired conclusion. In what follows, we focus on the proving the aforementioned two bounds.

\subparagraph{Proof of the bound on $\sum_{s\in\mathcal{S}_2} \rho(s) p_1(s)$.}

Fix a state $s\in\mathcal{S}_{2}$. In view of the data coverage assumption, one has 
\begin{align}\label{eq:CB-prob-lower}
    \mu(s, \pi(s)) \geq \frac{\rho(s)}{C^\pi}  \geq  \frac{2 L}{N}.
\end{align}
In contrast, for any $a \in \mathcal{A}_1(s)$, we have  
\begin{align}\label{eq:CB-prob-upper}
    \mu(s,a) \leq   \frac{L}{200N}.
\end{align}
Therefore one has $\mu(s, \pi(s)) \gg \mu(s,a)$ for any 
non-expert action $a$. As a result, the optimal action is selected more frequently than the sub-optimal ones.
It turns out that under such circumstances, the LCB algorithm picks the right action with high probability. We make this intuition precise below. 

The bounds~\eqref{eq:CB-prob-lower} and~\eqref{eq:CB-prob-upper} together with Chernoff's bound give
\begin{align*}
    \prob \left( N(s,a) \leq \frac{5L}{200} \right) & \geq 1 -  \exp \left(-\frac{L}{200} \right); \\
    \prob \left(  N(s,\pi(s)) > L \right) & \geq 1 -  \exp \left(-\frac{L}{4} \right). 
\end{align*}
These allow us to obtain an upper bound for the function $\hat r - b$ evaluated at sub-optimal actions and a lower bound on $\hat r(s,\pi(s)) - b(s, \pi(s))$. More precisely, if $N(s, a)=0$, we know that $\hat r(s, a) = -1$; when $1\leq N(s,a) \leq \frac{5L}{200}$, we have 
\begin{align*}
   \hat r(s,a) - b(s, a)    & \leq 1 -  \sqrt{\frac{L}{5L/200}} \leq -5.
\end{align*}
Now we turn to lower bounding the function $\hat r - b$ evaluated at the optimal action. When $N(s,\pi(s)) > L$, one has
\begin{align*}
    \hat r(s,\pi(s)) - b(s,\pi(s))  
    > - \sqrt{\frac{L}{N(s, \pi(s))}} = -1.
\end{align*}

To conclude, if both $N(s,a) \leq \frac{5L}{200}$ and $N(s,\pi(s)) \geq L$ hold, we must have $\hat r(s, a) - b(s, a)< \hat r(s, \pi(s)) - b(s, \pi(s))$. 
Therefore  we can deduce that 
\begin{align*}
 \sum_{s\in\cS_2}\rho(s)p_1(s) & = \sum_{s\in\cS_2}\rho(s)\mathbb{P}(\exists a\in \mathcal{A}_1(s), A(s,a), N(s,\pi(s))\geq1)\\
 & \leq  (|\mathcal{A}|-1)\exp\left(-\frac{L}{200}\right)+\exp\left(-\frac{1}{4}  L\right) \\
 & \leq |\mathcal{A}|\exp\left(-\frac{L}{200}\right) \\
 & \lesssim \frac{1}{N^{10}}.
\end{align*}
The last inequality comes from the choice of $L = 2000\log(2S|\mathcal{A}|N)$.
\subparagraph{Proof of the bound on $\sum_{s \in \cS_{2}} \rho(s) p_2(s)$.}
Before embarking on the proof of $\sum_{s\in\mathcal{S}_2} \rho(s) p_2(s) \lesssim \lesssim \frac{1}{N^9}$, it is helpful to pause and gather a few useful properties of $(s,a)$ with $s\in\mathcal{S}_2$, $a\in\mathcal{A}_2(s)$:
\begin{enumerate}
    \item $\rho(s)\geq  \frac{2C^\pi L}{N}$ and hence $\mu(s,\pi(s)) \geq  \frac{2L}{N}$ by the definition of $C^\pi$;
    \item $\frac{L}{200N}\leq\mu(s, a)\leq\frac{1}{10}\mu(s,\pi(s))$;
    \item $\sum_{a\in \mathcal{A}_2} \mu(s, a)\leq \frac{1}{10}\mu(s, \pi(s))$;
    \item $|\cA_2(s)| \leq 200 N / L$. 
\end{enumerate}
In addition, we define a high probability event on which the sample sizes $N(s,a)$ concentrate around their respective means $N\mu(s,a)$:
\begin{align*}
    \mathcal{E}_{2}(s) \coloneqq &  \Bigg\{\frac{1}{2}N \mu(s,\pi(s))\leq N(s,\pi(s))\leq2N\mu(s,\pi(s)), \\
    & \quad \forall a\in\mathcal{A}_2(s), \frac{1}{2}N \mu(s, a)\leq N(s,a)\leq2N\mu(s, a) \Bigg\},
\end{align*}
which---in view of the Chernoff bound and the union bound---obeys
\begin{equation}\label{eq:action-event-CB}
    \mathbb{P}(\mathcal{E}_2(s)) \geq 1 - 1 /N^9.
\end{equation}

With these preparations in place, we can derive
\begin{align*}
    p_2(s) & =  \mathbb{P}(\exists a\in \mathcal{A}_2, A(s,a), N(s,\pi(s))\geq1)\\ 
    & \leq\mathbb{P}(\mathcal{E}_2^c(s))+ \mathbb{P}(\exists a\in \mathcal{A}_2, A(s,a), N(s,\pi(s))\geq1, \mathcal{E}_2(s))\\
    & \leq \mathbb{P}(\mathcal{E}_2^c(s))+  \sum_{a\in\mathcal{A}_2} \mathbb{P}(A(s, a), N(s,\pi(s))\geq1, \mathcal{E}_2(s)) \\
    & \lesssim \frac{1}{N^9} + \frac{|\mathcal{A}_2|}{N^{10}}   \lesssim \frac{1}{N^9},
\end{align*}
where the last line arises from the bound
\begin{equation}\label{eq:error-prob-CB}
    \mathbb{P}(A(s, a), N(s,\pi(s))\geq1, \mathcal{E}_2(s)) \lesssim \frac{1}{N^{10}},
\end{equation}
and the cardinality upper bound $|\cA_2(s)| \lesssim N$. 
This completes the bound on $\sum_{s \in \cS_{2}} p(s)$. 

\begin{proof}[Proof of the bound~\eqref{eq:error-prob-CB}]
On the event $\cE_2(s)$, one must have $N(s,a) \geq 1$ and $N(s,\pi(s)) \geq 1$. Therefore, we can define 
\begin{align*}
    \epsilon\coloneqq\sqrt{\frac{L}{N(s,a)}}-\sqrt{\frac{L}{N(s,\pi(s))}} \quad \text{and} \quad \Delta= r(s,\pi(s)) - r(s,a),
\end{align*}
and obtain the following bound on the conditional probability
\begin{align*}
 & \mathbb{P}\left(\hat{r}(s,a)-\sqrt{\frac{L}{N(s,a)}}\geq\hat{r}(s,\pi(s))-\sqrt{\frac{L}{N(s,\pi(s))}}\;\middle| \; N(s, \pi(s)), N(s,a), \mathcal{E}_2\right)\\
 &\quad \leq  \exp\left(-2\frac{N(s,a)N(s,\pi(s))(\epsilon+\Delta)^{2}}{N(s,a)+N(s,\pi(s))} \;\middle| \; N(s, \pi(s)), N(s,a), \mathcal{E}_2 \right),
\end{align*}
where the inequality arises from Lemma~\ref{lemma:hoeffding_on_difference_empirical_average}. 
Note that under event $\cE_2(s)$ and the property $\mu(s, a)\leq \frac{1}{10}\mu(s, \pi(s))$, we have $N(s,\pi(s))\geq 4N(s,a)$ and thus $\epsilon\geq\frac{1}{2}\sqrt{\frac{L}{N(s,a)}}$. 
This allows us to further upper bound the probability as
\begin{align*}
 & \mathbb{P}\left(\hat{r}(s,a)-\sqrt{\frac{L}{N(s,a)}}\geq\hat{r}(s,\pi(s))-\sqrt{\frac{L}{N(s,\pi(s))}}\;\middle| \; N(s, \pi(s)), N(s,a), \mathcal{E}_2\right)\\
     &\quad \leq \exp\left(-N(s,a)(\epsilon+\Delta)^{2}\right)\\
 &\quad \leq \exp\left(-\left(\frac{1}{2}\sqrt{L}+ \sqrt{N(s,a)}\Delta\right)^{2}\right)\\
 &\quad \leq \exp\left(-\frac{1}{4}L\right)
 \lesssim  \frac{1}{N^{10}},
\end{align*}
under the choice of $L = 2000\log(2S|\mathcal{A}|N)$.
Since this upper bound holds for any configuration of $N(s,a)$ and $N(s,\pi(s))$, one has the desired claim. 
\end{proof}

\paragraph{Proof of the bound on $T_{2,3}$. }

On the good event $\mathcal{E}$, we know that 
\begin{align*}
r(s,\pi(s))-r(s,\hat{\pi}(s)) & \leq r(s,\pi(s))-\left[\hat{r}(s,\hat{\pi}(s))-b(s,\hat{\pi}(s))\right]\\
 & \leq r(s,\pi(s))-\left[\hat{r}(s,\pi(s))-b(s,\pi(s))\right]\\
 & \leq2b(s,\pi(s)).
\end{align*}
Here the middle line arises from the definition of the LCB algorithm,
i.e., $\hat{\pi}(s)\in\arg\max_{a}\hat{r}(s,a)-b(s,a)$ for each $s$.
Substitute this upper bound into the definition of $T_2$ to obtain
\begin{align*}
T_{2,3} & \leq2\sum_{s\in\mathcal{S}_{3}}\rho(s)\mathbb{E}_\cD\left[b(s,\pi(s))\mathbbm{1}\{N(s,\pi(s))\geq1\}\mathbbm{1}\{\mathcal{E}\}\right]\\
 & =2\sum_{s\in\mathcal{S}_{3}}\rho(s)\mathbb{E}_\cD\left[\sqrt{\frac{L}{N(s,\pi(s))}}\mathbbm{1}\{N(s,\pi(s))\geq1\}\mathbbm{1}\{\mathcal{E}\}\right]\\
 & \leq2\sqrt{L}\sum_{s\in\mathcal{S}_{3}}\rho(s)\mathbb{E}_\cD\left[\sqrt{\frac{1}{N(s,\pi(s))\vee1}}\mathbbm{1}\{N(s,\pi(s))\geq1\}\right],
\end{align*}
where we have used the definition of $b(s,a)$. Lemma~\ref{lemma:binomial_inverse_moment_bound}
tells us that there exists a universal constant $c>0$ such that 
\[
\mathbb{E}_\cD\left[\sqrt{\frac{1}{N(s,\pi(s))\vee1}}\mathbbm{1}\{N(s,\pi(s))\geq1\}\right]\leq\frac{c}{\sqrt{N\mu(s,\pi(s))}}.
\]
As a result, we reach the conclusion that 
\[
T_{2,3}\leq2\sqrt{L}\sum_{s\in\mathcal{S}_{3}}\rho(s)\frac{c}{\sqrt{N\mu(s,\pi(s))}}.
\]
In view of the assumption $\max_{s}\rho(s)/\mu(s,\pi(s))\le C^{\pi}$,
one further has 
\begin{align*}
T_{2,3} & \leq2c\sqrt{\frac{C^{\pi}L}{N}}\sum_{s\in\mathcal{S}_{3}}\sqrt{\rho(s)}\leq2c\sqrt{\frac{C^{\pi}L}{N}}\sqrt{S}\sqrt{\sum_{s\in\mathcal{S}_{3}}\rho(s)},
\end{align*}
with the last inequality arising from Cauchy-Schwarz's inequality. The desired
bound on $T_{2,3}$ follows from the following simple fact regarding
$\sum_{s\in\mathcal{S}_{3}}\rho(s)$:
\begin{equation}\label{eq:prob-upper-bound-CB}
\sum_{s\in\mathcal{S}_{3}}\rho(s)\leq\min\left\{ 1,10(C^{\pi}-1)\right\} .    
\end{equation}

\begin{proof}[Proof of the inequality~\eqref{eq:prob-upper-bound-CB}]
The upper bound $1$ is trivial to see. To achieve the other upper bound, we first use the assumption $\max_{s}\rho(s)/\mu(s,\pi(s))\le C^{\pi}$ to see
\begin{equation*}
    \sum_{s\in\mathcal{S}_{3}}\rho(s)\leq \sum_{s\in\mathcal{S}_{3}} C^\pi \mu(s,\pi(s)) \leq 10 C^\pi \sum_{s\in\mathcal{S}_{3}}  \bar{\mu}(s).
\end{equation*}
Here the last relation follows from the definition of $\mathcal{S}_3$. 
Note that 
\begin{equation*}
    \sum_{s\in\mathcal{S}_{3}}  \bar{\mu}(s) \leq \sum_{s}  \bar{\mu}(s) = 1 -\sum_{s} \mu(s,\pi(s)) \leq 1 - \frac{1}{C^\pi}, 
\end{equation*}
where we have reused the assumption $\max_{s}\rho(s)/\mu(s,\pi(s))\le C^{\pi}$. 
Taking the previous two inequalities collectively yields the final claim. 
\end{proof}

\subsubsection{Proof of the bound~\eqref{eq:event-CB-proof} on $T_{3}$}

It is not hard to see that 
\[
\sum_{s}\rho(s)\left[r(s,\pi(s))-r(s,\hat{\pi}(s))\right]\mathbbm{1}\{N(s,\pi(s))\geq1\}\leq1,
\]
which further implies 

\[
T_{3}\leq\mathbb{E}_\cD\left[\mathbbm{1}\{\mathcal{E}^{c}\}\right]=\mathbb{P}(\mathcal{E}^{c}).
\]
It then boils down to upper bounding the probability $\mathbb{P}(\mathcal{E}^{c})$.
The proof is similar in spirit to that of Lemma~\ref{lemma:hoeffding-MAB}.

Fix a state-action pair $(s,a)$. If $N(s,a)=0$, one clearly has
$-1=\hat{r}(s,a)-b(s,a)\leq r(s,a)\leq\hat{r}(s,a)+b(s,a)=1$. Therefore
we concentrate on the case when $N(s,a)\geq1$. Apply the Hoeffding's
inequality to see that for any $\delta_{1}\in(0,1)$, one has
\[
\mathbb{P}\left(\left|\hat{r}(s,a)-r(s,a)\right|\geq\sqrt{\frac{\log(2/\delta_{1})}{2N(s,a)}}\mid N(s,a)\right)\leq\delta_{1}.
\]
In particular, setting $\delta_{1}=\delta/(S|\mathcal{A}|)$
yields 
\begin{equation}
\mathbb{P}\left(\left|\hat{r}(s,a)-r(s,a)\right|\geq \sqrt{\frac{\log(2S|\mathcal{A}|/\delta)}{2N(s,a)}}\mid N(s,a)\right)\leq\frac{\delta}{S|\mathcal{A}|},\label{eq:hoeffding-CB}
\end{equation}
Recall that $b(s,a)$ is defined such that when $N(s,a)\geq1$,
\[
b(s,a)=\sqrt{\frac{2000\log(2S|\mathcal{A}|/\delta)}{N(s,a)}}.
\]
Since the inequality~\eqref{eq:hoeffding-CB} holds for any $N(s,a)$,
we have for any fixed $(s,a)$, 
\[
\mathbb{P}\left(\left|\hat{r}(s,a)-r(s,a)\right|\geq b(s,a)\right)\leq\frac{\delta}{S|\mathcal{A}|}.
\]
Taking a union bound over $\mathcal{S}\times\mathcal{A}$ leads to
the conclusion that $\mathbb{P}(\mathcal{E}^{c})\leq\delta$, and
hence $T_{3}\leq\delta$. Taking $\delta=1/N$ gives the advertised result.

\subsection{Proof of Theorem \ref{theorem:lower_bound_offline_CB}}\label{app:proof_CB_lower}

We prove the lower bound differently for the following regimes: $C^{\star}=1$, $C^{\star}\geq2$, and $C^{\star}\in(1,2)$.
When $C^{\star}=1$, the offline RL problem reduces to the imitation
learning problem in contextual bandits, whose lower bound has been
shown in the paper~\citet{rajaraman2020toward}. When $C^{\star}\in(1,2)$ or $C^{\star}\geq2$, we generalize the lower bound given for the multi-armed
bandits with different choices of initial distributions. In what follows,
we detail the proofs for each regime.

\paragraph{The case when $C^{\star}=1$. }

When $C^{\star}=1$, one has $d^{\star}(s,a)=\mu(s,a)$ for any $(s,a)$
pair. This recovers the imitation learning problem, where the rewards are also included in the dataset. Thus the lower
bound proved in Lemma~\ref{lem:IL_LB_MDP_lower} 
is applicable, which comes from a modified version of Theorem 6 in the paper~\citet{rajaraman2020toward}: 
\begin{equation}
\inf_{\hat{\pi}}\sup_{(\rho,\mu,R)\in\mathsf{CB}(1)}\mathbb{E}_\cD[J(\pi^{\star})-J(\hat{\pi})]\gtrsim\min\left(1,\frac{S}{N}\right).\label{eq:CB-lower-bound-C-1}
\end{equation}

\paragraph{The case when $C^{\star}\protect\geq2$. }

Fix a contextual bandit instance $(\rho,\mu,R)$, define the loss/sub-optimality of an estimated policy $\pi$ to be
\[
\mathcal{L}(\pi;(\rho,\mu,R))\coloneqq J(\pi^{\star})-J(\hat{\pi}).
\]
We intend to show that when $C^{\star}\geq2$, 
\begin{equation}
\inf_{\hat{\pi}}\sup_{(\rho,\mu,R)\in\mathsf{CB}(C^{\star})}\mathbb{E}_{\mu \otimes R}[\mathcal{L}(\pi;(\rho,\mu,R))]\gtrsim\min\left(1,\sqrt{\frac{SC^{\star}}{N}}\right).\label{eq:CB-lower-bound-C-2}
\end{equation}
Our proof follows the standard recipe of proving minimax lower bounds,
namely, we first construct a family of hard contextual bandit instances,
and then apply Fano's inequality to obtain the desired lower bound. 

\subparagraph{Construction of hard instances. }
Consider a CB with state space $\cS\coloneqq\{1,2,\ldots,S\}$.
Set the initial distribution $\rho_{0}(s)=1/S$ for any $s\in\mathcal{S}$. 
Each state $s \in \cS$ is associated with two actions $a_{1}$ and $a_{2}$.
The behavior distribution for each $s,a$ is specified below 
\[
\mu_{0}(s,a_{1})=\frac{1}{S}-\frac{1}{SC^{\star}}\qquad\text{and}\qquad\mu_{0}(s,a_{2})=\frac{1}{SC^{\star}}.
\]
It is easy to check that for any reward distribution $R$, one has $(\rho_{0},\mu_{0},R)\in\mathsf{CB}(C^{\star}).$ It remains to construct a set of reward distributions that are nearly
indistinguishable from the data. To achieve this goal, we leverage the Gilbert-Varshamov lemma (cf.~Lemma~\ref{lem:V_G}) to obtain a set
$\mathcal{V}\subseteq\{-1,1\}^{S}$ that obeys (1) $|\mathcal{V}|\ge\exp(S/8)$
and (2) $\|\bm{v}_{1}-\bm{v}_{2}\|_{1}\geq S/2$ for any $\bm{v}_{1}\bm{v}_{2}\in\mathcal{V}$
with $\bm{v}_{1}\neq\bm{v}_{2}$. With this set $\mathcal{V}$ in
place, we can continue to construct the following set of Bernoulli
reward distributions
\[
\mathcal{R}\coloneqq\left\{ \left\{ \mathsf{Bern}\left(\frac{1}{2}\right),\mathsf{Bern}\left(\frac{1}{2}+v_{s}\delta\right)\right\} _{s\in\mathcal{S}}\mid\bm{v}\in\mathcal{V}\right\} .
\]
Here $\delta\in(0,1/3)$ is a parameter that will be specified later.
Each element $\bm{v}\in\mathcal{V}$ is mapped to a reward
distribution such that for the state $s$, the reward distribution
associated with $(s,a_{2})$ is $\mathsf{Bern}(\frac{1}{2}+v_{s}\delta)$.
In view of the second property of the set $\mathcal{V}$, one has
for any policy $\pi$ and any two different reward distributions $R_{1},R_{2}\in\mathcal{R}$,
\[
\mathcal{L}(\pi;(\rho_{0},\mu_{0},R_{1}))+\mathcal{L}(\pi;(\rho_{0},\mu_{0},R_{2}))\geq\frac{\delta}{4}.
\]

\subparagraph{Application of Fano's inequality. }

Now we are ready to apply Fano's inequality, that is 
\[
\inf_{\hat{\pi}}\sup_{(\rho_{0},\mu_{0},R)\mid R\in\mathcal{R}}\mathbb{E}_{\mu_0 \otimes R}[\mathcal{L}(\pi;(\rho_{0},\mu_{0},R))]\geq\frac{\delta}{8}\left(1-\frac{N\max_{i\neq j}\mathsf{KL}\left(\mu\otimes R_{i}\|\mu\otimes R_{j}\right)+\log2}{\log|\mathcal{R}|}\right).
\]
It then remains to control $\max_{i\neq j}\mathsf{KL}\left(\mu\otimes R_{i}\|\mu\otimes R_{j}\right)$
and $\log|\mathcal{R}|$. For the latter quantity, we have 
\[
\log|\mathcal{R}|=\log|\mathcal{V}|\geq S/8,
\]
where the inequality comes from the first property of the set $\mathcal{V}$.
With regards to the KL divergence, one has
\[
\max_{i\neq j}\mathsf{KL}\left(\mu\otimes R_{i}\|\mu\otimes R_{j}\right)\leq S\cdot\frac{1}{SC^{\star}}\cdot16\delta^{2}=\frac{16\delta^{2}}{C^{\star}}.
\]
As a result, we conclude that as long as 
\[
\frac{200N\delta^{2}}{SC^{\star}}\leq1,
\]
one has
\[
\inf_{\hat{\pi}}\sup_{(\rho_{0},\mu_{0},R)\mid R\in\mathcal{R}}\mathcal{L}(\pi;(\rho_{0},\mu_{0},R))\gtrsim\delta.
\]
To finish the proof, we can set $\delta=\sqrt{\frac{SC^{\star}}{200N}}$
when $\sqrt{\frac{SC^{\star}}{200N}}<\frac{1}{3}$, and $\delta=\frac{1}{3}$
otherwise. This yields the desired lower bound~\eqref{eq:CB-lower-bound-C-2}. 

\paragraph{The case when $C^{\star}\in(1,2)$. }

We intend to show that 
\begin{equation}
\inf_{\hat{\pi}}\sup_{(\rho,\mu,R)\in\mathsf{CB}(C^{\star})}\mathbb{E}[\mathcal{L}(\pi;(\rho, \mu, R))]\gtrsim\min\left(C^{\star}-1,\sqrt{\frac{S(C^{\star}-1)}{N}}\right).\label{eq:CB-lower-bound-C-1-2}
\end{equation}
The proof is similar to that of the previous case, with the difference
lying in the construction of $\rho_{0}$ and $\mu_{0}$.

\subparagraph{Construction of hard instances. }
Consider a CB with state space $\cS\coloneqq\{0,1,2,\ldots,S\}$ and action space $\cA \coloneqq \{ a_1, a_2\}$.
Set the initial distribution $\rho_{0}(s)=(C^{\star}-1)/S$
for any $1\leq s\leq S$ and $\rho_{0}(0)=2-C^{\star}$. Each
state $1\leq s\le S$ is associated with two actions $a_{1}$ and
$a_{2}$ such that 
\[
\mu_{0}(s,a_{1})=\mu_{0}(s,a_{2})=\frac{C^{\star}-1}{SC^{\star}}.
\]
In contrast, for $s=0$, one has a single action $a_{1}$ with $\mu_{0}(0,a_{1})=\frac{2-C^{\star}}{C^{\star}}$.
Similar to the above case, we have for any reward distribution $R$, that $(\rho_{0},\mu_{0},R)\in\mathsf{CB}(C^{\star}).$

We deploy essentially the same family $\mathcal{R}$ of reward distributions
as before with an additional reward of $R(0,a_{1})\equiv0$ on state $s = 0$. As a result,
one can show that for any policy $\pi$ and any two different reward
distributions $R_{1},R_{2}\in\mathcal{R}$, 
\[
\mathcal{L}(\pi;(\rho_{0},\mu_{0},R_{1}))+\mathcal{L}(\pi;(\rho_{0},\mu_{0},R_{2}))\geq\frac{\delta}{4}(C^{\star}-1).
\]

\subparagraph{Application of Fano's inequality. }

Fano's inequality tells us that 
\[
\inf_{\hat{\pi}}\sup_{(\rho_{0},\mu_{0},R)\mid R\in\mathcal{R}}\mathbb{E}[\mathcal{L}(\pi;(\rho_{0},\mu_{0},R))]\geq\frac{\delta(C^\star-1)}{8}\left(1-\frac{N\max_{i\neq j}\mathsf{KL}\left(\mu\otimes R_{i}\|\mu\otimes R_{j}\right)+\log2}{S/8}\right).
\]
In the current case, we have 
\[
\max_{i\neq j}\mathsf{KL}\left(\mu\otimes R_{i}\|\mu\otimes R_{j}\right)\leq S\cdot\frac{C^{\star}-1}{SC^{\star}}\cdot16\delta^{2}=\frac{16(C^{\star}-1)}{C^{\star}}\delta^{2}.
\]
As before, setting 
\[
\delta=\min\left(\sqrt{\frac{SC^{\star}}{200(C^{\star}-1)N}},\frac{1}{3}\right)
\]
yields the lower bound 
\[
\inf_{\hat{\pi}}\sup_{(\rho_{0},\mu_{0},R)\mid R\in\mathcal{R}}\mathbb{E}[\mathcal{L}(\pi;(\rho_{0},\mu_{0},R))]\gtrsim\min\left(C^{\star}-1,\sqrt{\frac{SC^{\star}(C^{\star}-1)}{N}}\right)\gtrsim\min\left(C^{\star}-1,\sqrt{\frac{S(C^{\star}-1)}{N}}\right).
\]

\paragraph{Putting the pieces together. }
We are now in position to summarize and simplify the three established
lower bounds~\eqref{eq:CB-lower-bound-C-1}, \eqref{eq:CB-lower-bound-C-2},
and \eqref{eq:CB-lower-bound-C-1-2}. 

When $C^{\star}=1$, the claim in Theorem~\ref{theorem:lower_bound_offline_CB} is identical to the
bound~\eqref{eq:CB-lower-bound-C-1}. 

When $C^{\star}\geq2$, we have from the bound~{\eqref{eq:CB-lower-bound-C-2}
that }
\[
\inf_{\hat{\pi}}\sup_{(\rho,\mu,R)\in\mathsf{CB}(C^{\star})}\mathbb{E}[\mathcal{L}(\pi;(\rho,\mu,R))]\gtrsim\min\left(1,\sqrt{\frac{SC^{\star}}{N}}\right)\asymp\min\left(1,\sqrt{\frac{S(C^{\star}-1)}{N}}\right).
\]
Further notice that 
\[
\sqrt{\frac{S(C^{\star}-1)}{N}}\geq\sqrt{\frac{S}{N}}\geq\min\left(1,\frac{S}{N}\right).
\]
The claimed lower bound in Theorem~\ref{theorem:lower_bound_offline_CB}
arises.

In the end, when $C^{\star}\in(1,2)$, we know from the bounds~\eqref{eq:CB-lower-bound-C-1}
and \eqref{eq:CB-lower-bound-C-1-2} that 
\[
\inf_{\hat{\pi}}\sup_{(\rho,\mu,R)\in\mathsf{CB}(C^{\star})}\mathbb{E}[\mathcal{L}(\pi;(\rho,\mu,R))]\gtrsim\max\left\{ \min\left(1,\frac{S}{N}\right),\min\left(C^{\star}-1,\sqrt{\frac{S(C^{\star}-1)}{N}}\right)\right\} .
\]
Elementary calculations reveal that 
\[
\max\left\{ \min\left(1,\frac{S}{N}\right),\min\left(C^{\star}-1,\sqrt{\frac{S(C^{\star}-1)}{N}}\right)\right\} \asymp\min\left(1,\sqrt{\frac{S(C^{\star}-1)}{N}}+\frac{S}{N}\right),
\]
which completes the proof.

\subsection{Proof of Proposition~\ref{thm:CB_IL_lower}}\label{app:proof_CB_IL_lower}

We design the hard instance with state space  $\{s_0, s_1\}$ and action space $\{a_0, a_1\}$. Only under state $(s_0, a_0)$ we can possibly get non-zero reward, and all other state-action pairs give $0$ rewards. We set $d^\star(s_0) = d^\star(s_0, a_0) = C^\star-1-\epsilon$, $d^\star(s_1) = 2-C^\star+\epsilon$ for some small $\epsilon>0$. The constraints introduced by concentrability are $\mu(s_0, a_0)\geq (C^\star-1-\epsilon)/C^\star$, $\mu(s_1)\geq (2-C^\star+\epsilon)/C^\star$.

We set $\mu(s_0, a_0)= (C^\star-1-\epsilon)/C^\star$, $\mu(s_0, a_1)=(C^\star-1)/C^\star$, $\mu(s_1) = (2-C^\star+\epsilon)/C^\star$. One can verify that $d^\star, \mu$ are valid probability distributions and the concentrability assumption still holds. 

In this case, since $\mu(s_0, a_0)<\mu(s_0, a_1) $, the algorithm fails to identify the optimal arm $a_0$ as $N\rightarrow \infty$. This incurs the following expected sub-optimality 
\begin{align*}
\lim_{N\rightarrow \infty}\E_\cD[J(\pi^{\star})-J(\hat{\pi})] =  d^\star(s_0) \geq C^\star-1-\epsilon. 
\end{align*}
Setting $\epsilon\rightarrow 0$ gives us the conclusion.

\section{Proofs for MDPs}
We begin by presenting several Bellman equations for discounted MDPs, which is followed by the proof of Lemma~\ref{lemma:penalty_LCB}. We then prove general properties of Algorithm \ref{alg:OVI-LCB-DS} under the clean event \eqref{def:clean_event}. These include the contraction properties given in Proposition \ref{prop:monotonicity_data_splitting} as well as the value difference lemma (cf.~Lemma~\ref{lemma:value_difference}). Next, we prove the LCB sub-optimality Theorem \ref{thm:MDP_upperbound_hoffding}. In the end, we prove the minimax lower bound followed by an analysis of imitation learning with an alternative data coverage assumption.

\subsection{Bellman and Bellman-like equations}
Given a discounted MDP, the Bellman value operator $\cT_\pi$ associated with a policy $\pi$ is defined as
\begin{align}
    \cT_\pi V \coloneqq r_\pi + \gamma P_\pi V.
\end{align}
It is well-known that $V^\pi$ is the unique solution to $\cT_\pi V = V$, which is known as the Bellman equation. 

In addition to $V^\pi$, other quantities in an MDP also follow a Bellman-like equation, which we briefly review here. For discounted occupancy measures, simple algebra gives
\begin{alignat}{2}\label{def:state_occupancy_vector}
    d_\pi & = (1-\gamma)\rho + \gamma d_\pi P_\pi \quad && \Rightarrow \quad d_\pi = (1-\gamma )\rho (I - \gamma P_\pi)^{-1},\\ \label{def:state_action_occupancy_vector}
    d^\pi & = (1-\gamma)\rho^\pi + \gamma d^\pi P^\pi \quad && \Rightarrow \quad d^\pi = (1-\gamma)\rho^\pi (I - \gamma P^\pi)^{-1}.
\end{alignat}

\subsection{Proof of Lemma \ref{lemma:penalty_LCB}}\label{app:proof_penalty_LCB}
The proof is similar to that of Lemma~\ref{lemma:hoeffding-MAB}. For completeness, we include it here.

From the algorithmic design, it is clear (in particular the $Q$ update and the monotonic improvement step) that 
\begin{equation*}
    V_t(s) \in [0, V_{\max}], \qquad \text{for all }s \in \mathcal{S} \text{ and } t \geq 0. 
\end{equation*}
As a result, for a fixed tuple $(s,a,t)$, if $m_t(s,a) = 0$, one has
\begin{align*}
    \left|r(s,a) +\gamma P_{s,a} \cdot V_t  - r_t(s,a) - \gamma P_{s,a}^t \cdot V_{t-1}\right| \leq 1 + \gamma V_{\max} = V_{\max} \leq b_t(s,a).
\end{align*}
When $m_t(s,a) \geq 1$, exploiting the independence between $V_t$ and $P^t_{s,a}$ and using Hoeffding's inequality 
to obtain 
\begin{equation*}
    \mathbb{P}\left(\left|r(s,a)+\gamma P_{s,a}\cdot V_{t}-r_{t}(s,a)-\gamma P_{s,a}^{t}\cdot V_{t-1}\right|\geq V_\max \sqrt{L / m_t(s,a)} \mid m_{t}(s,a)\right)\leq2\exp\left(-2L\right).
\end{equation*}
Since the above inequality holds for any $m_t(s,a)$, one necessarily has 
\begin{equation*}
\mathbb{P}\left(\left|r(s,a)+\gamma P_{s,a}\cdot V_{t}-r_{t}(s,a)-\gamma P_{s,a}^{t}\cdot V_{t-1}\right|\geq b_{t}(s,a)\right)\leq2\exp\left(-2L\right).    
\end{equation*}
Taking a union bound over $s,a$ and $t \in \{0,\dots, T\}$ and setting $\delta_1 = \frac{\delta}{2 S|\cA|(T+1)}$ finishes the proof.

\subsection{Proof of Proposition \ref{prop:monotonicity_data_splitting}}\label{app:proof_monotonicity}
We prove the claims one by one. 
\paragraph{Proof of $V_{t-1} \leq V_t$.} 
The first claim $V_{t-1} \leq V_t$ is directly implied by line~\ref{line:imposing_monotonicity} of Algorithm \ref{alg:OVI-LCB-DS}. 

\paragraph{Proof of $V_t \leq V^{\pi_t}$.}  
For the second claim $V_t \leq V^{\pi_t}$, it suffices to prove that $V_t \leq \cT_{\pi_t}V_t$. Indeed, $V_t \leq \cT_{\pi_t}V_t$ together with the monotonicity of the Bellman's operator yield the conclusion $V_t \leq V^{\pi_t}$. In what follows, we prove $V_t \leq \cT_{\pi_t}V_t$ via induction. 

The base case $V_0 \leq \cT_{\pi_0}V_0$ holds due to zero initialization. Hence from now on, we assume $V_k \leq \cT_{\pi_k}V_k$ for $0 \leq k \leq t-1$ and intend to prove $V_t \leq \cT_{\pi_t}V_t$.
We split the proof into two cases. 

\begin{itemize}
    \item If $V_{t-1}(s) \geq \max_a \{ r_{t-1}(s,a) - b_{t-1}(s,a) + \gamma P^{t-1}_{s,a} \cdot V_{t-1} \}$, the algorithm sets $V_t(s) = V_{t-1}(s)$ and $\pi_t(s) = \pi_{t-1}(s)$. 
Consequently, we have
\begin{align*}
    V_t(s) = V_{t-1}(s) \leq (\cT_{\pi_{t-1}} V_{t-1}) (s) \leq (\cT_{\pi_{t}} V_{t}) (s),
\end{align*}
where the first inequality arises from the induction hypothesis and the last one holds since $V_{t-1} \leq V_{t}$ and $\pi_t(s) = \pi_{t-1}(s)$. 
\item 
If instead, the algorithm sets $Q_{t}(s,a) = r_{t}(s,a) -  b_{t}(s,a) + \gamma P^{t}_{s,a} \cdot V_{t-1} $ with $\pi_t(s) = \argmax_a Q_t(s,a)$ and $V_t(s) = Q_t(s, \pi_t(s))$, then we have 
\begin{align*}
    (\cT_{\pi_t} V_{t}) (s) 
    = & r(s,\pi_t(s)) + \gamma P_{s,\pi_t(s)} \cdot V_t \\
    \geq & r(s,\pi_t(s)) + \gamma P_{s,\pi_t(s)} \cdot V_{t-1} \\
    = & r_{t}(s,\pi_t(s)) -  b_{t}(s,\pi_t(s)) + \gamma  P_{s,\pi_t(s)}^{t} \cdot V_{t-1} \\
    & +  b_{t}(s,\pi_t(s)) + r(s,\pi_t(s))- r_{t}(s,\pi_t(s)) + \gamma ( P_{s,\pi_t(s)} - P^{t}_{s, \pi_t(s)}) \cdot V_{t-1} \\
    = & V_t(s) + b_{t}(s,\pi_t(s)) + r(s,\pi_t(s))- r_{t}(s,\pi_t(s)) + \gamma ( P_{s,\pi_t(s)} - P^{t}_{s, \pi_t(s)}) \cdot V_{t-1}\\
    \geq & V_t(s),
\end{align*}
where the first inequality is due to $V_{t-1} \leq V_t$ and the last inequality holds under the clean event $\cE_{\text{MDP}}$. 
\end{itemize}
This finishes the proof of $V_t \leq \cT_{\pi_t}V_t$ and hence $V_t \leq V^{\pi_t}$. The claim $V^{\pi_t} \leq V^\star$ is trivial to see.

\paragraph{Proof of $Q_t \leq r + \gamma P V_{t-1} \leq r + \gamma P V_t$.} 
Since $V_t \geq V_{t-1}$, we have
\begin{align*}
     r(s,a) + \gamma P_{s,a} \cdot  V_t  & \geq  r(s,a) + \gamma P_{s,a} \cdot  V_{t-1}\\
    & \quad =  r_{t}(s,a) - b_{t}(s,a) + \gamma P^{t}_{s,a} \cdot  V_{t-1} \\
    &\quad \quad + b_{t}(s,a) + r(s,a) - r_{t}(s,a) + \gamma(P_{s,a} - P^{t}_{s,a}) \cdot  V_{t-1}\\
    & \quad \geq Q_t(s,a),
\end{align*}
where the last inequality holds under $\cE_{\text{MDP}}$. 

\paragraph{Proof of $Q^\pi - Q_{t} \leq \gamma P^\pi (Q^\pi - Q_{t-1}) + 2 b_t$.} Let $Q(:,\pi) \in \R^S$ be a vector with elements $Q^\pi(s,\pi(s))$. By definition, one has
\begin{align*}
    & Q^\pi(s,a) - Q_{t}(s,a) \\
    &\quad =  r(s,a)+ \gamma P_{s,a} \cdot V^\pi - r_t(s,a) + b_t(s,a) - \gamma P^t_{s,a} \cdot V_{t-1}\\
    &\quad =  \gamma P_{s,a} \cdot V^\pi - \gamma P_{s,a} \cdot V_{t-1} + b_t(s,a) + r(s,a) - r_t(s,a)+ \gamma (P_{s,a} - P^t_{s,a}) \cdot V_{t-1}\\
     & \quad \leq \gamma P_{s,a} \cdot (Q^\pi(:,\pi) - Q_{t-1}(:,\pi)) + b_t(s,a) + r(s,a) - r_t(s,a)+ \gamma (P_{s,a} - P^t_{s,a}) \cdot V_{t-1}\\
    & \quad \leq  \gamma P_{s,a} \cdot (Q^\pi(:,\pi) - Q_{t-1}(:,\pi)) + 2 b_t(s,a).
\end{align*}
Here, the first inequality comes from the fact that $V_{t-1} \geq \max_a Q_{t-1}(:,a) \geq Q_t(:,\pi)$ and the last inequality again holds under $\cE_{\text{MDP}}$.

\subsection{Proof of Lemma \ref{lemma:value_difference}}\label{app:proof_value_difference}
In view of Proposition~\ref{prop:monotonicity_data_splitting}, one has $V_t  \leq V^{\pi_t}$. Therefore we obtain
\[
\mathbb{E}_{\rho}\left[V^{\pi}(s)-V^{\pi_{t}}(s)\right] \leq \mathbb{E}_{\rho}\left[V^{\pi}(s)-V_{t}(s)\right] \leq \mathbb{E}_{\rho}\left[V^{\pi}(s)-V_{t}^{\mathrm{mid}}(s)\}\right],
\]
where the last inequality arises from the monotonicity imposed by
Algorithm~\ref{alg:OVI-LCB-DS}. Note that $V_{t}^{\mathrm{mid}}(s)=Q_{t}(s,\pi_{t}^{\mathrm{mid}})$
and that $\pi_{t}^{\mathrm{mid}}$ is greedy with respect to
$Q_{t}$. We can continue the upper bound as
\[
\mathbb{E}_{\rho}\left[V^{\pi}(s)-V^{\pi_{t}}(s)\right] \leq
\mathbb{E}_{\rho}\left[Q^{\pi}(s,\pi(s))-Q_{t}(s,\pi_{t}^{\mathrm{mid}})\}\right]
\leq \mathbb{E}_{\rho}\left[Q^{\pi}(s,\pi(s))-Q_{t}(s,\pi(s))\right].
\]
Rewriting using the matrix notation gives
\begin{equation}
\mathbb{E}_{\rho}\left[V^{\pi}(s)-V^{\pi_{t}}(s)\right]\leq\mathbb{E}_{\rho}\left[Q^{\pi}(s,\pi(s))-Q_{t}(s,\pi(s))\right]=\rho^{\pi}(Q^{\pi}-Q_{t}).\label{eq:value-diff-first}
\end{equation}
Now we are ready to apply the third claim in Proposition~\ref{prop:monotonicity_data_splitting} to deduce that on the event $\cE_{\text{MDP}}$:
\begin{align*}
Q^{\pi}-Q_{t} & \leq\gamma P^{\pi}(Q^{\pi}-Q_{t-1})+2 b_{t}\leq\gamma P^{\pi}\left[ \gamma P^{\pi}(Q^{\pi}-Q_{t-2})+2 b_{t-1}\right] +2 b_{t}\\
 & \leq\cdots\\
 & \le\gamma^{t}(P^{\pi})^{t}(Q^{\pi}-Q_{0})+2\sum_{j=1}^{t}(\gamma P^{\pi})^{t-j}b_{j}\\
 & \leq\frac{\gamma^{t}}{1-\gamma}\mathbf{1}+2\sum_{j=1}^{t}(\gamma P^{\pi})^{t-j}b_{j}.
\end{align*}
Here $\mathbf{1}$ denotes the all-one vector with dimension $S|\mathcal{A}|$,
and the last inequality arises from the fact that $Q^{\pi}-Q_{0}=Q^{\pi}\leq(1-\gamma)^{-1}\mathbf{1}$.
Multiplying both sides the of the equation above by $\rho^{\pi}$,
we conclude that
\begin{equation}
\rho^{\pi}(Q^{\pi}-Q_{t})\leq\frac{\gamma^{t}}{1-\gamma}+2\sum_{j=1}^{t}\rho^{\pi}(\gamma P^{\pi})^{t-j}b_{j}=\frac{\gamma^{t}}{1-\gamma}+2\sum_{j=1}^{t}v_{t-j}^{\pi}b_{j},\label{eq:value-diff-second}
\end{equation}
where we use the definition of $v_{k}^{\pi}=\rho^{\pi}(\gamma P^{\pi})^{k}$.
Combine the inequalities~\eqref{eq:value-diff-first} and~\eqref{eq:value-diff-second}
to reach the desired result.

\subsection{Proof of Theorem \ref{thm:MDP_upperbound_hoffding}}\label{app:proof_MDP_upperbound}
Similar to the proof given for contextual bandits, we prove a stronger result than Theorem~\ref{thm:MDP_upperbound_hoffding}. Fix any deterministic expert policy $\pi$.
Assume that the data coverage assumption holds, that is
\begin{align*}
    \max_{s,a} \frac{{d}^\pi(s,a)}{\mu(s,a)} \leq C^\pi. 
\end{align*}
Then for all $C^\pi \geq 1$, Algorithm \ref{alg:OVI-LCB-DS} with $\delta = 1/N$ achieves
\begin{align}
    \E_\cD \left[J(\pi) - J(\hat{\pi})\right]\lesssim \min \left(\frac{1}{1-\gamma}, \sqrt{\frac{SC^\pi}{(1-\gamma)^5 N}}\right).
    \label{eq:MDP-upper-bound-case-1-proof}
\end{align}
In addition, if $1 \leq C^\pi \leq 1+\frac{L\log(N)}{200(1-\gamma)N}$, then we have a tighter performance upper bound
\begin{align}\label{eq:MDP-upper-bound-case-2-proof}
    \E_\cD \left[J(\pi) - J(\hat{\pi})\right] \lesssim \min \left(\frac{1}{1-\gamma}, \frac{S}{(1-\gamma)^4N} \right).
\end{align}
The result in Theorem~\ref{thm:MDP_upperbound_hoffding} can be recovered by taking $\pi=\pi^\star$.

We split the proof into two cases: (1) the general case when $C^\pi\geq 1$ and (2) 
the regime where $C^\pi \leq 1 + L / (200m)$. 

\paragraph{The general case when $C^\pi \geq 1$.}
The proof of the general case follows similar steps as those in the proof of Theorem~\ref{thm:LCB_CB_upper_bound}.
We first decompose the expected sub-optimality into three terms: 
\begin{align*}
&\E_\cD \left[\sum_s \rho(s) [V^\pi(s) - V^{\pi_T}(s)] \right]  \\
&\quad =\E_\cD \left[\sum_s \rho(s) [V^\pi(s) - V^{\pi_T}(s)] \ind \{ \exists t\leq T, m_t(s,\pi(s)) = 0 \}\right] \eqqcolon T_1\\
 & \quad \quad+\E_\cD \left[\sum_s \rho(s) [V^\pi(s) - V^{\pi_T}(s)] \ind \{\forall t\leq T, m_t(s,\pi(s)) \geq 1 \} \ind \{\cE_{\text{MDP}}\} \right] \eqqcolon T_2\\
 & \quad \quad+\E_\cD \left[\sum_s \rho(s) [V^\pi(s) - V^{\pi_T}(s)] \ind \{\forall t\leq T, m_t(s,\pi(s)) \geq 1 \} \ind \{\cE^c_{\text{MDP}}\} \right] \eqqcolon T_3.
\end{align*}
Similar to before, the first term $T_1$ captures the sub-optimality incurred by the missing mass on the expert action $\pi(s)$. The second term $T_2$ is the sub-optimality under the clean event $\cE_{\text{MDP}}$, while the last one $T_3$ denotes the sub-optimality suffered under the complement event  $\cE_{\text{MDP}}^c$, on which the empirical average of Q-function falls outside the constructed confidence interval.

As we will show in subsequent sections, these error terms satisfy the following upper bounds:
\begin{subequations}
\begin{alignat}{2}
T_{1} & \leq \frac{4SC^\pi (T+1)^2}{9(1-\gamma)^2N};\label{eq:missing-mass-MDP}\\
T_{2} & \leq \frac{\gamma^{T}}{1-\gamma} + 32 \frac{1}{(1-\gamma)^2} \sqrt{\frac{LSC^\pi (T+1)}{N}};
\label{eq:hard-term-MDP-crude}\\
T_{3} & \leq V_{\max}\delta.\label{eq:event-MDP}
\end{alignat}
\end{subequations}
Setting $\delta = 1/ N$, $T = \log N / (1 - \gamma)$ and noting that $\gamma^T\leq 1/N$ yield 
that
\begin{align*}
    \E_\cD \left[J(\pi) - J(\hat{\pi})\right]\lesssim  \left( \sqrt{\frac{SC^\pi}{(1-\gamma)^5 N}}+ \frac{ SC^\pi }{(1-\gamma)^4N}\right).
\end{align*}
Note that we always have $ \E_\cD \left[J(\pi) - J(\hat{\pi})\right]\leq \frac{1}{1-\gamma}$. In the interesting regime of $\frac{SC^\pi}{(1-\gamma)^3N}\leq 1$, the first term above always dominates. This gives 
the desired claim~\eqref{eq:MDP-upper-bound-case-1-proof}.

\paragraph{The case when $C^\pi \leq 1+ L/(200m)$.} 
Under this circumstance, the following lemma proves useful.
\begin{lemma}\label{lemma:reduction_imitation}
   For any deterministic policy $\hat \pi$, one has
    \begin{align}
        J(\pi) - J(\hat{\pi}) \leq V_{\max}^2 \mathbb{E}_{s\sim{d}_{\pi}}\left[\ind \left\{ \hat \pi(s)\neq\pi(s)\right\} \right].
    \end{align}
\end{lemma}
\begin{proof}
In view of the performance difference lemma in~\citet[Lemma 6.1]{kakade2002approximately}, one has
\begin{align*}
J(\pi)-J(\hat \pi) & =\frac{1}{1-\gamma}\mathbb{E}_{s\sim{d}_{\pi}}\left[Q^{\hat \pi}(s,\pi(s))-Q^{\hat \pi}(s,\hat \pi(s))\right]\\
 & =\frac{1}{1-\gamma}\mathbb{E}_{s\sim{d}_{\pi}}\left[\left[Q^{\hat \pi}(s,\pi(s))-Q^{\hat \pi}(s,\hat \pi(s))\right]\ind\left\{ \hat\pi(s)\neq \pi(s)\right\} \right]\\
 & \leq V_{\max}^2\mathbb{E}_{s\sim{d}_{\pi}}\left[\ind\left\{ \hat \pi(s)\neq\pi(s)\right\} \right].
\end{align*}
Here the last  line uses the fact that $Q^{\hat\pi}(s,\pi(s))-Q^{\hat\pi}(s,\hat\pi(s))\leq V_{\max}$.
\end{proof}

Lemma~\ref{lemma:reduction_imitation} links the sub-optimality of a policy to its disagreement with the optimal policy. 
With Lemma~\ref{lemma:reduction_imitation} at hand, we can continue to decompose the expected sub-optimality into:
 \begin{align*}
&\E_\cD \left[\sum_s \rho(s) [V^\pi(s) - V^{\pi_T}(s)] \right]  \\
&\quad\leq  V_{\max}^{2}\E_\cD[\E_{s \sim d_\pi}[ \ind \{ \pi_T(s) \neq \pi(s) \}]] \\
&\quad= V_{\max}^{2}\E_\cD[\E_{s \sim d_\pi}[ [\ind \{ \pi_T(s) \neq \pi(s) \}\ind \{ \exists t\leq T, m_t(s,\pi(s)) = 0 \}]] \eqqcolon T_1'\\
&\quad \quad +  V_{\max}^{2}\E_\cD[\E_{s \sim d_\pi}[ [\ind \{ \pi_T(s) \neq \pi(s) \}\ind \{\forall t\leq T, m_t(s,\pi(s)) \geq 1 \}]] \eqqcolon T_2'
\end{align*}
We bound each term according to
\begin{subequations}
\begin{alignat}{2}
    T_{1}' & \leq \frac{4SC^\pi (T+1)^2}{9(1-\gamma)^2N};\label{eq:missing-term-MDP-small-C}\\
T_{2}' & \lesssim \frac{S C^\pi LT}{(1-\gamma)^2 N}+ \frac{ST^{10}}{(1-\gamma)^2N^9}. \label{eq:hard-term-MDP-small-C}
\end{alignat}
\end{subequations}
The claimed bound~\eqref{eq:MDP-upper-bound-case-2-proof} follows by taking $\delta = 1/N$ and $T = \log N /(1-\gamma)$.

\subsubsection{Proof of the bound~\eqref{eq:missing-mass-MDP} on $T_1$ and the bound~\eqref{eq:missing-term-MDP-small-C} on $T_1'$}\label{app:T_1_MDP_Proof}
Since for any $s \in \cS$, $V^\pi(s) - V^{\pi_T}(s) \leq V_{\max}$ one has
\begin{align*}
    T_1 \leq V_{\max} \E_\cD \left[ \sum_{s} \rho(s)  \ind \{ \exists t\leq T, m_t(s,\pi(s)) = 0 \} \right] = V_{\max}  \sum_{s} \rho(s)  \mathbb{P} \left( \exists t\leq T, m_t(s,\pi(s)) = 0   \right).
\end{align*}
The definition of the normalized occupancy measure \eqref{def:state_action_occupancy} entails $\rho(s) \leq d^\pi(s,\pi(s))$ and thus
\begin{align*}
    \frac{\rho(s)}{\mu(s,\pi(s))} \leq \frac{1}{1-\gamma} \cdot \frac{d^\pi(s,\pi(s))}{\mu(s,\pi(s))} \leq \frac{C^\pi}{1-\gamma}.
\end{align*}
Here the last relation follows from the data coverage assumption. 
Combine the above two inequalities to see that
\begin{align*}
    T_1 & \leq V_{\max}  \sum_{s} \frac{C^\pi}{1-\gamma} \mu(s, \pi(s)) \mathbb{P} \left( \exists t\leq T, m_t(s,\pi(s)) = 0   \right) \\
    & = \frac{C^\pi}{(1-\gamma)^2} \sum_s \mu(s,\pi(s)) \prob\left(\exists t\leq T, m_t(s,\pi(s)) = 0\right) \\
    & \leq \frac{C^\pi}{(1-\gamma)^2} \sum_{t=0}^{T} \sum_s \mu(s,\pi(s)) \prob\left(m_t(s,\pi(s)) = 0\right),
\end{align*}
where in the penultimate line, we identify $V_{\max}$ with $1 / (1- \gamma)$, and the last relation is by the union bound.
Direct calculations yield
\begin{equation*}
    \prob\left(m_t(s,\pi(s)) = 0\right) = (1-\mu(s,\pi(s)))^m,
\end{equation*}
which further implies
\begin{align*}
    T_1  \leq \frac{C^\pi (T+1)}{(1-\gamma)^2} \sum_s \mu(s,\pi(s)) (1-\mu(s,\pi(s)))^m
    \leq \frac{4C^\pi S(T+1)}{9(1-\gamma)^2m} = \frac{4C^\pi S(T+1)^2}{9(1-\gamma)^2N}.
\end{align*}
Here, we have used $\max_{x \in [0,1]} x(1-x)^m \leq 4/(9m)$ and the fact that $m = N/(T+1)$.

The bound~\eqref{eq:missing-term-MDP-small-C} on $T_1'$ follows from exactly the same argument as above, except that we replace $\rho$ with $d^\pi$.

\subsubsection{Proof of the bound~\eqref{eq:hard-term-MDP-crude} on $T_2$}
Lemma \ref{lemma:value_difference} asserts that on the clean event $\cE_{\text{MDP}}$, one has
\begin{align}
    T_2 & \leq \frac{\gamma^T}{1-\gamma} + 2 \sum_{t=1}^{T}  \E_{\cD,\nu^\pi_{T-t}} \left[ b_t(s,\pi(s)) \ind \{ m_t(s,\pi(s)) \geq 1 \} \right] \nonumber\\
    &= \frac{\gamma^T}{1-\gamma} + 2 \sum_{t=1}^{T}  \E_{\cD,\nu^\pi_{T-t}} \left[ V_{\max} \sqrt{\frac{L}{m_t(s, \pi(s))}} \ind \{ m_t(s,\pi(s)) \geq 1 \} \right] \nonumber\\
    &\leq \frac{\gamma^T}{1-\gamma} + 2 \sum_{t=1}^{T}  \E_{\nu^\pi_{T-t}} \left[ 16 V_{\max} \sqrt{\frac{L}{m\mu(s,\pi(s))}}  \right]. \label{eq:MDP-T-2}
\end{align}
Here, we substitute in the definition of $b_t(s,a)$ in the middle line and the last inequality arises from Lemma~\ref{lemma:binomial_inverse_moment_bound} with $c_{1/2} \leq 16$.

By definition of $\nu^\pi_k = \rho^\pi (\gamma P^\pi)^k$, we have $\sum_{k=0}^\infty \nu_{k}^\pi = d^\pi / (1-\gamma)$. Therefore, one has
\begin{align*}
    \sum_{t=1}^{T} \E_{\nu_{T-t}^\pi} \left[ \frac{1}{\sqrt{\mu(s,\pi(s))}}\right] = & \sum_{t=1}^{T} \sum_{s} \nu_{T-t}^\pi(s,\pi(s)) \frac{1}{\sqrt{\mu(s,\pi(s))}}\\
    = & \sum_s \left[ \sum_{t=1}^{T} \nu_{T-t}^\pi(s,\pi(s)) \right] \frac{1}{\sqrt{\mu(s,\pi(s))}}\\
    \leq & \sum_s \frac{d^\pi(s,\pi(s))}{1-\gamma} \frac{1}{\sqrt{\mu(s,\pi(s))}}.
\end{align*}
We then apply the concentrability assumption and the Cauchy--Schwarz inequality to deduce~that  
\begin{align*}
    \sum_{t=1}^{T} \E_{\nu_{T-t}^\pi} \left[ \frac{1}{\sqrt{\mu(s,\pi(s))}}\right] & \leq \sqrt{\frac{C^\pi}{(1-\gamma)^2}} \sum_s \sqrt{d^\pi(s,\pi(s))}\\
    & \leq \sqrt{\frac{C^\pi}{(1-\gamma)^2}} \sqrt{S} \sqrt{\sum_s d^\pi(s,\pi(s))}\\
    & = \frac{\sqrt{SC^\pi}}{1-\gamma}.
\end{align*}
Substitute the above bound into the inequality~\eqref{eq:MDP-T-2} to arrive at the conclusion
\begin{equation*}
    T_2 \leq \frac{\gamma^T}{1-\gamma} + 32 \frac{1}{(1-\gamma)^2} \sqrt{\frac{LSC^\pi}{m}}.
\end{equation*}
The proof is completed by noting that $m = N/(T+1)$.

\subsubsection{Proof of the bound~\eqref{eq:event-MDP} on $T_3$}
It is easy to see that 
\begin{align*}
    \sum_s \rho(s) [V^\pi(s) - V^{\pi_T}(s)] \ind \{\forall s,t, m_t(s,\pi(s)) \geq 1 \} \leq V_{\max},
\end{align*}
which further implies
\begin{align*}
    T_3 \leq V_{\max} \E_\cD[\ind \{\cE^c_{\text{MDP}}\}] = V_{\max} \prob( \cE^c_{\text{MDP}}) \leq V_{\max} \delta.
\end{align*}
Here, the last bound relies on Lemma \ref{lemma:penalty_LCB}.

\subsubsection{Proof of the bound~\eqref{eq:hard-term-MDP-small-C} on $T_2'$}

Partition the state space into the following two disjoint
sets:
\begin{subequations}
\begin{align}
\mathcal{S}_{1} & \coloneqq\left\{ s\mid  d_\pi(s)<\frac{2 C^{\pi}L}{m}\right\} ,\label{eq:defn-S-1-MDP}\\
\mathcal{S}_{2} & \coloneqq\left\{ s\mid d_\pi(s)\geq\frac{2 C^{\pi}L}{m}\right\},\label{eq:defn-S-2-MDP}
\end{align}
\end{subequations}
In words, the set $\mathcal{S}_{1}$ includes the
states that are less important in evaluating the performance of LCB. 
We can then decompose
the term $T_{2}'$ accordingly:
\begin{align*}
T_{2}' & =V_{\max}^2\sum_{s\in\mathcal{S}_{1}}d_\pi(s) \E_\cD[\ind \{ \pi_T(s) \neq \pi(s) \}\ind \{\forall t, m_t(s,\pi(s)) \geq 1 \}]\eqqcolon T_{2,1}\\
 & \quad+ V_{\max}^2\sum_{s\in\mathcal{S}_{2}}d_\pi(s)\E_\cD[\ind \{ \pi_T(s) \neq \pi(s) \}\ind \{\forall t, m_t(s,\pi(s)) \geq 1 \}] \eqqcolon T_{2,2}.
\end{align*}
The proof is completed by observing the following two upper bounds: 
\begin{align*}
    T_{2,1}  \leq \frac{2 S C^\pi LT}{(1-\gamma)^2 N}, \qquad \text{and} \qquad
    T_{2,2}  \lesssim \frac{S}{(1-\gamma)^2} \left( \frac{T}{N}\right)^9.
\end{align*}

\paragraph{Proof of the bound on $T_{2,1}$. }
We again use the basic fact that 
\[
\E_\cD[\ind \{ \pi_T(s) \neq \pi(s) \}\ind \{\forall s,t, m_t(s,\pi(s)) \geq 1 \}]\leq 1
\]
to reach
\[
T_{2,1}\leq V_\max^2 \sum_{s\in\mathcal{S}_{1}}d_\pi(s)\leq\frac{2SC^{\pi}L}{(1-\gamma)^2m},
\]
where the last inequality hinges on the definition of $\mathcal{S}_{1}$ given in~\eqref{eq:defn-S-1-MDP}, namely for any $s\in\mathcal{S}_{1}$, one
has $d_\pi(s)<\frac{2 C^{\pi}L}{m}$.
Identifying $m$ with $N/(T+1)$ concludes the proof.

\paragraph{Proof of the bound on $T_{2,2}$. } 

Equivalently, we can write $T_{2,2}$ as 
\[
T_{2,2}=V_{\max}^2 \sum_{s\in\mathcal{S}_{2}}d_\pi(s)\mathbb{P}\left(\pi_{T}(s)\neq\pi(s),\; m_{t}(s,\pi(s))\geq1 \; \forall t \right).
\]
By inspecting Algorithm~\ref{alg:OVI-LCB-DS}, one can realize the following inclusion
\[
\{\pi_T(s)\neq\pi(s)\}\subseteq\{\pi_0(s) \neq \pi(s)\}\cup\{\exists0\leq t\leq T-1\text{ and }\exists a\neq\pi(s),Q_{t+1}(s,a)\geq Q_{t+1}(s,\pi(s))\}.
\]
Indeed, if $\pi_0(s) = \pi(s)$ and for all $t$, $Q_{t+1}(s, \pi(s)) > \max_{a \neq \pi(s)} Q_{t+1}(s,a)$, LCB would select the expert action in the end, i.e., $\pi_T(s) = \pi(s)$. 
Therefore, we can upper bound $T_{2,2}$ as 
\begin{align*}
T_{2,2} & \leq V_{\max}^2\sum_{s\in\mathcal{S}_{2}}d_\pi(s)\mathbb{P}\left(\pi_0(s) \neq \pi(s),m_{t}(s,\pi(s))\geq1 \; \forall t\right)\eqqcolon\beta_{1}\\
 & +V_{\max}^2\sum_{s\in\mathcal{S}_{2}}d_\pi(s)\mathbb{P}\left(\exists t\leq T-1, \exists a\neq\pi(s),Q_{t+1}(s,a)\geq Q_{t+1}(s,\pi(s)),m_{t}(s,\pi(s))\geq1 \; \forall t\right)\eqqcolon\beta_{2}.
\end{align*}
In the sequel, we bound $\beta_{1}$ and $\beta_{2}$ in the reverse
order.

\subparagraph{Bounding $\beta_2$.}

Fix a state $s\in\mathcal{S}_{2}$. In view of the data coverage assumption, one has 
\begin{align}\label{eq:MDP-prob-lower}
    \mu(s, \pi(s)) \geq \frac{1}{C^\pi} d_\pi(s) \geq \frac{1}{C^\pi}\frac{2 C^\pi L}{m} = \frac{2 L}{m}.
\end{align}
In contrast, for any $a \neq \pi(s)$, since $C^\pi \leq 1+ \frac{L}{200m}$, we have  
\begin{align}\label{eq:MDP-prob-upper}
    \mu(s,a) \leq \sum_{a \neq \pi(s)}\mu(s,a) \leq 1-\frac{1}{C^\pi} \leq \frac{L}{200m},
\end{align}
where the middle inequality reuses the concentrability assumption. 
One has $\mu(s, \pi(s)) \gg \mu(s,a)$ for any 
non-expert action $a$. As a result, the expert action is pulled more frequently than the others.
It turns out that under such circumstances, the LCB algorithm picks the expert action with high probability. We shall make this intuition precise below. 

The bounds~\eqref{eq:MDP-prob-lower} and~\eqref{eq:MDP-prob-upper} together with Chernoff's bound give
\begin{align*}
    \prob \left(  m_t(s,a) \leq \frac{5L}{200} \right) & \geq 1 -  \exp \left(-\frac{L}{200} \right); \\
    \prob \left(  m_t(s,\pi(s)) \geq L \right) & \geq 1 -  \exp \left(-\frac{L}{4} \right).
\end{align*}
These allow us to obtain an upper bound for the function $Q_{t+1}$ evaluated at non-expert actions and a lower bound on $Q_{t+1}(s,\pi(s))$. More precisely, when $m_t(s,a) \leq \frac{5L}{200}$, we have
\begin{align*}
    Q_{t}(s,a) & = r_t(s,a) - b_t(s,a) + \gamma P_{s,a}^t \cdot V_{t-1}\\ 
    & = r_t(s,a) -  V_{\max}\sqrt{\frac{L}{m_t(s,a)\vee 1}} + \gamma P_{s,a}^t \cdot V_{t-1} \\
    & \leq 1 -  V_{\max} \sqrt{\frac{L}{5L/200}} + \gamma V_{\max} \\
    & \leq - 5 V_{\max}.
\end{align*}
Here we used the fact that $L \geq 70$.
Now we turn to lower bounding the function $Q_{t}$ evaluated at the optimal action. When $m_t(s,\pi(s)) \geq L$, one has
\begin{align*}
    Q_{t}(s,\pi(s)) & = r_t(s,\pi(s)) -  V_{\max}\sqrt{\frac{L}{m_t(s,\pi(s))}} + \gamma P_{s,\pi(s)}^t \cdot V_{t-1} \geq  - V_{\max}.
\end{align*}

To conclude, if both $m_t(s,a) \leq \frac{5L}{200}$ and $m_t(s,\pi(s)) \geq L$ hold, we must have $Q_{t}(s, a)< Q_{t}(s, \pi(s))$. 
Therefore  we can deduce that 
\begin{align*}
 & \mathbb{P}\left(\exists0\leq t\leq T\text{ and }\exists a\neq\pi(s),Q_{t}(s,a)\geq Q_{t}(s,\pi(s)),m_{t}(s,\pi(s))\geq1 \; \forall t\right)\\
 & \quad\le\sum_{0\leq t\leq T}\mathbb{P}\left(\exists a\neq\pi(s),Q_{t}(s,a)\geq Q_{t}(s,\pi(s)),m_{t}(s,\pi(s))\geq1 \; \forall t\right)\\
 & \quad\leq\sum_{0\leq t\leq T-1}\left\{ (|\mathcal{A}|-1)\exp\left(-\frac{L}{200}\right)+\exp\left(-\frac{1}{4}  L\right)\right\} \\
 & \quad\leq T|\mathcal{A}|\exp\left(-\frac{L}{200}\right),
\end{align*}
which further implies
\begin{align*}
\beta_{2} & \leq V_{\max}^2\sum_{s\in\mathcal{S}_{2}}d_\pi(s)T|\mathcal{A}|\exp\left(-\frac{L}{200}\right)\\
 & \leq TV_{\max}|\cA|\cdot \frac{1}{1-\gamma}\exp\left(-\frac{L}{200}\right)\\
 & \lesssim Tm^{-9}.
\end{align*}

\subparagraph{Bounding $\beta_1$.}
In fact, the analysis of $\beta_2$ has revealed that with high probability, $\pi(s)$ is the most played arm among all actions. More precisely, we have
\begin{align*}
\beta_{1} & \leq V_{\max}^2\sum_{s\in\mathcal{S}_{2}}d_\pi(s)\mathbb{P}\left(\pi_{0}(s)\neq\pi(s)\right)\\
 & \leq V_{\max}^2\sum_{s\in\mathcal{S}_{2}}d_\pi(s)\left\{ \mathbb{P}\left(\max_{a}m_{0}(s,a)\geq\frac{5L}{200}\right)+\mathbb{P}\left(m_{0}(s,\pi(s))\leq L\right)\right\} \\
 & \leq V_{\max}^2|\mathcal{A}|\exp\left(-\frac{L}{200}\right)\lesssim \frac{1}{(1-\gamma)^2m^{-9}}.
\end{align*}

\noindent Combine the bounds on $\beta_1$ and $\beta_2$ to arrive at the claim on $T_{2,2}$.

\subsection{Proof of Theorem~\ref{thm:MDP_lower}}\label{app:proof_MDP_lower}
Similar to the proof of the lower bound for contextual bandits, we split the proof into three cases: (1) $C^\star = 1$, (2) $C^\star \geq 2$, and (3) $C^\star \in(1, 2)$. For $C^\star = 1$, we adapt the lower bound from episodic imitation learning~\cite{rajaraman2020toward} to the discounted case. For both $C^\star \in(1, 2)$ and $C^\star\geq 2$, we rely on the construction of the MDP in the paper~\citet{lattimore2012pac}, which reduces the policy learning problem in MDP to a bandit problem. The key difference is that in our construction, we need to carefully design the initial distribution $\rho$ to incorporate the effect of $C^\star$ in the lower bound.

\paragraph{The case when $C^\star = 1$.}
In this case we have $\mu(s, a) = {d}^\star(s, a)$ for all $(s, a)$ pairs, which is the imitation learning setting. We adapt the lower bound given in \citet{rajaraman2020toward} for episodic imitation learning to the discounted case and obtain the following lemma:

\begin{lemma} \label{lem:IL_LB_MDP_lower}
When $C^\star=1$, one has 
 \begin{align}\label{eq:hard_MDP_lower_C_1}
\inf_{\hat \pi} \sup_{(\rho, \mu, P, R) \in \mathsf{MDP}(1)}  \E_\cD[J(\pi^{\star})-J(\hat{\pi})]  \gtrsim \min \left\{ \frac{1}{1-\gamma}, \frac{S}{ (1-\gamma)^2N} \right\}.
 \end{align}
\end{lemma}

\noindent We defer the proof to Appendix~\ref{app:proof_lem_IL_lower}, which follows exactly the analysis by~\citet{rajaraman2020toward} except for changing the setting from episodic to discounted.

\paragraph{The case when $C^\star\geq 2$.}
When $C^\star\geq 2$, we intend to show that
\begin{align}\label{eq:MDP-lower-bound-C-2}
\inf_{\hat \pi} \sup_{(\rho, \mu, P, R) \in \mathsf{MDP}(C^\star)}  \E_\cD[J(\pi^{\star})-J(\hat{\pi})]\gtrsim\min\left(\frac{1}{1-\gamma}, \sqrt{\frac{SC^\star}{(1-\gamma)^3 N}}\right).
\end{align}
We adopt the following construction of the hard MDP instance from the work~\citet{lattimore2012pac}. 

\subparagraph{Construction of hard instances.}
Consider the MDP  which consists of $S/4$ replicas of MDPs in Figure~\ref{fig:hard_MDP_lower} and an extra state $s_{-1}$. The total number of states is $S+1$.  
For each replica,  we have four states  $ s_0,s_1,s_\oplus,s_\ominus$. 
There is only one action, say $a_1$, in all the states except $s_1$, which has two actions $a_1, a_2$. 
The rewards are all deterministic. 
In addition, the transitions for states $s_0, s_\oplus, s_\ominus$ are shown in the diagram. 
More specifically, we have $\mathbb{P}(s^j_\oplus \mid s^j_1, a_1) = \mathbb{P}(s^j_\ominus \mid s^j_1, a_1) = 1/2$ and $\mathbb{P}(s^j_\oplus \mid s^j_1, a_2) = 1/2 + v_j \delta$, and  $\mathbb{P}(s^j_\ominus \mid s^j_1, a_2) = 1/2 - v_j \delta$. Here $v_j \in \{-1, +1\}$ is the design choice associated with the $j$-th replica and $\delta \in [0,1/4]$ will be specified later. Clearly, if $v_j = 1$, the optimal action at $s_1^j$ is $a_2$, otherwise, the optimal one is $a_1$. Under the extra state $s_{-1}$, there is only one action with reward $0$ which transits to itself with probability $1$. 
We use $s_i^j$ to denote state $i$ in $j$-th replica, where $j\in [S/4]$. 
Based on the description above, the only parameter in this MDP is the transition dynamics associated 
with the state $s_1^j$. We will later specify how to set these for each $s_1^j$.
\begin{figure}[t]
    \centering
    \includegraphics[width=0.7\linewidth]{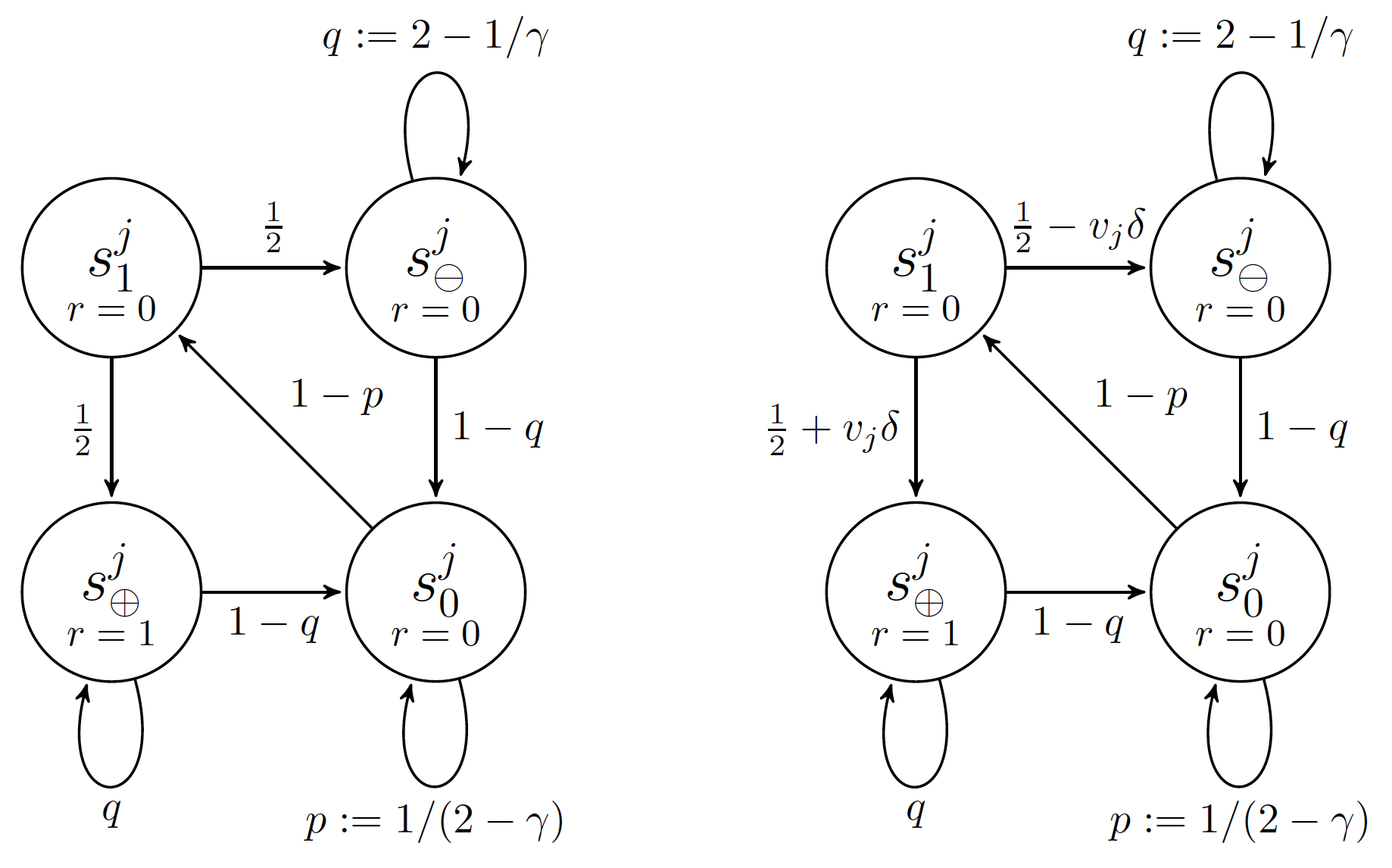}
    \caption{Illustration of one replica in the hard $\mathsf{MDP}_h$. The left plot shows the transition probabilities from $(s_1^j, a_1)$ and the right plot shows them from $(s_1^j, a_2)$.}
    \label{fig:hard_MDP_lower}
\end{figure}
The single replica has the following important properties: 
\begin{enumerate}
\item The probabilities $p, q$ are designed such that the three states $s_0, s_\ominus, s_\oplus$ are mostly absorbing, while any action in $s_1$ will lead to immediate transition to $s_\oplus$ or $s_\ominus$.
\item The state $s_\oplus$ is the only state that gives reward $1$, which helps reduce the MDP problem to a bandit one: the MDP only depends on the choice of transition probabilities at state $s_1^j$; once a policy reaches state $s_1$ it should choose the action most likely to lead to
state $\oplus$ whereupon it will either be rewarded or punished (visit state $\oplus$ or $\ominus$). Eventually, it will return to state $1$ where the whole process repeats. 
\end{enumerate}

We also need to specify the initial distribution $\rho_{0}$ and the behavior distribution $\mu_{0}$. 
When $C^\star \geq 2$, we set the initial distribution $\rho_0$ to be uniformly distributed on the state $s_0$ in all the $S/4$ replicas, i.e., $\forall j\in[S/4], \rho_0(s_0^j) = 4/S$. 
From  ${d}^\star = (1-\gamma)\rho(I-\gamma P^{\pi^\star})^{-1}$ we can derive ${d}^\star$ as follows:
\begin{align*}
{d}^\star(s_0^j) &= 
\frac{8}{(2+\gamma)S},  \qquad {d}^\star(s_1^j)  = \frac{8\gamma(1-\gamma)}{(2-\gamma)(2+\gamma)S}\in\left[\frac{1-\gamma}{S}, \frac{4(1-\gamma)}{S}\right],\\
{d}^\star(s_\oplus^j) &= \frac{\gamma(\frac{1}{2}\ind \{v_j =-1\} + (\frac{1}{2}+\delta) \ind \{v_j = 1\})}{2(1-\gamma)}\cdot {d}^\star(s_1^j),  \\
   {d}^\star(s_\ominus^j) &= \frac{\gamma(\frac{1}{2}\ind \{v_j =1\} + (\frac{1}{2}-\delta) \ind \{v_j = -1\})}{2(1-\gamma)}\cdot {d}^\star(s_1^j), \qquad {d}^\star(s_{-1}) =0.
\end{align*} 
This allows us to construct the behavior distribution $\mu_0$ as follows: 
\begin{align*}
\mu_0(s_0^j) &= \frac{{d}^\star(s_0^j)}{C^\star},  \qquad \mu_0(s_1^j, a_2)  = \frac{{d}^\star(s_1^j)}{C^\star}, \qquad \mu_0(s_1^j, a_1)  = {d}^\star(s_1^j)\cdot \left(1-\frac{1}{C^\star}\right)\\
\mu_0(s_\oplus^j) &= \frac{3}{4}\cdot \frac{\gamma}{2(1-\gamma)C^\star}\cdot {d}^\star(s_1^j),   \qquad   \mu_0(s_\ominus^j) =  \frac{1}{2}\cdot \frac{\gamma}{2(1-\gamma)c^\star}\cdot {d}^\star(s_1^j), \\
\mu_0(s_{-1}) &= 1 - \sum_j (\mu_0(s_0^j) + \mu_0(s_1^j) + \mu_0(s_\oplus^j)+\mu_0(s_\ominus^j))
\end{align*} 
It is easy to check that for any $v_j \in \{-1,1\}$, $\delta \in [0,1/4]$, one has $\mu_0(s_{-1})>0$, and more importantly
\[
(\rho_{0},\mu_{0},P, R)\in\mathsf{MDP}(C^{\star}).
\]
Since in this construction of MDP, the reward distribution is deterministic and fixed, and we only need to change the transition dynamics $P$, which is governed by the choice of $\delta$ and ${v_j}_{1\leq k \leq S/4}$. Hence we write the loss/sub-optimality of a policy $\pi$ w.r.t.~a particular design of $P$ as 
\[
\mathcal{L}(\pi;P)= J_{P}(\pi^{\star})-J_{P}(\pi) .
\]
Our target then becomes
\begin{align*}
\inf_{\hat{\pi}}\sup_{(\rho_0, \mu_0, P, R) \in \mathsf{MDP}(C^\star)}\mathbb{E}[\mathcal{L}(\hat \pi;P)]\gtrsim\min\left(\frac{1}{1-\gamma}, \sqrt{\frac{SC^\star}{(1-\gamma)^3 N}}\right).
\end{align*}

It remains to construct a set of transition probabilities (determined by $\delta$ and $\bm{v}$) that are nearly
indistinguishable given the data. Similar to the construction in the lower bound for contextual bandits, we leverage
the Gilbert-Varshamov lemma (cf.~Lemma~\ref{lem:V_G}) to obtain a set
$\mathcal{V}\subseteq\{-1,1\}^{S/4}$ that obeys (1) $|\mathcal{V}|\ge\exp(S/32)$
and (2) $\|\bm{v}_{1}-\bm{v}_{2}\|_{1}\geq S/8$ for any $\bm{v}_{1},\bm{v}_{2}\in\mathcal{V}$
with $\bm{v}_{1}\neq\bm{v}_{2}$. 
Each element $\bm{v}\in\mathcal{V}$ is mapped to a transition probability at $s_1^j$ such that the probability  of transiting to $s_\oplus^j$
associated with $(s_1^j,a_{2})$ is $\frac{1}{2}+v_{j}\delta$.
We denote the resulting set of transition probabilities as $\cP$.
We record a useful characteristic of this family $\cP$ of transition dynamics below, which results from the second property of the set $\mathcal{V}$.
\begin{lemma}\label{lem:MDP_fano}
For any policy $\pi$ and any two different transition probabilities
$P_{1},P_{2}\in\mathcal{P}$, the following holds:
\[
\mathcal{L}(\pi;P_{1})+\mathcal{L}(\pi;P_{2})\geq \frac{\delta}{32(1-\gamma)}.
\]
\end{lemma}

\subparagraph{Application of Fano's inequality. }
We are now ready to apply Fano's inequality, that is 
\[
\inf_{\hat{\pi}}\sup_{P\in\mathcal{P}} \mathbb{E}[\mathcal{L}(\hat \pi;P)]\geq\frac{\delta}{64(1-\gamma)}\left(1-\frac{N\max_{i\neq j}\mathsf{KL}\left(\mu_0\otimes P_{i}\|\mu_0\otimes P_{j}\right)+\log2}{\log|\mathcal{P}|}\right).
\]
It remains to controlling $\max_{i\neq j}\mathsf{KL}\left(\mu_0\otimes P_{i}\|\mu_0\otimes P_{j}\right)$
and $\log|\mathcal{P}|$. For the latter quantity, we have 
\[
\log|\mathcal{P}|=\log|\mathcal{V}|\geq S/32,
\]
where the inequality comes from the first property of the set $\mathcal{V}$.
With regards to the KL divergence, one has
\[
\max_{i\neq j}\mathsf{KL}\left(\mu_0\otimes P_{i}\|\mu_0\otimes P_{j}\right)\leq\frac{4(1-\gamma)}{SC^\star}\cdot \frac{S}{4}\cdot16\delta^{2} = \frac{16(1-\gamma)\delta^2}{C^\star},
\]
since $\mu_0(s_1^j, a_2)\in[\frac{1-\gamma}{SC^\star}, \frac{4(1-\gamma)}{SC^\star}]$.
As a result, we conclude that as long as 
\[
\frac{c_3(1-\gamma) N\delta^2}{SC^\star }\leq 1
\]
for some universal constant $c_3$,
one has 
\[
\inf_{\hat{\pi}}\sup_{P}\mathbb{E}[[\mathcal{L}(\hat \pi;P)]\gtrsim\frac{\delta}{1-\gamma}.
\]
To finish the proof, we can set $\delta=\sqrt{\frac{SC^\star}{c_3(1-\gamma)N}}$
when $\sqrt{\frac{SC^\star}{c_3(1-\gamma)N}}<\frac{1}{4}$ and $\delta=\frac{1}{4}$
otherwise. This yields the desired lower bound~\eqref{eq:MDP-lower-bound-C-2}.

\paragraph{The case when $C^\star\in (1, 2)$.}

We intend to show that when $C^\star\in (1, 2)$,
\begin{equation}
\inf_{\hat \pi} \sup_{(\rho, \mu, P, R) \in \mathsf{MDP}(C^\star)}  \E_\cD[J(\pi^{\star})-J(\hat{\pi})]\gtrsim \min\left(\frac{C^\star-1}{1-\gamma}, \sqrt{\frac{S(C^\star-1)}{(1-\gamma)^3 N}}\right).\label{eq:MDP-lower-bound-C-1-2}
\end{equation}
The proof is similar to that of the previous case but with a different construction for $\rho_{0}$ and~$\mu_{0}$. 

\subparagraph{Construction of the hard instance.}
Let $\rho_0(s_0^j) = 4(C^\star-1)/S$, $\rho_0(s_{-1}) = 2-C^\star$. 
From  ${d}^\star = (1-\gamma)\rho(I-\gamma P^{\pi^\star})^{-1}$ we can derive ${d}^\star$ as follows.
\begin{align*}
{d}^\star(s_0^j) &= 
\frac{8(C^\star-1)}{(2+\gamma)S},  \qquad {d}^\star(s_1^j)  = \frac{8\gamma(1-\gamma)(C^\star-1)}{(2-\gamma)(2+\gamma)S}\in\left[\frac{(1-\gamma)(C^\star-1)}{S}, \frac{4(1-\gamma)(C^\star-1)}{S}\right],\\
{d}^\star(s_\oplus^j) &= \frac{\gamma(\frac{1}{2}\ind \{v_j =-1\} + (\frac{1}{2}+\delta) \ind \{v_j = 1\})}{2(1-\gamma)}\cdot {d}^\star(s_1^j),    \\
{d}^\star(s_\ominus^j) &= \frac{\gamma(\frac{1}{2}\ind \{v_j =1\} + (\frac{1}{2}-\delta) \ind \{v_j = -1\})}{2(1-\gamma)}\cdot {d}^\star(s_1^j), \qquad {d}^\star(s_{-1}) =2-C^\star.
\end{align*} 
This allows us to construct the behavior distribution $\mu_0$ as follows 
\begin{align*}
\mu_0(s_0^j) &= \frac{{d}^\star(s_0^j)}{C^\star},  \qquad \mu_0(s_1^j, a_1)  = \mu_0(s_1^j, a_2)  = \frac{{d}^\star(s_1^j)}{C^\star}\\
\mu_0(s_\oplus^j) &= \frac{3}{4}\cdot \frac{\gamma}{2(1-\gamma)}\cdot {d}^\star(s_1^j),   \qquad   \mu_0(s_\ominus^j) =  \frac{1}{2}\cdot \frac{\gamma}{2(1-\gamma)}\cdot {d}^\star(s_1^j), \\
\mu_0(s_{-1}) &= 1 - \sum_j (\mu_0(s_0^j) + \mu_0(s_1^j) + \mu_0(s_\oplus^j)+\mu_0(s_\ominus^j))
\end{align*} 
Again, one can check that for any $v_j \in \{-1,1\}$ and $\delta \in [0,1/4]$, we have $\mu_0(s_{-1}) > 0$ and 
\[
(\rho_{0},\mu_{0},P, R)\in\mathsf{MDP}(C^\star).
\]

We use the same family $\cP$ of transition probabilities as before. Following the same proof as Lemma~\ref{lem:MDP_fano} and noting that the initial distribution is multiplied by an extra $C^\star-1$ factor, we know that for any policy $\pi$, and any two different distributions $P_1, P_2\in\mathcal{P}$,
\[
\mathcal{L}(\pi;P_{1})+\mathcal{L}(\pi;P_{2})\geq \frac{(C^\star-1)\delta}{32(1-\gamma)}.
\] 

\subparagraph{Application of Fano's inequality. }

Now we are ready to apply Fano's inequality, that is 
\[
\inf_{\hat{\pi}}\sup_{P\in\mathcal{P}} \mathbb{E}[\mathcal{L}(\hat \pi;P)]\geq\frac{\delta}{64(1-\gamma)}\left(1-\frac{N\max_{i\neq j}\mathsf{KL}\left(\mu_0\otimes P_{i}\|\mu_0\otimes P_{j}\right)+\log2}{\log|\mathcal{P}|}\right).
\]
Now the KL divergence satisfies
\[
\mathsf{KL}(\mu_0\otimes P_{i}\|\mu_0\otimes P_{j})\leq\frac{4(1-\gamma)(C^\star-1)}{SC^\star}\cdot \frac{S}{4}\cdot16\delta^{2} = \frac{16(1-\gamma)(C^\star-1)\delta^2}{C^\star}.
\] 
Here the first inequality comes from that $\mu_0(s_1^j) = \frac{c_2(1-\gamma)(C^\star-1)}{SC^\star}$ for some constant $c_2\in[1, 4]$. 
As a result, we conclude that as long as 
\[
\frac{c_3(1-\gamma)(C^\star-1) N\delta^2}{SC^\star }\leq 1
\]
for some universal constant $c_3$,
one has 
\[
\inf_{\hat{\pi}}\sup_{P\in\mathcal{P}}\mathbb{E}[\mathcal{L}(\pi;P)]\gtrsim\frac{(C^\star-1)\delta}{1-\gamma}.
\]
To finish the proof, we can set $\delta=\sqrt{\frac{SC^\star}{c_3(1-\gamma)(C^\star-1)N}}$
when $\sqrt{\frac{SC^\star}{c_3(1-\gamma)(C^\star-1)N}}<\frac{1}{4}$, and $\delta=\frac{1}{4}$
otherwise. This yields the desired lower bound~\eqref{eq:MDP-lower-bound-C-1-2}.

\paragraph{Putting the pieces together.}

Now we are in position to summarize and simplify the three established
lower bounds~\eqref{eq:hard_MDP_lower_C_1}, \eqref{eq:MDP-lower-bound-C-2},
and \eqref{eq:MDP-lower-bound-C-1-2}. 

When $C^{\star}=1$, the claim in Theorem~\ref{thm:MDP_lower} is identical to the
bound~\eqref{eq:hard_MDP_lower_C_1}. 

When $C^{\star}\geq2$, we have from the bound~{\eqref{eq:MDP-lower-bound-C-2}
that }
\[
\inf_{\hat{\pi}}\sup_{P} \mathbb{E}[\mathcal{L}(\hat \pi;P)]\gtrsim\min\left(\frac{1}{1-\gamma}, \sqrt{\frac{SC^\star}{(1-\gamma)^3 N}}\right)\asymp\min\left(\frac{1}{1-\gamma}, \sqrt{\frac{S(C^\star-1)}{(1-\gamma)^3 N}}\right).
\]
Further notice that 
\[
\sqrt{\frac{S(C^{\star}-1)}{(1-\gamma)^3N}}\geq\sqrt{\frac{S}{(1-\gamma)^4N}}\geq\min\left(\frac{1}{1-\gamma},\frac{S}{(1-\gamma)^2N}\right).
\]
The claimed lower bound in Theorem~\ref{thm:MDP_lower}
arises.

In the end, when $C^{\star}\in(1,2)$, we know from the bounds~\eqref{eq:hard_MDP_lower_C_1}
and \eqref{eq:MDP-lower-bound-C-1-2} that
\begin{align*}
\inf_{\hat{\pi}}\sup_{P}\mathbb{E}[\mathcal{L}(\hat \pi;P)]&\gtrsim\max\left\{ \min\left(\frac{1}{1-\gamma},\frac{S}{(1-\gamma)^2N}\right),\min\left(\frac{C^{\star}-1}{1-\gamma},\sqrt{\frac{S(C^{\star}-1)}{(1-\gamma)^3N}}\right)\right\} \\ &\asymp\min\left(\frac{1}{1-\gamma}, \frac{S}{(1-\gamma)^2N} + \sqrt{\frac{S(C^\star-1)}{(1-\gamma)^3N}}\right),
\end{align*}
which completes the proof.

\subsubsection{Proof of Lemma~\ref{lem:MDP_fano}}

By definition, one has
\begin{align*}
\mathcal{L}(\pi;P_{1})+\mathcal{L}(\pi;P_{2}) & =J_{P_{1}}(\pi^{\star})-J_{P_{1}}(\pi)+J_{P_{2}}(\pi^{\star})-J_{P_{2}}(\pi)\\
 & =\sum_{j=1}^{S/4}\rho_{0}(s_{0}^{j})\left(V_{P_{1}}^{\star}(s_{0}^{j})-V_{P_{1}}^{\pi}(s_{0}^{j})+V_{P_{2}}^{\star}(s_{0}^{j})-V_{P_{2}}^{\pi}(s_{0}^{j})\right),
\end{align*}
where we have ignored the state $s_{-1}$ since it has zero rewards. 
Our proof consists of three steps. We first connect the value difference
$V_{P_{1}}^{\star}(s_{0}^{j})-V_{P_{1}}^{\pi}(s_{0}^{j})$ at $s_{0}^{j}$
to that $V_{P_{1}}^{\star}(s_{1}^{j})-V_{P_{1}}^{\pi}(s_{1}^{j})$
at $s_{1}^{j}$. Then, we further link the value difference at $s_{1}^{j}$
to the difference in transition probabilities, i.e., $\delta$ in
our design. In the end, we use the property of the set $\mathcal{V}$
to conclude the lower bound. 

\paragraph{Step 1. }

Since at state $s_0^j$, we only have one action $a_1$ with $r(s_0^j, a_1) = 0$, from the definition of value function one has
\[
V^\pi_{P_1}(s_0^j) = \sum_{i=0}^{\infty}\gamma^{i+1}(1-p)p^{i} V_{P_{1}}^{\pi}(s_{1}^{j}),
\]
for any policy $\pi$. 
Thus we have
\[
V_{P_{1}}^{\star}(s_{0}^{j})-V_{P_{1}}^{\pi}(s_{0}^{j})=\sum_{i=0}^{\infty}\gamma^{i+1}(1-p)p^{i}\left(V_{P_{1}}^{\star}(s_{1}^{j})-V_{P_{1}}^{\pi}(s_{1}^{j})\right)>\frac{1}{4}\left(V_{P_{1}}^{\star}(s_{1}^{j})-V_{P_{1}}^{\pi}(s_{1}^{j})\right),
\]
where we have used the fact that (assuming $\gamma\geq1/2$)
\[
\sum_{i=0}^{\infty}\gamma^{i+1}(1-p)p^{i}=\frac{1}{2}\gamma\geq\frac{1}{4}.
\]
 The same conclusion holds for $P_{2}$. Therefore we can obtain the
following lower bound
\[
\mathcal{L}(\pi;P_{1})+\mathcal{L}(\pi;P_{2})\geq\frac{1}{S}\sum_{j=1}^{S/4}\left(V_{P_{1}}^{\star}(s_{1}^{j})-V_{P_{1}}^{\pi}(s_{1}^{j})+V_{P_{2}}^{\star}(s_{1}^{j})-V_{P_{2}}^{\pi}(s_{1}^{j})\right).
\]

\paragraph{Step 2.}

Without loss of generality, we assume that under $P_{1}$, $\mathbb{P}(s_{\oplus}^{j}\mid s_{1}^{j},a_{2})=\frac{1}{2}+\delta$,
i.e., $v_{j}=+1$. Clearly, in this case, $a_{2}$ is the optimal
action at $s_{1}^{j}$. If the policy $\pi$ chooses the sub-optimal
action (i.e., $a_{1}$) at $s_{1}^{j}$, then we have 
\begin{align*}
V_{P_{1}}^{\star}(s_{1}^{j})-V_{P_{1}}^{\pi}(s_{1}^{j}) & =\gamma\left(\left(\frac{1}{2}+\delta\right)V_{P_{1}}^{\star}\left(s_{\oplus}^{j}\right)+\left(\frac{1}{2}-\delta\right)V_{P_{1}}^{\star}\left(s_{\ominus}^{j}\right)-\frac{1}{2}V_{P_{1}}^{\pi}\left(s_{\oplus}^{j}\right)-\frac{1}{2}V_{P_{1}}^{\pi}\left(s_{\ominus}^{j}\right)\right)\\
 & \geq\gamma\delta\left(V_{P_{1}}^{\star}\left(s_{\oplus}^{j}\right)-V_{P_{1}}^{\star}\left(s_{\ominus}^{j}\right)\right)\\
 & \geq\gamma\delta\sum_{i=0}^{\infty}\gamma^{i}q^{i}= \frac{\gamma\delta}{1-\gamma q}=\frac{\gamma\delta}{2(1-\gamma)}.
\end{align*}
On the other hand, if $\pi(s_{1}^{j})$ is not the optimal action ($a_{1}$
in this case), we have the trivial lower bound $V_{P_{1}}^{\star}(s_{1}^{j})-V_{P_{1}}^{\pi}(s_{1}^{j})\geq0$.
As a result, we obtain
\[
V_{P_{1}}^{\star}(s_{1}^{j})-V_{P_{1}}^{\pi}(s_{1}^{j})\geq\frac{\gamma\delta}{2(1-\gamma)}1\left\{ \pi(s_{1}^{j})\neq\pi_{P_{1}}^{\star}(s_{1}^{j})\right\} ,
\]
which implies
\begin{align*}
\mathcal{L}(\pi;P_{1})+\mathcal{L}(\pi;P_{2}) & \geq\frac{1}{S}\cdot\frac{\gamma\delta}{2(1-\gamma)}\sum_{j=1}^{S/4}\left(1\left\{ \pi(s_{1}^{j})\neq\pi_{P_{1}}^{\star}(s_{1}^{j})\right\} +1\left\{ \pi(s_{1}^{j})\neq\pi_{P_{2}}^{\star}(s_{1}^{j})\right\} \right)\\
 & \geq\frac{1}{S}\cdot\frac{\gamma\delta}{2(1-\gamma)}\sum_{j=1}^{S/4}1\left\{ \pi_{P_{1}}^{\star}(s_{1}^{j})\neq\pi_{P_{2}}^{\star}(s_{1}^{j})\right\} .
\end{align*}

\paragraph{Step 3.}

In the end, we use the second property of the set $\mathcal{V}$,
namely for any $\bm{v}_{i}\neq\bm{v}_{j}$ in $\mathcal{V}$, one
has $\|\bm{v}_{i}-\bm{v}_{j}\|_{1}\geq S/8$. An immediate
consequence is that 
\[
\sum_{j=1}^{S/4}1\left\{ \pi_{P_{1}}^{\star}(s_{1}^{j})\neq\pi_{P_{2}}^{\star}(s_{1}^{j})\right\} =\|\bm{v}_{P_{1}}-\bm{v}_{P_{2}}\|_{1}\geq\frac{S}{8}.
\]

Taking the previous three steps collectively completes the proof.

\subsubsection{Proof of Lemma~\ref{lem:IL_LB_MDP_lower}}\label{app:proof_lem_IL_lower}
In the case of $C^\star = 1$, we have ${d}^\star = \mu$ which is the imitation learning setting. We adapt the information-theoretic lower bound for the episodic MDPs given in the work~\citet[Theorem 6]{rajaraman2020toward} to the discounted setting.

\paragraph{Notations and Setup:}
Let $\mathcal{S}(\cD)$  be the set of all states that are observed in dataset $\cD$. When $C^\star = 1$, we  know the optimal policy $\pi^\star (s)$ at all states $s \in \mathcal{S} (\mathcal{D})$ visited in the dataset $\cD$.
We define $\Pi_{\mathrm{mimic}} (\cD)$  as the family of deterministic policies which always take the optimal action on each state visited in $\cD$, namely,
\begin{equation} \label{eq:Pi.mimic.DA}
    \Pi_{\mathrm{mimic}} (\cD) \coloneqq \Big\{
    \forall s \in \mathcal{S}  (\cD),\ \pi (s) = \pi^\star (s) \Big\},
\end{equation} 
Informally, $\Pi_{\mathrm{mimic}} (\cD)$ is the family of  policies which are ``compatible'' with the dataset  collected by the learner.

Define $\mathbb{M}_{\mathcal{S}, \mathcal{A}}$ as the family of MDPs over state space $\mathcal{S}$ and action space $\mathcal{A}$. We proceed by by lower bounding the Bayes expected suboptimality. That is, we aim at finding a distribution $\mathcal{P}$ over MDPs supported on $\mathbb{M}_{\mathcal{S} , \mathcal{A}}$ such that,
\begin{align*}
    \mathbb{E}_{\mathsf{MDP} \sim \mathcal{P}} \Big[ J(\pi^\star) - \mathbb{E}_{\cD} \left[ J (\hat{\pi})\right]\Big] \gtrsim \min \left\{ \frac{1}{1-\gamma}, \frac{S}{ (1-\gamma)^2N} \right\},
\end{align*}
where $\hat{\pi}$ is a function of dataset $\cD$.

\paragraph{Construction of the distribution $\mathcal{P}$:} 
We first determine the distribution of the optimal policy, and then we design $\mathcal{P}$ such that conditioned on the optimal policy, the distribution is deterministic. We  let the distribution of the optimal policy be uniform over all deterministic policies. That is, for each  $s \in \mathcal{S}$, $\pi^\star (s) \sim \mathrm{Unif} (\mathcal{A})$.  For every $\pi^\star$, we construct an MDP instance in  in Figure~\ref{fig:det:Nsim0:LB:repeat}. Hence the distribution over MDPs comes from the randomness in $\pi$.

For a fixed optimal policy $\pi^\star$, the MDP instance  $\mathsf{MDP} [\pi^\star]$ is determined as follows: we initialize with  a fixed initial distribution over states $\rho = \{ \zeta,\cdots,\zeta, 1 {-}(S{-}2) \zeta, 0\}$ where $\zeta = \frac{1}{N+1}$. Let the last state be a special state $b$ which we refer to as the ``bad state''. At each state $s \in \mathcal{S} \setminus \{ b \}$, choosing the optimal action renews the state in the initial distribution $\rho$ and gives a reward of $1$, while any other choice of action deterministically induces a transition to the bad state $b$ and offers zero reward. In addition, the bad state is absorbing and dispenses no reward regardless of the choice of action. That is,
\begin{equation}
    P (\cdot\mid s,a) =
    \begin{cases}
    \rho, \qquad &s \in \mathcal{S} \setminus \{ b \},\ a = \pi^\star (s) \\
    \delta_b, &\text{otherwise},
    \end{cases}
\end{equation}
and the reward function of the MDP is given by
\begin{equation}
    r (s,a) = \begin{cases} 1, \qquad & s \in \mathcal{S} \setminus \{ b \},\ a = \pi^\star (s),\\
    0, & \text{otherwise}.
    \end{cases}
\end{equation}

\begin{figure}[t]
\centering
    \scalebox{1}{
    \begin{tikzpicture}[observed/.style={circle, draw=black, fill=black!10, thick, minimum size=10mm},
    good/.style={circle, fill = citeColor!20, draw=gray, thick, minimum size=11mm},
    bad/.style={circle, fill = linkColor!20, draw=gray, thick, minimum size=11mm},
    squarednode/.style={rectangle, draw=red!60, fill=red!5, very thick, minimum size=10mm},
    treenode/.style={rectangle, draw=none, thick, minimum size=10mm},
    rootnode/.style={rectangle, draw=none, thick, minimum size=10mm},
    squarednode/.style={rectangle, draw=none, fill=citeColor!20, very thick, minimum size=7mm},]
    \node[good] (s1) at (0,0) {$1$};
    \node[good] (s2) at (2,0) {$2$};
    \node[good] (s3) at (4.5,0) {\footnotesize${S -1}$};
    \node (s5) at (3.25,0) {$\dots$};
    \node[bad] (s4) at (7,0) {$b$};
    \node (n1) at (0,1.8) {$\sim \rho$};
    \node (n2) at (2,1.8) {$\sim \rho$};
    \node (n3) at (4.5,1.8) {$\sim \rho$};
    \node (t1) at (0.5,0.9) {\footnotesize $\pi^\star(1)$};
    \node (t2) at (2.5,0.9) {\footnotesize $\pi^\star(2)$};
    \node (t3) at (5.3,0.9) {\footnotesize $\pi^\star(S-1)$};
    \draw[->, very thick, >=stealth, linkColor!80] (s1.south east) to [out=-50,in=220,looseness=0.7] (s4.south west);
    \draw[->, very thick, >=stealth, linkColor!80] (s2.south east) to [out=-40,in=210,looseness=0.7] (s4.south west);
    \draw[->, very thick, >=stealth, linkColor!80] (s3.east) to (s4.west);
    \draw[->, very thick, >=stealth, linkColor!80] (s4.north) to [out=70,in=20,looseness=4] (s4.east);
    \draw[->, very thick, >=stealth, citeColor!80] (s1.north) to (n1.south);
    \draw[->, very thick, >=stealth, citeColor!80] (s2.north) to (n2.south);
    \draw[->, very thick, >=stealth, citeColor!80] (s3.north) to (n3.south);
    \end{tikzpicture}}
\caption{The hard MDP instance for the case $C^\star = 1$. Upon playing the optimal (blue) action at any state except $b$, the learner returns to a new state according to initial distribution $\rho = \{ \zeta,{\cdots},\zeta,1 {-} (S{-}2) \zeta, 0\}$ where $\zeta {=} \frac{1}{N+1}$. Any other choice of action (red) deterministically transitions the state to $b$.}
\label{fig:det:Nsim0:LB:repeat}
\end{figure}
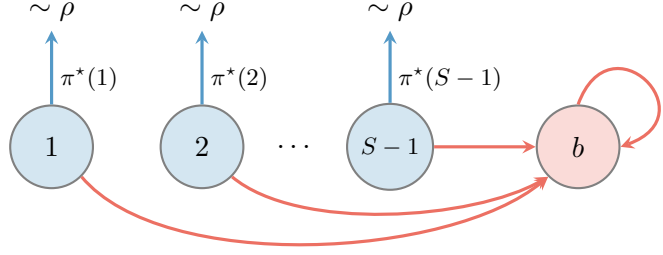

Under this construction, it is easy to see that $J_{\mathsf{MDP}}(\pi^\star(\mathsf{MDP})) = 1/(1-\gamma)$ since the optimal action always acquires reward $1$ throughout the trajectory. Thus the Bayes risk can be written as
\begin{equation} \label{eq:Rlowerbound-1}
    \mathbb{E}_{\mathsf{MDP} \sim \mathcal{P}} \Big[ \frac{1}{1-\gamma} - \mathbb{E} \Big[ J_{\mathsf{MDP}} (\widehat{\pi} (\cD))\Big]\Big].
\end{equation}

\paragraph{Understanding the conditional distribution.}
Now we study the conditional distribution of the MDP given the observed dataset $\mathcal{D}$. We start from the conditional distribution of the optimal policy.
We present the following lemma without proof.
\begin{lemma}[{\citet[Lemma A.14]{rajaraman2020toward}}] \label{lemma:cond-is-mimic}
Conditioned on the dataset $\cD$ collected by the learner, the optimal policy $\pi^\star$ is distributed $\sim \mathrm{Unif} (\Pi_{\mathrm{mimic}} (\cD))$. In other words, at each state visited in the dataset, the optimal action is fixed. At the remaining states, the optimal action is sampled uniformly from $\mathcal{A}$.
\end{lemma}

Now we define the conditional distribution of the MDPs given the dataset $\cD$ collected by the learner as below.
\begin{definition} \label{def:PDA}
Define $\mathcal{P} (\cD)$ as the distribution of $\mathsf{MDP}$ conditioned on the observed dataset $\cD$. In particular, $\pi^\star \sim \mathrm{Unif} (\Pi_{\mathrm{mimic}} (\cD))$ and $\mathsf{MDP} = \mathsf{MDP} [\pi^\star]$.
\end{definition}

\noindent From Lemma~\ref{lemma:cond-is-mimic} and the definition of $\mathcal{P} (\cD)$ in Definition~\ref{def:PDA}, applying Fubini's theorem gives
\begin{equation} \label{eq:Rlowerbound-2}
    \mathbb{E}_{\mathsf{MDP} \sim \mathcal{P}} \Big[ \frac{1}{1-\gamma} - \mathbb{E}_{\cD} \left[ J (\widehat{\pi})\right]\Big] = \mathbb{E}_\cD \left[ \mathbb{E}_{\cM \sim \mathcal{P} } \left[ \frac{1}{1-\gamma} - J (\widehat{\pi} )\right] \right].
\end{equation}

\paragraph{Lower bounding the Bayes Risk.}

Next we relate the Bayes risk to the first time the learner visits a state unobserved in $\cD$.

\begin{lemma}\label{lemma:hatvalue-UB}
In the trajectory induced by the  infinite-horizon MDP and policy, 
define the stopping time $\tau$ as the first time that the learner encounters a state $s \ne b$ that has not been visited in $\cD$ at time $t$. That is,
\begin{equation}
    \tau = \begin{cases}
    \inf \{ t : s_t \not\in \mathcal{S} (\cD) \cup \{ b \} \} \quad & \exists t : s_t \not\in \mathcal{S} (\cD) \cup \{ b \}\\
    +\infty & \text{otherwise}.
    \end{cases}
\end{equation}
Then, conditioned on the dataset $\cD$ collected by the learner,
\begin{equation}
    \mathbb{E}_{\mathsf{MDP} \sim \mathcal{P}(\cD)} \Big[ J (\pi^\star) - \mathbb{E} \left[ J (\widehat{\pi})\right]\Big] \ge \left( 1 - \frac{1}{|\mathcal{A}|} \right) \mathbb{E}_{\mathsf{MDP} \sim \mathcal{P}(\cD)}\left[ \mathbb{E}_{\widehat{\pi}(\cD)} \left[ \frac{\gamma^\tau}{1-\gamma}\right] \right]
\end{equation}
\end{lemma}

\noindent We defer the proof to the end of this section.

\medskip
Plugging the result of Lemma~\ref{lemma:hatvalue-UB} into equality \eqref{eq:Rlowerbound-2}, we obtain
\begin{align*}
    \mathbb{E}_{\mathsf{MDP} \sim \mathcal{P}} \Big[ J (\pi^\star) - \mathbb{E} \left[ J (\widehat{\pi})\right]\Big] &\ge \left( 1 - \frac{1}{|\mathcal{A}|} \right) \mathbb{E}_{\cD} \left[ \mathbb{E}_{\mathsf{MDP} \sim \mathcal{P} (\cD)} \left[ \mathbb{E}_{\widehat{\pi}(\cD)} \left[ \frac{\gamma^\tau}{1-\gamma} \right] \right] \right], \\
    &\overset{(i)}{\ge} \left( 1 - \frac{1}{|\mathcal{A}|} \right) \frac{1}{2(1-\gamma)} \mathbb{E}_{\cD} \left[ \mathbb{E}_{\mathsf{MDP} \sim \mathcal{P} (\cD)} \left[ \mathrm{Pr}_{\widehat{\pi}(\cD)} \Big[ \tau \le \lfloor \frac{1}{\log(1/\gamma)} \rfloor \Big] \right] \right], \\
    &= \left( 1 - \frac{1}{|\mathcal{A}|} \right) \frac{1}{2(1-\gamma)} \mathbb{E}_{\cM \sim \mathcal{P}} \left[ \mathbb{E}_{\cD} \left[ \mathrm{Pr}_{\widehat{\pi}(\cD)} \Big[ \tau \le \lfloor \frac{1}{\log(1/\gamma)} \rfloor \Big] \right] \right], \label{eq:Rlowerbound-3}
\end{align*}
where $(i)$ uses Markov's inequality. Lastly we bound the probability  that we visit a state unobserved in the dataset before time $\lfloor \frac{1}{\log(1/\gamma)} \rfloor$. For any policy $\widehat{\pi}$, from a similar proof as~\citet[Lemma A.16]{rajaraman2020toward} we have
\begin{equation}
    \mathbb{E}_{\mathsf{MDP} \sim \mathcal{P}} \left[ \mathbb{E}_{\cD} \left[ \mathrm{Pr}_{\widehat{\pi}} \Big[ \tau \le \lfloor \frac{1}{\log(1/\gamma)}\rfloor \Big] \right] \right] \gtrsim \min \left\{ 1, \frac{S }{\log(1/\gamma)N} \right\}.
\end{equation}  Therefore,
\begin{align*}
    \mathbb{E}_{\mathsf{MDP} \sim \mathcal{P}} \Big[ J (\pi^\star) - \mathbb{E} \left[ J (\widehat{\pi})\right]\Big] &\gtrsim \left( 1 - \frac{1}{|\mathcal{A}|} \right) \frac{1}{\log(1/\gamma)} \min \left\{ 1, \frac{S}{(1-\gamma)N} \right\} \\
    &\geq \left( 1 - \frac{1}{|\mathcal{A}|} \right) \frac{\gamma}{1-\gamma} \min \left\{ 1, \frac{S}{(1-\gamma)N} \right\}
\end{align*}
Here we use the fact that  $\log(x)\leq x-1$. Since $1 - \frac{1}{|\mathcal{A}|} \geq 1/2$ for $|\mathcal{A}| \ge 2$, the final result follows.

\paragraph{Proof of Lemma~\ref{lemma:hatvalue-UB}.}
To facilitate the analysis, we define an auxiliary random variable $\tau_b$ to be the first time the learner encounters the state $b$. If no such state is encountered, $\tau_b$ is defined as $+\infty$. Formally,
\begin{equation*}
    \tau_b = \begin{cases} \inf \{ t : s_t = b \}, &\exists t : s_t = b, \\
   +\infty, &\text{otherwise}.
    \end{cases}
\end{equation*}
Conditioned on the observed dataset $\cD$, we have 
\begin{align}
    \frac{1}{1-\gamma} - \mathbb{E}_{\cM \sim \mathcal{P} (\cD)} \left[  J (\widehat{\pi})\right] &= \frac{1}{1-\gamma} - \mathbb{E}_{\cM \sim \mathcal{P} (\cD)} \left[ \mathbb{E}_{\widehat{\pi}} \left[ \sum\nolimits_{t=0}^\infty \gamma^t {r} (s_t,a_t) \right]\right] \\
    &\ge \mathbb{E}_{\cM \sim \mathcal{P} (\cD)} \left[ \mathbb{E}_{\widehat{\pi}} \left[ \frac{\gamma^{\tau_b-1}}{1-\gamma} \right]\right] \label{eq:regretbound}
\end{align}
where the last inequality follows from the fact that ${r}$ is bounded in $[0,1]$, and the state $b$ is absorbing and always offers $0$ reward. Fixing the dataset $\cD$ and the optimal policy $\pi^\star$ (which determines the MDP $\mathsf{MDP} [\pi^\star]$), we study $\mathbb{E}_{\widehat{\pi} (\cD)} \left[ \frac{\gamma^{\tau_b-1}}{1-\gamma}\right]$ and try to relate it to $\mathbb{E}_{\widehat{\pi} (\cD)} \left[ \frac{\gamma^{\tau}}{1-\gamma}\right]$. Note that for any $t $ and state $s \in \mathcal{S}$,
\begin{align*}
    \mathrm{Pr}_{\widehat{\pi}} \left[ \tau_b = t + 1, \tau = t, s_t = s\right] &= \mathrm{Pr}_{\widehat{\pi}} \left[ \tau_b = t + 1 \mid \tau = t, s_t = s \right] \mathrm{Pr}_{\widehat{\pi}} \left[ \tau = t, s_t = s \right]\\
    &= \Big( 1 - \ind\{\widehat{\pi}(s)= \pi^\star(s)\}
    \Big) \mathrm{Pr}_{\widehat{\pi}} \left[ \tau = t, s_t = s \right].
\end{align*}
In the last equation, we use the fact that the learner must play an action other than $\pi^\star (s_t)$ to visit $b$ at time $t+1$. Next we take an expectation with respect to the randomness of $\pi^\star$ which conditioned on $\cD$ is drawn from $\mathrm{Unif} (\Pi_{\mathrm{mimic}} (\cD))$. Note that $\mathsf{MDP}[\pi^\star]$ is also determined conditioning on $\pi^\star$. 
Observe that the dependence of the second term $\mathrm{Pr}_{\widehat{\pi}} \left[ \tau = t, s_t = s \right]$ on $\pi^\star$ comes from the probability computed with the underlying MDP chosen as $\mathsf{MDP} [\pi^\star]$. However it only depends on the characteristics of $\mathsf{MDP} [\pi^\star]$ on the observed states in $\mathcal{D}$. On the other hand, the first term $\left( 1 - \ind\{\widehat{\pi}(s)= \pi^\star(s)\} \right)$ depends only on $\pi^\star(s)$, where $s$ is an unobserved state. Thus the two terms are independent. By taking expectation with respect to the randomness of $\pi^\star \sim \mathrm{Unif} (\Pi_{\mathrm{mimic}} (\cD))$ and $\mathsf{MDP} = \mathsf{MDP} [\pi^\star]$, we have
\begin{align*}
    &\mathbb{E}_{\cM \sim \mathcal{P} (\cD)} \Big[ \mathrm{Pr}_{\widehat{\pi}(\cD)} \left[ \tau_b = t + 1, \tau = t, s_t = s \right] \Big] \nonumber\\
    &= \mathbb{E}_{\cM \sim \mathcal{P} (\cD)} \Big[  1 - \ind\{\widehat{\pi}(s)= \pi^\star(s)\} \Big] \ \mathbb{E}_{\cM \sim \mathcal{P} (\cD)} \Big[ \mathrm{Pr}_{\widehat{\pi}} \left[ \tau = t, s_t = s \right] \Big] \\
    &= \left( 1 - \frac{1}{|\mathcal{A}|} \right) \mathbb{E}_{\cM \sim \mathcal{P} (\cD)} \Big[ \mathrm{Pr}_{\widehat{\pi}} \left[ \tau = t, s_t = s \right] \Big]
\end{align*}
where in the last equation, we use the fact that conditioned on $\cD$ either $(i)$ $s = b$, in which case $\tau \ne t$ and both sides are $0$, or (ii) if $s \ne b$, then $\tau = t$ implies that the state $s$ visited at time $t$ must not be observed in $\cD$, so $\pi^\star (s) \sim \mathrm{Unif} (\mathcal{A})$. Using the fact that $\mathrm{Pr}_{\widehat{\pi}} \left[ \tau_b = t + 1, \tau = t, s_t = s\right] \le \mathrm{Pr}_{\widehat{\pi}} \left[ \tau_b = t + 1 , s_t = s\right]$ and summing over $s \in \mathcal{S}$ results in the inequality,
\begin{equation*}
    \mathbb{E}_{\cM \sim \mathcal{P} (\cD)} \Big[ \mathrm{Pr}_{\widehat{\pi}} \left[ \tau_b = t + 1 \right] \Big] \ge \left( 1 - \frac{1}{|\mathcal{A}|} \right) \mathbb{E}_{\cM \sim \mathcal{P} (\cD)} \Big[ \mathrm{Pr}_{\widehat{\pi}} \left[ \tau = t \right] \Big].
\end{equation*}
Multiplying both sides by $\frac{\gamma^t}{1-\gamma}$ and summing over $t=1,\cdots,\infty$,
\begin{equation*}
    \mathbb{E}_{\cM \sim \mathcal{P} (\cD)} \Big[ \mathbb{E}_{\widehat{\pi}} \left[ \frac{\gamma^{\tau_b-1}}{1-\gamma} \right] \Big] \ge \left( 1 - \frac{1}{|\mathcal{A}|} \right) \mathbb{E}_{\cM \sim \mathcal{P} (\cD)} \Big[ \mathbb{E}_{\widehat{\pi}} \left[ \frac{\gamma^{\tau}}{1-\gamma} \right] \Big].
\end{equation*}
here we use the fact that the initial distribution $\rho$ places no mass on the bad state $b$. Therefore, $\mathrm{Pr}_{\widehat{\pi}(D)} \left[ \tau_b = 1 \right] = \rho (b) = 0$. This equation in conjunction with \eqref{eq:regretbound} completes the proof.

\subsection{Imitation learning in discounted MDPs}

In Theorem~\ref{thm:CB_IL_lower}, we have shown that imitation learning has a worse rate than LCB even in the contextual bandit case when $C^\star \in(1, 2)$. 
In this section, we show that if we change the concentrability assumption from density ratio to conditional density ratio, behavior cloning continues to work in certain regime. This also shows that behavior cloning works when $C^\star = 1$ in the discounted MDP case.

\begin{theorem}
Assume the expert policy $\pi^\star$ is deterministic and that $\max \frac{(1-\gamma) d^{*}(a|s)}{\mu(a|s)} \leq C^\star$ for some $C^\star\in[1, 2)$. 
 We consider a variant of behavior cloning policy: 
\begin{align}        \Pi_{\text{mimic}} = \{ \pi \in \Pi_{\text{det}}: \forall s \in \cD, \pi(\cdot\mid s) = \argmax_a N(s, a)\}.
    \end{align}
 Here $ \pi \in \Pi_{\text{det}}$ refers to the set of all deterministic policies.   Then for any $\hat \pi \in \Pi_{\text{mimic}} $, we have
    \begin{align*}
        \E_\cD[ J(\pi^\ast) - J(\hat{\pi})] \lesssim \frac{S}{C_0 N (1-\gamma)^2},
    \end{align*}
where    $C_0 = 1-\exp\left(-\mathsf{KL}\left(\frac{1}{2} \| \frac{1}{C^\star}\right)\right)$.
\end{theorem}
 
\begin{proof}
    Define the following population loss:
    \begin{align}
        \cL(\hat{\pi}, \pi^\ast) =\E_\cD[\E_{s \sim d_{\star}}[ 1 \{ \hat\pi(s) \neq \pi^\ast(s) \}]].
    \end{align}
 From Lemma~\ref{lemma:reduction_imitation}, we know that  it suffices to control the population loss $\mathcal{L}(\hat \pi, \pi^\star)$. From a similar argument as in~\citet{rajaraman2020toward}, we know that when $C^\star=1$, the expected suboptimality of $\hat \pi$ is upper bounded by $\min(\frac{1}{1-\gamma}, \frac{S}{(1-\gamma)^2N})$.  

When $C^\star\in(1, 2)$, the contribution to the indicator loss can be decomposed into two parts: (1) the loss incurred due to the states not included in $\cD$ whose expected value is upper bounded by $S / N$; (2) the loss incurred due to states the states for which the optimal action is not the most frequent in $\cD$. Conditioned on $N(s)$ and from $\mu( \pi^\star(s)|s)\geq d^{\star}( \pi^\star(s)|s) / C^\star = 1/C^\star$ the probability of not picking the optimal action is upper bounded by $\exp(-N(s) \cdot \mathsf{KL}\left(\mathrm{Bern}\left(\frac{1}{2}\right) \| \mathrm{Bern}\left(\frac{1}{C^\star}\right)\right))$ using Chernoff's inequality. We have 
\begin{align}
    & \E[\mathcal{L}(\hat \pi, \pi^\star)] \\
    & \quad = \E_{s \sim d_{\star}, \mathcal{D}}[  1 \{ \hat\pi(s) \neq \pi^\star(s) \}] \nonumber \\ 
    & \quad \leq  \E_{s\sim d_{\star}, \mathcal{D}}[\prob(N(s)=0)]+ \E_{s\sim d_{\star}}\E_{ \mathcal{D}}[\prob(\hat \pi(s) \neq \pi^\star(s)) \mid N(s)\geq 1]  \nonumber \\ 
    & \quad \lesssim  \frac{S}{N} + \E_{s\sim d_{\star}}\E_ {\mathcal{D}}\left[\exp\left(-N(s)\cdot \mathsf{KL}\left(\mathrm{Bern}\left(\frac{1}{2}\right) \| \mathrm{Bern}\left(\frac{1}{C^\star}\right)\right)\right) \mid   N(s)\geq 1)\right] \nonumber \\ 
    & \quad \lesssim \frac{S}{N} + \sum_{s} p(s) \sum_{n=1}^N  {N \choose n} \exp\left(-n\cdot \mathsf{KL}\left(\mathrm{Bern}\left(\frac{1}{2}\right) \| \mathrm{Bern}\left(\frac{1}{C^\star}\right)\right)  \right) p(s)^n (1-p(s))^{N-n} \nonumber \\ 
    & \quad \leq \frac{S}{N} +\sum_{s} p(s)\left(1-p(s)\left(1-\exp\left(-\mathsf{KL}\left(\mathrm{Bern}\left(\frac{1}{2}\right) \| \mathrm{Bern}\left(\frac{1}{C^\star}\right)\right)\right)\right)\right)^N.
\end{align}
Denote $C_0 = 1-\exp\left(-\mathsf{KL}\left(\mathrm{Bern}\left(\frac{1}{2}\right) \| \mathrm{Bern}\left(\frac{1}{C^\star}\right)\right)\right)$. Note that $\max_{x\in[0, 1]} x(1-C_0x)^N \leq \frac{1}{C_0(N+1)}(1-\frac{1}{N+1})^N \leq \frac{4}{9C_0N}$. Thus we have $\E[\mathcal{L}(\hat \pi, \pi^\star)]\leq \frac{4S}{9C_0N}$. We then use Lemma~\ref{lemma:reduction_imitation} to conclude that the final sub-optimality is upper bounded by $\frac{S}{C_0 N (1-\gamma)^2}$.
\end{proof}

\section{LCB in episodic Markov decision processes}\label{app:episodic_MDP}
The aim of this section is to illustrate the validity of Conjecture \ref{cnj:discounted_MDP} in episodic MDPs. In Section \ref{app:episodic_model}, we give a brief review of episodic MDPs, describing the batch dataset and offline RL objective in this setting, and introducing additional notation. We then present a variant of the VI-LCB algorithm (Algorithm \ref{alg:episodic-VI-LCB}) for episodic MDPs and state its sub-optimality guarantees in Section \ref{app:episodic_VILCB}. In Section \ref{app:properties_VI-LCB}, we show that the proposed penalty captures a confidence interval and prove a value difference lemma for Algorithm \ref{alg:episodic-VI-LCB}. Section \ref{app:episodic_proof} is devoted to the proof of the sub-optimality upper bound. In Section \ref{app:episodic_C_moderate}, we give an alternative sub-optimality decomposition as an attempt to obtain a tight dependency on $C^\star$ in the regime $C^\star \in [1,2)$. We analyze the sub-optimality in this regime in a special example provided in Section \ref{app:episodic_example}.

\subsection{Model and notation} \label{app:episodic_model}
\paragraph{Episodic MDP.}
We consider an episodic MDP described by a tuple $(\cS, \cA, \cP, \cR, \rho, H)$, where $\cS = \{\cS_h\}_{h=1}^H$ is the state space, $\cA$ is the action space, $\cP = \{P_h\}_{h=1}^H$ is the set of transition kernels with $P_h: \cS_h \times \cA \mapsto \Delta(\cS_{h+1})$, $\cR = \{R_h\}_{h=1}^H$ is the set of reward distributions $R_h: \cS_h \times \cA \rightarrow \Delta([0,1])$ with $r: \cS \times \cA \mapsto [0,1]$ as the expected reward function, $\rho: \cS_1 \rightarrow \Delta(\cS_1)$ is the initial distribution, and $H$ is the horizon. To streamline our analysis, we assume that $\{\cS_h\}_{h=1}^H$ partition the state space $\cS$ and are disjoint.

\paragraph{Policy and value functions.} Similar to the discounted case, we consider deterministic policies $\pi: \cS \mapsto \cA$ that map each state to an action. For any $h \in \{1, \dots, H\}$, $s \in \cS_h$, and $a \in \cA_h$, the value function $V^\pi_h: \cS \mapsto \mathbb{R}$ and Q-function $Q^\pi_h: \cS \times \cA \mapsto \mathbb{R}$ are respectively defined as
\begin{align*}
    V_h^\pi(s) &\coloneqq \E \left[\sum_{i = h}^H r_i \middle| s_h = s, a_i = \pi(s_i) \text{ for } i \geq h\right],\\ 
    Q_h^\pi(s,a) & \coloneqq \E \left[\sum_{i = h}^H r_i \middle| s_h = s, a_h = a, a_i = \pi(s_i) \text{ for } i \geq h+1\right].
\end{align*}
Since we assume that the set of state in different levels are disjoint, we drop the subscript $h$ when it is it clear from the context. The expected value of a policy $\pi$ is defined analogously to the discounted case:
\begin{align*}
    J(\pi) \coloneqq \E_{s \sim \rho}[V^\pi_1(s)].
\end{align*}
It is well-known that a deterministic policy $\pi^\star$ exists that maximizes the value function from any state.

\paragraph{Episodic occupancy measures.} We define the (normalized) state occupancy measure $d_\pi: \cS \mapsto [0,H]$ and state-action occupancy measure $d^\pi: \cS \times \cA \mapsto [0,H]$ as 
\begin{align}\label{def:episodic_occupancy}
    d_\pi(s) \coloneqq \frac{1}{H}\sum_{h=1}^H \prob_h(s_h = s; \pi), \quad \text{and} \quad d^\pi(s,a) \coloneqq  \frac{1}{H} \sum_{h=1}^H \prob_h(s_h = s, a_h = a; \pi),
\end{align}
where we overload notation and write $\prob_h(s_h = s; \pi)$ to denote the probability of visiting state $s_h = s$ (and similarly $s_h = s, a_h = a$) at level $h$ after executing policy $\pi$ and starting from $s_1 \sim \rho(\cdot)$.

\paragraph{Batch dataset.} The batch dataset $\cD$ consists of tuples $(s, a, r, s')$, where $r = r(s,a)$ and $s' \sim P(\cdot\mid s,a)$. As in the discounted case, we assume that $(s,a)$ pairs are generated i.i.d.~according to a data distribution $\mu$, unknown to the agent. We denote by $N(s,a) \geq 0$ the number of times a pair $(s,a)$ is observed in $\cD$ and by $N = |\cD|$ the total number of samples.

\paragraph{The learning objective.} Fix a deterministic policy $\pi$. The expected sub-optimality of policy $\hat{\pi}$ computed based on dataset $\cD$ competing with policy $\pi$ is defined as
\begin{align}
    \E_\cD \left[J(\pi) - J(\hat{\pi})\right].
\end{align}

\paragraph{Assumption on dataset coverage.} Equipped with the definitions for occupancy densities in episodic MDPs, we define the concentrability coefficient in the episodic case analogously: given a deterministic policy $\pi$, $C^\pi$ is the smallest constant satisfying
\begin{align}
    \frac{d^\pi(s,a)}{\mu(s,a)} \leq C^\pi \qquad \forall s \in \cS, a \in \cA.
\end{align}

\paragraph{Matrix notation.} We adopt a matrix notation similar to the one described in Section \ref{sec:LCB_MDP}.

\paragraph{Bellman equations.} Given any value function $V: \cS_{h+1} \mapsto \R$, the Bellman value operator at each level $h \in \{1, \dots, H\}$
\begin{align}\label{def:bellman_value_operator_episodic}
    \cT_h V = r_h + P_h V.
\end{align}
We write $(\cT_h V)(s,a) = r_h(s,a) + (P_h V)(s,a)$ for $\cS \in \cS_h, a \in \cA$.

\subsection{Episodic value iteration with LCB}\label{app:episodic_VILCB}
Algorithm \ref{alg:episodic-VI-LCB} presents a pseudocode for value iteration with LCB in the episodic setting. As in the classic value iteration in episodic MDPs, this algorithm computes values and policy through a backward recursion starting at $h = H$ with the distinction of subtracting penalties when computing the Q-function. This algorithm can be viewed as an instance of Algorithm \ref{alg:episodic-VI-LCB} of \citet{jin2020pessimism}. 

\begin{algorithm}[t]
\caption{Episodic value iteration with LCB}
\begin{algorithmic}[1]
\State \textbf{Inputs:} Batch dataset $\cD$.
\State $\hat{V}_{H+1} \mapsfrom 0$.
\For{$h = H-1, \dots, 1$}
\For{$s \in \cS_h, a \in \cA$}
\If{$N(s,a) = 0$}
\State Set $r(s,a) = 0$.
\State Set the empirical transition vector $\hat{P}_{s,a}$ randomly.
\State Set the penalty $b(s,a) = H\sqrt{L}$.
\Else 
\State Set $r(s,a)$ according to dataset.
\State Compute the empirical transition vector $\hat{P}_{s,a}$ according to dataset.
\State Set the penalty $b(s,a) = H \sqrt{ L/N(s,a)}$, where $L = 2000 \log(2S|\cA|/\delta)$.
\EndIf
\State Compute $\hat{Q}_h(s,a) \mapsfrom r(s,a) - b(s,a) + \hat{P}_{s,a} \cdot \hat{V}_{h+1}$.
\State Compute $\hat{V}_h(s) \mapsfrom \max_a \hat{Q}_h(s,a)$ and $\hat{\pi}(s) \in \argmax_a \hat{Q}_h(s,a)$.
\EndFor
\EndFor
\State \textbf{Return:} $\hat{\pi}$.
\end{algorithmic}
\label{alg:episodic-VI-LCB}
\end{algorithm}

In the following theorem, we provide an upper bound on the expected sub-optimality of the policy returned by Algorithm \ref{alg:episodic-VI-LCB}. The proof is presented in Appendix \ref{app:episodic_proof}.

\begin{theorem}[LCB sub-optimality, episodic MDP]\label{thm:MDP_episodic_upper} Consider an episodic MDP and assume that 
\begin{align*}
    \frac{d^\pi(s,a)}{\mu(s,a)} \leq C^\pi \qquad \forall s \in \cS, a \in \cA
\end{align*}
holds for an arbitrary deterministic policy $\pi$. Set $\delta = 1/N$ in Algorithm \ref{alg:episodic-VI-LCB}. Then, for all $C^\pi \geq 1$, one has
\begin{align*}
    \E_\cD [J(\pi) - J(\hat \pi)] \lesssim \min \left\{H, \widetilde{O}\left(H^2\sqrt{\frac{SC^\pi}{N}}\right) \right\}.
\end{align*}
In addition, if $1 \leq C^\pi \leq 1+L/(200N)$, then we have a tighter performance guarantee
\begin{align*}
    \E_\cD [J(\pi) - J(\hat \pi)] \lesssim \min \left\{H, \widetilde{O}\left(H^2 \frac{S}{N}\right) \right\}.
\end{align*}
\end{theorem}

We make the following conjecture that the sub-optimality rate smoothly transitions from $1/N$ to $1/\sqrt{N}$ as $C^\pi$ increases from 1 to 2.
\begin{conjecture}\label{cnj:MDP_episodic} Assume as in Theorem \ref{thm:MDP_episodic_upper}. If $1 \leq C^\pi \leq 2$, then policy $\hat \pi$ returned by Algorithm \ref{alg:episodic-VI-LCB} obeys
\begin{align*}
    \E_\cD[J(\pi) - J(\hat \pi)] \lesssim \min \left\{H, \widetilde{O}\left(H^2 \sqrt{\frac{S(C^\pi-1)}{N}}\right) \right\}.
\end{align*}
\end{conjecture}
We present our attempt in proving the above conjecture in part in Appendix \ref{app:episodic_C_moderate} followed by an example in Appendix \ref{app:episodic_example}.

\subsection{Properties of Algorithm \ref{alg:episodic-VI-LCB}}\label{app:properties_VI-LCB}
In this section, we prove two properties of Algorithm \ref{alg:episodic-VI-LCB}. We first prove that the penalty captures the Q-function lower confidence bound. Then, we prove a value difference lemma.

\paragraph{Clean event in episodic MDPs.} Define the following clean event
\begin{align}\label{def:clean_event_EMDP}
    \cE_{\text{EMDP}} \coloneqq \left\{ \forall h, \forall s \in \cS_h, \forall a: \; \; \middle| r(s,a) + P_{s,a} \cdot \hat{V}_{h+1} - \hat{r}(s,a) -  \hat{P}_{s,a} \cdot \hat{V}_{h+1} \middle| \leq b_h(s,a)  \right\},
\end{align}
where $\hat{V}_{H+1} = 0$. In the following lemma, we show that the penalty used in Algorithm \ref{alg:episodic-VI-LCB} captures the confidence interval of the empirical expectation of the Q-function.
\begin{lemma}[Clean event probability, episodic MDP]\label{lemma:clean_event_episodic_MDP} One has $\prob(\cE_{\text{EMDP}}) \geq 1 - \delta.$
\end{lemma}
\begin{proof}
The proof is analogous to the proof of Lemma \ref{lemma:penalty_LCB}. Fix a tuple $(s,a,h)$. If $N(s,a) = 0$, it is immediate that 
\begin{align*}
    | r(s,a) + P_{s,a} \cdot \hat{V}_{h+1} - \hat{r}(s,a) -  \hat{P}_{s,a} \cdot \hat{V}_{h+1} | \leq H\sqrt{L}.
\end{align*}
When $N(s,a) \geq 1$, we exploit the independence of $\hat{V}_{h+1}$ and $\hat{P}_{s,a}$ (thanks to the disjoint state space at each step $h$) and conclude by Hoeffding's inequality that for any $\delta_1 \in (0,1)$
\begin{align*}
    \prob \left(| r(s,a) + P_{s,a} \cdot \hat{V}_{h+1} - \hat{r}(s,a) +  P_{s,a} \cdot \hat{V}_{h+1} | \geq H \sqrt{\frac{2\log(2/\delta_1)}{N(s,a)}}\right) \leq \delta_1.
\end{align*}
The claim follows by taking a union bound over $s \in \cS_h, a \in \cA, h \in [H]$ and setting $\delta_1 = \delta/(S|\cA|)$.
\end{proof}

\paragraph{Value difference lemma.}
The following lemma bounds the sub-optimality of Algorithm \ref{alg:episodic-VI-LCB} by expected bonus. This result is similar to Theorem 4.2 in~\citet{jin2020pessimism}. We present the proof for completeness.
\begin{lemma}[Value difference for Algorithm \ref{alg:episodic-VI-LCB}]\label{lemma:value_difference_episodic} Let $\pi$ be an arbitrary policy. On the event $\cE_{\text{EMDP}}$, the policy $\hat{\pi}$ returned by Algorithm \ref{alg:episodic-VI-LCB} satisfies
\begin{align*}
    J(\pi) - J(\hat{\pi}) \leq 2H \E_{d^\pi}\left[b(s,a)\right].
\end{align*}
\end{lemma}
\begin{proof}
Define the following self-consistency error
\begin{align*}
    \iota_h(s,a) = \cT_h \hat{V}_{h+1} (s,a) - \hat{Q}_h(s,a),
\end{align*}
where $\cT_h$ is the Bellman value operator defined in \eqref{def:bellman_value_operator_episodic}. Let $\pi'$ be an arbitrary policy. By \citet[Lemma A.1]{jin2020pessimism}, one has 
\begin{align}\label{eq:sub-opt-difference}
    \begin{split}
        \hat{V}_1(s) - V^{\pi'}_1(s) & = \sum_{h=1}^H \E [\hat{Q}_h(s_h, \hat{\pi}(s_h)) - \hat{Q}_h(s_h,\pi'(s_h)) \mid s_1 = s] \\
        & \quad - \sum_{h=1}^H  \E [ \iota_h(s_h,\pi'(s_h))  \mid s_1 = s]
    \end{split}
\end{align}

Setting $\pi' \mapsfrom \pi$ in \eqref{eq:sub-opt-difference} gives
\begin{align}\notag 
    V^\pi_1(s) - \hat{V}_1(s) & = \sum_{h=1}^H  \E [ \iota_h(s_h,\pi(s_h))  \mid s_1 = s] - \sum_{h=1}^H \E [\hat{Q}_h(s_h, \hat{\pi}(s_h)) - \hat{Q}_h(s_h,\pi(s_h)) \mid s_1 = s]\\\label{eq:sub-opt-1}
    & \quad \leq \sum_{h=1}^H  \E [ \iota_h(s_h,\pi(s_h))  \mid s_1 = s],
\end{align}
where the last line uses the fact that $\hat{\pi}(s)$ maximizes $\hat{Q}_h(s,a)$.

We apply \eqref{eq:sub-opt-difference} once more, this time setting $\pi' \mapsfrom \hat{\pi}$:
\begin{align}\notag 
    \hat{V}_1(s) - V^{\hat{\pi}}_1(s) & = \sum_{h=1}^H \E [\hat{Q}_h(s_h, \hat{\pi}(s_h)) - \hat{Q}_h(s_h,\hat{\pi}(s_h)) \mid s_1 = s] - \sum_{h=1}^H  \E [ \iota_h(s_h,\hat{\pi}(s_h))  \mid s_1 = s]\\\label{eq:sub-opt-2}
    & \quad \leq - \sum_{h=1}^H  \E [ \iota_h(s_h,\pi'(s_h))  \mid s_1 = s].
\end{align}
Adding \eqref{eq:sub-opt-1} and \eqref{eq:sub-opt-2}, we have 
\begin{align}\notag 
    V^\pi_1(s) - V^{\hat{\pi}}_1(s) & = V^\pi_1(s) - \hat{V}_1(s) + \hat{V}_1(s) - V^{\hat{\pi}}_1(s)\\\label{eq:sub-opt-3}
    & \leq \sum_{h=1}^H  \E [ \iota_h(s_h,\pi(s_h))  \mid s_1 = s] - \sum_{h=1}^H  \E [ \iota_h(s_h,\pi'(s_h))  \mid s_1 = s].
\end{align}
By \citet[Lemma 5.1]{jin2020pessimism}, conditioned on $\cE_{\text{EMDP}}$, we have
\begin{align*}
    0 \leq \iota_h(s,a) \leq 2 b_h(s,a) \quad \forall s,a,h.
\end{align*}
The proof is completed by applying the above bound in \eqref{eq:sub-opt-3} and taking an expectation with respect to $\rho$
\begin{align*}
    \E_\rho[V^\pi_1(s) - V^{\hat{\pi}}_1(s)] & \leq 2 \sum_{h=1}^H \E[b_h(s_h,\pi(s_h))]\\
    & \quad = 2 \sum_{h=1}^H P_h(s_h;\pi) b_h(s_h,\pi(s_h)) = 2H \E_{d^\pi}[b(s,a)],
\end{align*}
where the last equation hinges on the definition of occupancy measure for episodic MDPs given in \eqref{def:episodic_occupancy}.
\end{proof}

\subsection{Proof of Theorem \ref{thm:MDP_episodic_upper}}\label{app:episodic_proof}
The proof follows a similar decomposition argument as in Theorem \ref{thm:MDP_upperbound_hoffding}. Nonetheless, we present a complete proof for the reader's convenience.

We divide the proof into two parts and separately analyze the general case $C^\pi \geq 1$ and $C^\star \leq 1+ L/(200N)$ since the techniques used in the proof of these two claims are rather distinct.

\paragraph{The general case when $C^\pi \geq 1$.} We decompose the expected sub-optimality into two terms
\begin{align}\label{eq:EMDP_general_C_decomposition}
    \begin{split}
        \E_\cD \left[ \sum_s \rho(s) [V^\pi_1(s) - V^{\hat{\pi}}_1(s)] \right] & = \E_\cD \left[ \sum_s \rho(s) [V^\pi_1(s) - V^{\hat{\pi}}_1(s)] \ind \{ \cE_{\text{EMDP}}\}\right] \eqqcolon T_1\\
        & \quad + \E_\cD \left[ \sum_s \rho(s) [V^\pi_1(s) - V^{\hat{\pi}}_1(s)] \ind \{ \cE^c_{\text{EMDP}}\}\right] \eqqcolon T_2.
    \end{split}
\end{align}
The first term $T_1$ captures the sub-optimality under the clean event $\cE_{\text{EMDP}}$ whereas $T_2$ represents the sub-optimality suffered when the constructed confidence interval via the penalty function falls short of containing the empirical Q-function estimate. We will prove in subsequent sections that $T_1$ and $T_2$ are bounded according to:
\begin{subequations}
\begin{align}
T_{1} & \leq 32 H^2 \sqrt{\frac{SC^\pi L}{N}}
\label{eq:hard-term-EMDP-crude}\\
T_{2} & \leq H\delta.\label{eq:event-EMDP}
\end{align}
\end{subequations}
Taking the above bounds as given for the moment and setting $\delta = 1/N$, we conclude that 
\begin{align*}
    \E_\cD [J(\pi) - J(\hat{\pi})] \lesssim \min\left(H, 32 H^2 \sqrt{\frac{SC^\pi L}{N}}\right).
\end{align*}

\paragraph{The case when $C^\pi \leq 1+L/(200N)$.} To obtain faster rates in this regime, we resort to directly analyzing the policy sub-optimality instead of bounding the value sub-optimality (such as by Lemma \ref{lemma:value_difference_episodic}). It is useful to connect the sub-optimality of a policy to whether it disagrees with the optimal policy at each state. The following lemma due to \citet[Theorem 2.1]{ross2010efficient} provides such a connection.

\begin{lemma}\label{lemma:performance_difference_episodic} For any deterministic policies $\pi, \hat{\pi}$, one has
\begin{align*}
    J(\pi) - J(\hat{\pi}) \leq H^2 \E_{s \sim d_\pi} [\ind\{\pi(s) \neq \hat{\pi}(s)\}].
\end{align*}
\end{lemma}
We apply Lemma \ref{lemma:performance_difference_episodic} to bound the sub-optimality and further decompose it based on whether any samples are observed on each state $s$.
\begin{align*}
    & \E_\cD [ \rho(s) [V^\pi_1(s) - V^{\hat \pi}_1(s)]]\\
    & \quad \leq H^2 \E_\cD \E_{d_\pi} [\ind\{\pi(s) \neq \hat{\pi}(s)\}] \\
    & \quad = H^2 \E_\cD \E_{d_\pi} [\ind\{\pi(s) \neq \hat{\pi}(s)\} \ind \{ N(s,\pi(s)) = 0\}] \eqqcolon T_1'\\
    & \quad \quad + H^2 \E_\cD \E_{d_\pi} [\ind\{\pi(s) \neq \hat{\pi}(s)\} \ind \{ N(s,\pi(s)) \geq 1\}] \eqqcolon T_2'.
\end{align*}
In a similar manner to the proof of Theorem \ref{thm:MDP_upperbound_hoffding}, we prove the following bounds on $T_1'$ and $T_2'$: 
\begin{subequations}
\begin{alignat}{2}
    T_{1}' & \leq H^2  \frac{4C^\pi}{N};\label{eq:missing-term-EMDP-small-C}\\
    T_{2}' & \lesssim \frac{2SC^\pi H^2 L}{N} +H^2 \frac{|\cA|}{N^9}. \label{eq:hard-term-EMDP-small-C}
\end{alignat}
\end{subequations}
\subsubsection{Proof the bound \eqref{eq:hard-term-EMDP-crude} on $T_1$}
By the value difference Lemma \ref{lemma:value_difference_episodic}, one has
\begin{align*}
    \E_\cD[\sum_s \rho(s) [V^\pi(s) - V^{\hat \pi}(s)] \ind\{\cE_{\text{EMDP}} \}]  & \leq 2 H \sum_{s,a} d^\pi(s,a) \E_\cD[b(s,a)]\\
    & \leq 2 H \sum_{s,a} d^\pi(s,a) H \E_\cD \left[ \sqrt{\frac{L}{N(s,a) \vee 1}} \right]\\
    & \leq 32 H^2 \sum_{s,a} d^\pi(s,a) \left[ \sqrt{\frac{L}{N \mu(s,a)}} \right],
\end{align*}
where the last inequality uses the bound on inverse moments of binomial random variables given in \ref{lemma:binomial_inverse_moment_bound} with $c_{1/2} \leq 16$. We then apply the concentrability assumption and the Cauchy-Schwarz inequality to conclude that
\begin{align*}
    T_1 & \leq 32 H^2 \sum_{s,a} \sqrt{d^\pi(s,a)} \sqrt{HC^\pi \mu(s,a)} \left[ \sqrt{\frac{L}{N \mu(s,a)}} \right]\\
    & \leq 32H^2 \sqrt{\frac{C^\pi L H}{N}} \sum_{s} \sqrt{d^\pi(s,\pi(s))}
    \leq 32 H^2 \sqrt{\frac{SC^\pi L}{N}}.
\end{align*}

\subsubsection{Proof of the bound \eqref{eq:event-EMDP} on $T_2$}
We use a argument similar to that in the proof of \eqref{eq:event-MDP}. First, observe that $\sum_s \rho(s) [V_1^\pi(s) - V^{\hat{\pi}}(s)] \leq H$. Consequently, in light of Lemma \ref{lemma:clean_event_episodic_MDP} one can conclude 
\begin{align*}
    T_3 \leq H \E_\cD [\ind \{ \cE_{\text{EMDP}}^c\}] = H \prob(\cE_{\text{EMDP}}^c) \leq H \delta.
\end{align*}

\subsubsection{Proof of the bound \eqref{eq:missing-term-EMDP-small-C} on $T_1'$}
We have 
\begin{align*}
    T_1' \leq H^2 \E_{d_\pi} \E_{\cD}[\ind \{ N(s,\pi(s)) = 0\}] \leq H^2 \E_{d_\pi} \prob(N(s,\pi(s)) = 0).
\end{align*}
It follows from the concentrability assumption $d^\pi(s,\pi(s))/\mu(s, \pi(s)) \leq C^\pi$ that
\begin{align*}
    T_1 \leq H^2 \sum_s C^\pi \mu(s,\pi(s)) \prob(N(s,\pi(s)) = 0)
    & = H^2 C^\pi \sum_s \mu(s,\pi(s)) (1-\mu(s,\pi(s))^N.
\end{align*}
Note that $\max_{x \in [0,1]} x (1-x)^N \leq 4/(9N)$. We thus conclude that
\begin{align*}
    T_1 \leq H^2 C^\pi \sum_s \mu(s,\pi(s)) (1-\mu(s,\pi(s))^N \leq H^2  \frac{4C^\pi}{9N}.
\end{align*}

\subsubsection{Proof of the bound \eqref{eq:hard-term-EMDP-small-C} on $T_2'$}
We prove the bound on $T_2’$ by partitioning the states based on how much they are occupied under the target policy. Define the following set:
\begin{align}\label{eq:defn-S-1-EMDP}
    \mathcal{O}_{1} & \coloneqq\left\{ s\mid  d_\pi(s)<\frac{2 C^{\pi}L}{N}\right\}. 
\end{align}
We can then decompose $T_2'$ according to whether state $s$ belongs to $\cO_1$:
\begin{align*}
    T_2' & = H^2 \sum_{s \in \cO_1 } d_\pi(s) \E_\cD[\ind \{ \hat{\pi}(s) \neq \pi(s) \}\ind \{N(s,\pi(s)) \geq 1 \}]\eqqcolon T_{2,1}\\
    & \quad+ H^2 \sum_{s \not \in\cO_1}d_\pi(s)\E_\cD[\ind \{ \hat\pi(s) \neq \pi(s) \}\ind \{N(s,\pi(s)) \geq 1 \}] \eqqcolon T_{2,2}.
\end{align*}
Here, $T_{2,1}$ captures the sub-optimality due to the less important states under the target policy. We will shortly prove the following bounds on these two terms:
\begin{align*}
    T_{2,1}  \leq \frac{2SC^\pi H^2 L}{N} \qquad \text{and} \qquad
    T_{2,2}  \lesssim H^2 \frac{|\cA|}{N^9}.
\end{align*}

\paragraph{Proof of the bound on $T_{2,1}$.} Since $\E_\cD[\ind \{ \hat{\pi}(s) \neq \pi(s) \}\ind \{N(s,\pi(s)) \geq 1 \}] \leq 1$, it follows immediately that
\begin{align*}
    T_{2,1} \leq H^2 \sum_{s \in \cS_1} d_\pi(s) \leq \frac{2SC^\pi H^2 L}{N},
\end{align*}
where the last inequality relies on the definition of $\cO_1$ provided in \eqref{eq:defn-S-1-EMDP}.

\paragraph{Proof of the bound on $T_{2,2}$.} The term $T_{2,2}$ is equal to 
\[
T_{2,2}=H^2\sum_{s \not \in\mathcal{O}_{1}}d_\pi(s)\mathbb{P}\left(\hat\pi(s)\neq\pi(s),\; N(s,\pi(s))\geq1\right).
\]
We subsequently show that the probability $\mathbb{P}\left(\hat\pi(s)\neq\pi(s),\; N(s,\pi(s))\geq1\right)$ is small. Fix a state $s \not \in \cO_1$ and let $h$ be the level to which $s$ belongs. The concentrability assumption along with the constraint on $d_\pi(s)$ implies the following lower bound on $\mu(s,\pi(s))$:
\begin{align}\label{eq:EMDP-prob-lower}
    \mu(s, \pi(s)) \geq \frac{1}{C^\pi} d_\pi(s) \geq \frac{1}{C^\pi} \frac{2 C^\pi L}{N} = \frac{2 L}{N}.
\end{align}
On the other hand, by the concentrability assumption and using $C^\pi \leq 1+ \frac{L}{200N}$, the following upper bound holds for $\mu(s, a \neq \pi(s))$:
\begin{align}\label{eq:EMDP-prob-upper}
    \mu(s,a) \leq \sum_{a \neq \pi(s)}\mu(s,a) \leq 1-\frac{1}{C^\pi} \leq \frac{L}{200N},
\end{align}
The above bounds suggest that the target action is likely to be included in the dataset more frequently than the rest of the actions for $s \not \in \cO_1$. We will see shortly that in this scenario, the LCB algorithm picks the target action with high probability. 
The bounds~\eqref{eq:EMDP-prob-lower} and~\eqref{eq:EMDP-prob-upper} together with Chernoff's bound give
\begin{align*}
    \prob \left(  N(s,a \neq \pi(s)) \leq \frac{5L}{200} \right) & \geq 1 -  \exp \left(-\frac{L}{200} \right); \\
    \prob \left(  N(s,\pi(s)) \geq L \right) & \geq 1 -  \exp \left(-\frac{L}{4} \right).
\end{align*}
We can thereby write an upper bound $\hat{Q}_h(s,a \neq \pi(s))$ and a lower bound on $\hat{Q}_h(s,\pi(s))$. In particular, when $N(s,a) \leq \frac{5L}{200}$, one has
\begin{align*}
    \hat{Q}_{h}(s,a) & = r_h(s,a) - b_h(s,a) + \hat{P}_{s,a} \cdot \hat{V}_{h+1}\\ 
    & = r_h(s,a) -  H \sqrt{\frac{L}{N(s,a)\vee 1}} + \hat{P}_{s,a} \cdot \hat{V}_{h+1}\\
    & \leq 1 -  H\sqrt{\frac{L}{5L/200}} +H \leq - 4 H,
\end{align*}
where we used the fact that $L \geq 70$. When $N(s,\pi(s)) \geq L$, one has
\begin{align*}
    \hat{Q}_h(s,\pi(s)) & = r_h(s,\pi(s)) -  H\sqrt{\frac{L}{N(s,\pi(s))}} + \hat{P}_{s,\pi(s)} \cdot V_{h+1} \geq  - H.
\end{align*}
Note that if both $N(s,a \neq \pi(s)) \leq \frac{5L}{200}$ and $N(s,\pi(s)) \geq L$ hold, we must have $\hat{Q}_h(s, a \neq \pi(s))< \hat{Q}_h(s,\pi(s))$. 
Therefore, we deduce that 
\begin{align*}
 \mathbb{P}\left(\hat\pi(s)\neq\pi(s),\; N(s,\pi(s))\geq1\right)
 & \leq  (|\mathcal{A}|-1)\exp\left(-\frac{L}{200}\right)+\exp\left(-\frac{1}{4}  L\right)  
  \leq |\mathcal{A}|\exp\left(-\frac{L}{200}\right),
\end{align*}
which further implies
\begin{align*}
T_{2,2} & \leq H^2 \sum_{s\not \in\mathcal{O}_{1}}d_\pi(s) |\mathcal{A}|\exp\left(-\frac{L}{200}\right)
  \leq H^2 |\cA| \exp\left(-\frac{L}{200}\right)
  \lesssim H^2 |\cA| N^{-9}.
\end{align*}

\subsection{The case of $C^\pi \in [1,2)$}\label{app:episodic_C_moderate}
In this section, we present an attempt in obtaining tight bounds on the LCB algorithm for episodic MDPs in the regime $C^\pi \in [1,2)$. We start with a decomposition similar to the one given in \eqref{eq:EMDP_general_C_decomposition}.
\begin{align*}
    \E_\cD \left[\sum_s \rho(s) [V^\pi_1(s) - V^{\hat{\pi}}(s)] \right] & = \E_\cD \left[\sum_s \rho(s) [V^\pi_1(s) - V^{\hat{\pi}}(s)] \ind \{ \cE_{\text{EMDP}}\}\right] \eqqcolon T_1\\
    & \quad + \E_\cD \left[\sum_s \rho(s) [V^\pi_1(s) - V^{\hat{\pi}}(s)] \ind \{ \cE_{\text{EMDP}}^c\}\right] \eqqcolon T_2.
\end{align*}
An upper bound on the term $T_2$ is already proven in \eqref{eq:event-EMDP}. We follow a different route for bounding the term $T_1$. For any state $s \in \cS$, define 
\begin{align}
    \bar{\mu}(s) \coloneqq \sum_{a \neq \pi(s)} \mu(s,a)
\end{align}
to be the total mass on actions not equal to the target policy $\pi(s)$. Consider the following set:
\begin{align}\label{eq:C_moderate_set1}
    \cB & \coloneqq \{s \mid \mu(s,\pi(s)) \leq  9 \bar{\mu}(s)\}.
\end{align}
The set $\cB$ includes the states for which the expert action is drawn more frequently under the data distribution. We then decompose $T_1$ based on whether state $s$ belongs to $\cB$
\begin{align}
    T_2 & = \E_\cD \left[\sum_{s \in \cB} \rho(s) [V^\pi_1(s) - V^{\hat{\pi}}(s)] \ind \{ \cE_{\text{EMDP}}\} \right] \eqqcolon \beta_1\\ \label{eq:beta2_EMDP}
    & \quad + \E_\cD \left[\sum_{s \not \in \cB} \rho(s) [V^\pi_1(s) - V^{\hat{\pi}}(s)] \ind \{ \cE_{\text{EMDP}}\} \right] \eqqcolon \beta_2.
\end{align}
We prove the following bound on $\beta_1$:
\begin{align}\label{eq:EMDP_C_moderate_beta1}
    \beta_1 \leq 136 H^2 \sqrt{\frac{S(C^\pi-1)L}{N}}.
\end{align}
We \emph{conjecture} that $\beta_2$ is bounded similarly:
\begin{align}\label{eq:EMDP_C_moderate_beta2}
    \beta_2 \lesssim H^2 \sqrt{\frac{S(C^\pi-1)L}{N}}.
\end{align}
We demonstrate our conjecture on $\beta_2$ in a special episodic MDP case with $H=3, |\cS_h| = 2$, and $|\cA| =2$ in Appendix \ref{app:episodic_example}.

\paragraph{Proof of the bound \eqref{eq:EMDP_C_moderate_beta1} on $\beta_1$.}
By Lemma \ref{lemma:clean_event_episodic_MDP}, it follows that
\begin{align*}
    \beta_1 & = \E_\cD[\sum_{s \in \cB} \rho(s) [V^\pi(s) - V^{\hat \pi}(s)] \ind\{\cE_{\text{EMDP}} \}]\\
    & \leq 2 H \sum_{s \in \cB} d^\pi(s,\pi(s)) \E_\cD[b(s,\pi(s))]\\
    & \leq 2 H \sum_{s \in \cB} d^\pi(s,\pi(s)) H \E_\cD \left[ \sqrt{\frac{L}{N(s,\pi(s)) \vee 1}} \right]\\
    & \leq 32 H^2 \sum_{s \in \cB} d^\pi(s,\pi(s))  \left[ \sqrt{\frac{L}{N \mu(s,\pi(s))}} \right]
\end{align*}
In the first inequality, we substituted the definition of penalty and the second inequality arises from Lemma \ref{lemma:binomial_inverse_moment_bound} with $c_{1/2} \leq 16$. We then apply the concentrability assumption to bound $d^\pi(s,\pi(s)) \leq C^\pi \mu(s,\pi(s))$ and thereby conclude
\begin{align*}
    \beta_1 & \leq 32 H^2 \sum_{s \in \cB} C^\pi \mu(s,\pi(s)) \left[ \sqrt{\frac{L}{N \mu(s,\pi(s))}} \right]\\
    & = 32 C^\pi H^2 \sqrt{\frac{L}{N}} \sum_{s \in \cB} \sqrt{\mu(s,\pi(s))}\\
    & \leq 32 C^\pi H^2 \sqrt{\frac{L S}{N}} \sqrt{\sum_{s \in \cB} \mu(s,\pi(s))},
\end{align*}
where the last line is due to Cauchy-Schwarz inequality. We continue the bound relying on the definition of $\cB$
\begin{align}\label{eq:beta1_EMDP_term1}
    \beta_1 \leq 32 C^\pi H^2 \sqrt{\frac{L S}{N}} \sqrt{\sum_{s} \mu(s,\pi(s)) \ind\{ \mu(s,\pi(s)) \leq 9 \bar{\mu}(s)\} } 
    & \leq 32 C^\pi H^2 \sqrt{\frac{L S}{N}} \sqrt{\sum_{s} 9\bar{\mu}(s)}.
\end{align}
It is easy to check that the concentrability assumption implies the following bound on the total mass over the actions not equal to $\pi(s)$ 
\begin{align*}
    \sum_{s} \bar{\mu} \leq \frac{C^\pi - 1}{C^\pi}.
\end{align*}
Substituting the above bound to \eqref{eq:beta1_EMDP_term1} and bounding $C^\pi \leq 2$ yields
\begin{align*}
    \beta_1 & \leq 136 H^2 \sqrt{\frac{S(C^\pi-1)L}{N}}.
\end{align*}

\subsection{Analysis of LCB for a simple episodic MDP}\label{app:episodic_example}
We consider an episodic MDP with $H = 3$, $\cS_1 = \{1,2\}$, $\cS_2 = \{3,4\}$, $\cS_3 = \{5,6\}$, and $\cA = \{1,2\}$, where we assume without loss of generality that action 1 is optimal in all states. We are interested in bounding the $\beta_2$ term defined in \eqref{eq:beta2_EMDP} when $C^\pi \in [1,2)$:
\begin{equation}
    \beta_2 = \E_\cD \left[\sum_{s: \mu(s,\pi^\star(s)) \geq 9 \bar{\mu}(s)} \rho(s) [V^\pi_1(s) - V^{\hat{\pi}}_{1}(s)] \ind \{ \cE_{\text{EMDP}}\} \right].
\end{equation}
Note that $\beta_2$ captures sub-optimality in states for which $\mu(s,\pi(s)) > 9 \bar{\mu}(s)$. To illustrate the key ideas and avoid clutter, we consider the following setting:
\begin{enumerate}
    \item Competing with the optimal policy $\pi(s) = \pi^\star(s) = 1$ and thus the concentrability assumption $d^\star(s,a) \leq C^\star \mu(s,a)$ for all $s \in \cS, a \in \cA$;
    \item $\mu(s,1) \geq 9 \mu(s,2)$ for all $s \in \cS$;
    \item $N(s,a) = N \mu(s,a) \geq 1$ for all $s \in \cS, a \in \cA$.
    \item We assume that the rewards are deterministic and consider an implementation of Algorithm \ref{alg:episodic-VI-LCB} with deterministic rewards. In particular, at level $H$ this implementation of VI-LCB sets $\hat{Q}_H$ according to
    \begin{align*}
        \hat{Q}_H(s,a) = \begin{cases}
        0 & \quad N(s,a) = 0;\\
        r(s,a) & \quad N(s,a) \geq 1.
        \end{cases}
    \end{align*}
\end{enumerate}

\paragraph{Outline of the proof.} Let us first give an outline for the sub-optimality analysis of the episodic VI-LCB Algorithm \ref{alg:episodic-VI-LCB} in this example. We begin by showing that the concentrability assumption in conjunction with $\mu(s,1) \geq 9 \mu(s,2)$ dictates certain bounds on the penalties. Afterward, we argue that the episodic VI-LCB algorithm finds the optimal policy at levels 2 and 3 with high probability. This result allows writing the sub-optimality as an expectation over the product of the gap $g_1(s) = Q^\star_1(s,1) - Q^\star_1(s,2)$ and the probability that the agent chooses the wrong action, i.e., $\prob(\hat{\pi}(s) \neq 1)$. Consequently, if for state $s$ the gap $g_1(s)$ is small, the sub-optimality incurred by that state is also small. On the other hand, when the gap is large, we prove via Hoeffding's inequality that $\prob(\hat{\pi}(s) \neq 1)$ is negligible.

\paragraph{Bounds on penalties.} The setting introduced above dictates the following bounds on penalties
\begin{subequations}
\begin{align}\label{eq:penalty_difference}
    b_h(s,2) - b_h(s,1) & \geq \frac{1}{3} b_h(s,2) + b_h(s,1),\\\label{eq:bounds_on_b_s_1}
    3\sqrt{\frac{LC^\star}{N (\bar{d}(s,1)+C^\star - 1)}} & \leq b_h(s,1) \leq 3\sqrt{\frac{L C^\star }{N\bar{d}^\star(s,1)}},
\end{align}
\end{subequations}
whose proofs can be found at the end of this subsection.

\paragraph{VI-LCB policy in each level.} The main idea for a tight sub-optimality bound is to directly compare $\hat{Q}_h(s,1)$ to $\hat{Q}_h(s,2)$ at every level. Specifically, we first determine the conditions under which $\E [\hat{Q}_h(s,1) - \hat{Q}_h(s,2)]> 0$ and then show $\hat{Q}_h(s,1) > \hat{Q}_h(s,2)$ with high probability via a concentration argument. It turns out that these conditions depend on the value of the sub-optimality gap associated with a state defined as 
\begin{align}\label{def:gap}
    g_h(s) \coloneqq Q^\star_h(s,1) - Q^\star_h(s,2) \geq 0 \quad \forall s \in \cS, \forall h \in \{1, 2, 3\}.
\end{align}
We start the analysis at level 3 going backwards to level 1.

\begin{itemize}[leftmargin=*]

    \item \textbf{Level 3.} Since $N(s,a) \geq 1$ and the rewards are deterministic, the value function computed by VI-LCB algorithm is equal to $V^\star_3$ and action 1 is selected for both states 5 and 6, i.e.,
    \begin{align}
        \hat{V}_3 = V^\star_3.
    \end{align}
    \item \textbf{Level 2.} We first show that $\hat{Q}_2(s,1)$ is greater than $\hat{Q}_2(s,2)$ in expectation 
    \begin{align}\notag 
        \E[\hat{Q}_2(s,1) - \hat{Q}_2(s,2)]
        = & \E[r(s,1) - b_2(s,1) + \hat{P}_{s,1} \cdot V_3^\star - r(s,2) + b_2(s,2) - \hat{P}_{s,2} \cdot V_3^\star]\\ \notag 
        = & b_2(s,2) - b_2(s,1) + g_2(s) \\ \label{eq:level_2_expectation_bound}
        \geq & \frac{1}{3}b_2(s,2) + b_2(s,1) + g_2(s) \geq \frac{1}{3}b_2(s,2) \geq 0,
    \end{align}
    where we used the bound on $b_2(s,2) - b_2(s,1)$ given in \eqref{eq:penalty_difference}. By the concentration inequality in Lemma \ref{lemma:hoeffding_on_difference_empirical_average} we then show  $\hat{Q}_2(s,1) \geq \hat{Q}_2(s,2)$ with high probability:
    \begin{align}\notag 
        \prob( \hat{Q}_2(s,2) - \hat{Q}_2(s,1) \geq 0) & \leq \exp \left( -6 \frac{N(s,1)N(s,2) \E^2[\hat{Q}_2(s,1) - \hat{Q}_2(s,2)]}{N(s,1) + N(s,2)}\right) \\ \notag 
        & \leq \exp \left( -1.8 N(s,2) \left( \frac{1}{3}\right)^2 b^2_2(s,2)\right)\\ \label{eq:level_2_high_prob}
        & = \exp \left( -0.8 N(s,2)  \frac{L}{N(s,2)}\right) \lesssim \frac{1}{N^{160}},
    \end{align}
    where in the second inequality we used $N(s,2) \leq 1/9 N(s,1)$ as well as the bound given in~\eqref{eq:level_2_expectation_bound} and the last inequality holds for $c_1 \geq 1$ and $\delta = 1/N$. 
    \item \textbf{Level 1.} Define the following event
    \begin{align}\label{eq:event_Eo}
        \cE_o = \{\hat\pi(s) = 1, \; \forall s \in \cS_2\},
    \end{align}
    which refers to the event that action 1 is chosen for all states at level 2. Conditioned on $\cE_o$, the Q-function computed by VI-LCB in level 1 is given by
    \begin{align*}
        \begin{split}
            \hat{Q}_1(s,a)
        = r(s,a) - b_1(s,a) & + \hat{P}(3\mid s,a)[r(3,1) -b_2(3,1)+ \hat{P}_{3,1} V^\star_3]\\
        & + \hat{P}(4\mid s,a)[r(4,1) -b_2(4,1)+ \hat{P}_{4,1} V^\star_3].
        \end{split}
        \qquad \forall s \in \cS_1, a \in \cA.
    \end{align*}
    Taking the expectation with respect to the data randomness, one has for any $s \in \cB$ that
    \begin{align*}
        & \E[\hat{Q}_1(s,1) - \hat{Q}_1(s,2)]\\
        & \quad = [b_1(s,2) - b_1(s,1)] + [P(3|s,2) - P(3|s,1)]b_2(3,1) + [P(4|s,2) - P(4|s,1)]b_2(4,1) + g_1(s)\\
        & \quad = [b_1(s,2) - b_1(s,1)] + [P(3|s,1) - P(3|s,2)][b_2(4,1) - b_2(3,1)] + g_1(s),
    \end{align*}
    where the last equation uses $P(3\mid s,a) = 1-P(4\mid s,a)$. We continue the analysis assuming that $p \coloneqq P(3\mid s,1) - P(3\mid s,2) \geq 0$; the other case can be shown similarly. Using $p \geq 0$ and $b_2(4,1) \geq 0$ together with the penalty bound of \eqref{eq:penalty_difference}, we see that
    \begin{align*}
        \E[\hat{Q}_1(s,1) - \hat{Q}_1(s,2)] \geq \frac{1}{3}b_1(s,2)+ b_1(s,1) - p b_2(3,1) + g_1(s).
    \end{align*}
    We proceed by applying \eqref{eq:bounds_on_b_s_1} on $b_1(s,1)$ and $b_1(3,1)$
    \begin{align}\label{eq:episodic_step1}
    \begin{split}
        & \E[\hat{Q}_1(s,1) - \hat{Q}_1(s,2)] \geq \frac{1}{3}b_1(s,2) + 3 \sqrt{\frac{LC^\star}{N(d^\star(s,1)+ C^\star - 1)}} - 3p  \sqrt{\frac{LC^\star}{Nd^\star(3,1)}} + g_1(s).
    \end{split}
    \end{align}
    Note that $d^\star(s,1) = \rho(s)/3$ and $3 d^\star(3,1) = \rho(s) P(3|s,1) + \rho(2) P(3|s,2) \geq \rho(s) P(3|s,1) \geq \rho(s)p$. Substituting these quantities into \eqref{eq:episodic_step1}, we obtain 
    \begin{align*}
        \E[\hat{Q}_1(s,1) - \hat{Q}_1(s,2)] & \geq \frac{1}{3}b_1(s,2) + 3 \sqrt{\frac{LC^\star}{N(\rho(s)/3 + C^\star - 1)}} - 3p \sqrt{\frac{LC^\star}{N\rho(s)p/3}} + g_1(s)\\
        & \geq \frac{1}{3}b_1(s,2) + 3 \sqrt{\frac{LC^\star}{N(\rho(s)/3 + C^\star - 1)}} - 3 \sqrt{\frac{LC^\star}{N\rho(s)/3}} + g_1(s),
    \end{align*}
    where the last inequality uses $p \leq 1$. Observe that 
    \begin{align*}
         \frac{1}{\sqrt{\rho(s)/3}} - \frac{1}{\sqrt{\rho(s)/3 + C^\star - 1}} & = \frac{\sqrt{\rho/3 + C^\star - 1}- \sqrt{\rho/3}}{\sqrt{\rho(s)/3(\rho(s)/3 + C^\star-1)}}
          \leq 3\frac{\sqrt{C^\star-1}}{\rho(s)}.
    \end{align*}
    This implies
    \begin{align}\label{eq:gap_bound}
        \rho(s) g_1(s) \geq 9 \sqrt{\frac{2(C^\star-1)L}{N}} \quad \Rightarrow \quad \E[\hat{Q}_1(s,1) - \hat{Q}_1(s,2)]
        & \geq \frac{1}{3} b_1(s,2).
    \end{align}
    Then, a similar argument to \eqref{eq:level_2_high_prob} proves that $\hat{Q}(s,1) > \hat{Q}(s,2)$ with high probability:
    \begin{align}\label{eq:level_3_high_prob} 
        \prob( \hat{Q}_1(s,2) - \hat{Q}_1(s,1) \geq 0) & \lesssim \frac{1}{N^{160}}.
    \end{align}
\end{itemize}

\paragraph{Sub-optimality bound.} We are now ready to compute the sub-optimality. Decompose the sub-optimality based on whether event $\cE_o$ defined in \eqref{eq:event_Eo} has occurred and use the fact that we assumed $\mu(s,1) \geq 9 \mu(s,2)$ for all $s \in \cS$
\begin{align*}
    \beta_2 & = \E_\cD \left[\sum_{s: \mu(s,\pi^\star(s)) \geq 9 \bar{\mu}(s)} \rho(s) [V^\pi_1(s) - V^{\hat{\pi}}_{1}(s)] \ind \{ \cE_{\text{EMDP}}\} \right]\\
    & \quad \leq \E_{\cD, \rho} \left[ [V^\star(s) - V^{\hat{\pi}}(s)] \ind \{ \cE_o\}\right] + \E_{\cD, \rho} \left[ [V^\star(s) - V^{\hat{\pi}}(s)] \ind \{ \cE^c_o\}\right]\\
    & \quad \lesssim \E_{\cD, \rho} \left[ [V^\star(s) - V^{\hat{\pi}}(s)] \ind \{ \cE_o\}\right] + \frac{3}{N^{160}}.
\end{align*}
Here, the second line is by $\ind \{ \cE_{\text{EMDP}}\} \leq 1$ and the last line follows from $V^\star(s) - V^{\hat \pi}(s) \leq 3$ and the probability of the complement event $\cE^c_o$ given in \eqref{eq:level_2_high_prob}.

Conditioned on the event $\cE_o$, LCB-VI algorithm chooses the optimal action from every state at levels 2 and 3 and hence $V^{\hat{\pi}}_2 = V^\star_2$ and we get
\begin{align*}
    & \E_{\cD, \rho} \left[ [V^\star(s) - V^{\hat{\pi}}(s)] \ind \{ \cE_o\}\right] \\
    & \quad = \sum_{s} \rho(s) \E_\cD[[Q^\star(s,1) - Q^{\hat{\pi}}(s, \hat{\pi}(s))] \ind \{\cE_o\}]\\
    & \quad = \sum_{s} \rho(s) \E_\cD[ r(s,1) + P_{s,1}\cdot V_2^\star - r(s,\hat{\pi}(s)) - P_{s,\hat{\pi}(s)}\cdot V_2^\star ]\\
    & \quad = \sum_{s} \rho(s) \E_\cD[ \left(r(s,1) + P_{s,1}\cdot V_2^\star - r(s,2) - P_{s,2}\cdot V_2^\star\right) \ind \{\hat{\pi}(s) \neq 1\}].
\end{align*}
By definition, we have $g_1(s) = r(s,1) + P_{s,1}\cdot V_2^\star - r(s,2) - P_{s,2}\cdot V_2^\star$. Therefore, 
\begin{align*}
    \mathbb{E}_{\cD, \rho} \left[ V^\star(s) - V^{\hat{\pi}}(s) \ind \{ \cE_o\}\right] & \leq \sum_{s} \rho(s) g(s) \E_\cD[\ind \{\hat{\pi}(s) \neq 1\}]\\
    & = \sum_{s} \rho(s) g(s) \prob(\hat{Q}(s,2) - \hat{Q}(s,1) \geq 0).
\end{align*}
We decompose the sub-optimality based on whether $ \rho(s)g_1(s)$ is large
\begin{align*}
    \E_\cD[J(\pi^\star) - J(\hat{\pi})] & \leq \sum_{s} \rho(s) g(s) \prob(\hat{Q}(s,2) - \hat{Q}(s,1) \geq 0) \ind \left\{  \rho(s)g(s) \leq 9 \sqrt{\frac{2(C^\star-1)L}{N}} \right\} \eqqcolon \tau_1\\
    & \quad + \sum_{s} \rho(s) g_1(s) \prob(\hat{Q}(s,2) - \hat{Q}(s,1) \geq 0) \ind \left\{  \rho(s)g_1(s) > 9 \sqrt{\frac{2(C^\star-1)L}{N}} \right\} \eqqcolon \tau_2\\
    & \quad + \frac{3}{N^{160}}.
\end{align*}
The first term is bounded by
\begin{align*}
    \tau_1 \leq \sum_{s} 9 \sqrt{\frac{2(C^\star-1)L}{N}} = 18 \sqrt{\frac{2(C^\star-1)L}{N}}.
\end{align*}
The second term is bounded using \eqref{eq:level_3_high_prob}
\begin{align*}
    \tau_2 \lesssim \frac{3}{N^{160}}.
\end{align*}
Combining the bounds yields the following sub-optimality bound
\begin{align*}
    \beta_2 \lesssim   \sqrt{\frac{(C^\star-1)L}{N}} + \frac{1}{N^{160}}.
\end{align*}

\paragraph{Proof of inequality \eqref{eq:penalty_difference}.} From $\mu(s,1) \geq 9 \mu(s,2)$, one has $N(s,1) \geq 9 N(s,2)$ implying $b_h(s,2) \geq 3 b_h(s,1)$. Therefore, we conclude that 
\begin{align*}
    b_h(s,2) - b_h(s,1) & = \frac{1}{2}(b_h(s,2) - b_h(s,1)) + \frac{1}{2}(b_h(s,2) - b_h(s,1))
    \geq \frac{1}{3} b_h(s,2) + b_h(s,1).
\end{align*}

\paragraph{Proof of inequality \eqref{eq:bounds_on_b_s_1}.} The concentrability assumption implies the following bound on $\mu(s,1)$
\begin{align*}
    \frac{\bar{d}(s,1)}{C^\star} \leq \mu(s,1) \leq  \frac{\bar{d}(s,1)}{C^\star} + 1 - \frac{1}{C^\star},
\end{align*}
The upper bound is based on the fact that the probability mass of at least $1/C^\star$ is distributed on the optimal actions with a remaining mass of $1 - 1/C^\star$. Applying the above bounds to $b_h(s,1)$, gives
\begin{align*}
    3\sqrt{\frac{LC^\star}{N (\bar{d}(s,1)+C^\star - 1)}} \leq b_h(s,1) = 3\sqrt{\frac{L}{N\mu(s,1)}} \leq 3\sqrt{\frac{L C^\star }{N\bar{d}^\star(s,1)}}.
\end{align*}

\section{Auxiliary lemmas}
This section collects a few auxiliary lemmas that are useful in the analysis of LCB. 

We begin with a simple extension of the conventional Hoeffding bound to the two-sample case.
\begin{lemma} \label{lemma:hoeffding_on_difference_empirical_average}{Let $X_1,\dots, X_n$ be i.i.d.~in range $[0,1]$ with average $\E [X]$ and $Y_1, \dots, Y_m$ be i.i.d.~in range $[0,1]$ with average $\E[ Y]$. Further assume that $\{X_i\}$ and $\{Y_j\}$ are independent. Then for any $\epsilon$ such that $\epsilon + \mathbb{E}[Y] - \mathbb{E}[X] \geq 0$, we have
\begin{align*} 
    \prob \left(\frac{1}{n}\sum_i X_i - \frac{1}{m} \sum_j Y_j > \epsilon \right) \leq \exp \left( - 2 \frac{(mn) (\epsilon + \E [Y ]- \E [X])^2 }{m+n}\right).
\end{align*}}
\end{lemma}
\begin{proof}
It is easily seen that 
\begin{align*}
    & \prob \left(\sum_{i=1}^n m X_i - \sum_{j=1}^m n Y_j > mn \epsilon \right) \\
    = & \prob \left(\sum_{i=1}^n (m X_i - m\E [X]) - \sum_{j=1}^m (n Y_j - \E [Y]) > mn (\epsilon + \E [Y] - \E [X]) \right) \\
    \leq & \exp \left( - 2 \frac{(mn)^2 (\epsilon + \E [Y] - \E [X])^2 }{nm(m+n)}\right)\\
    = & \exp \left( - 2 \frac{(mn) (\epsilon + \E [Y] - \E [X])^2 }{m+n}\right),
\end{align*}
 where the inequality is based on Hoeffding's inequality on independent random variables.
\end{proof}

The next lemma provides useful bounds for the inverse moments of a binomial random variable. 
\begin{lemma}[Bound on binomial inverse moments] \label{lemma:binomial_inverse_moment_bound}  
Let $n \sim \Binomial(N, p)$. For any $k \geq 0$, there exists a constant $c_k$ depending only
on $k$ such that 
\begin{align*}
    \E \Big[ \frac{1}{(n \vee 1)^k} \Big] \leq \frac{c_k}{(Np)^k},
\end{align*}
where $c_k = 1 + k2^{k+1} + k^{k+1} + k \Big(\frac{16(k+1)}{e}\Big)^{k+1}$.
\end{lemma}

\begin{proof}
The proof is adapted from that of Lemma 21 in \citet{jiao2018minimax}. 

To begin with, when $p \leq 1/N$, the statement is clearly true for $c_k = 1$. 
Hence we focus on the case when $p > 1/N$. 
We define a useful helper function $g_N(p)$ to be
\begin{align*}
    g_N(p) \coloneqq \begin{cases}
    \frac{1}{p^k}, \qquad & p \geq \frac{1}{N},\\
    N^k - k N^{k+1} (p - \frac{1}{N}), \qquad & 0 \leq p <  \frac{1}{N}.
    \end{cases}
\end{align*}
Further denote $\hat{p} \coloneqq n/N$. The proof relies heavily on the following decomposition, which is an direct application of the triangle inequality:
\begin{align}
    \E \Big[ \frac{N^k}{(n \vee 1)^k}\Big] & \leq \Big|\E \Big[\frac{N^k}{(n \vee 1)^k} - g_N(\hat{p}) \Big] \Big| + |\E[g_N(p) - g_N(\hat{p})] | + g_N(p). \label{eq:inverse-moment-decomposition}
\end{align}
This motivates us to take a closer look at the helper function $g_{N}(p)$. Simple algebra reveals that 
\begin{align*}
    g_N(p) \leq \frac{1}{p^k} \quad \text{and} \quad g_N(\hat{p}) - \frac{N^k}{(n \vee 1)^k} = k N^{k} \ind \{\hat{p} = 0\}.
\end{align*}
Substitute these two facts back into the decomposition~\eqref{eq:inverse-moment-decomposition} to reach
\begin{align*}
    \E \Big[ \frac{N^k}{(n \vee 1)^k}\Big] & \leq k N^k (1-p)^N + \frac{1}{p^k} + |\E[g_N(p) - g_N(\hat{p})]|.
\end{align*}
It remains to bound the term $|\E[(g_N(p) - g_N(\hat{p}))^2]|$. To this goal, one has
\begin{align*}
    |\E[(g_N(p) - g_N(\hat{p}))^2]| & \leq |\E[(g_N(p) - g_N(\hat{p}))^2 1\{ \hat{p} \geq p/2 \}]| +  |\E[(g_N(p) - g_N(\hat{p}))^2 1\{ \hat{p} \geq p/2 \}]| \\
    & \stackrel{(i)}{\leq} \sup_{\xi \geq p/2} |g_N'(\xi)|^2 \E[(p - \hat{p})^2] + \sup_{\xi > 0} |g_N'(\xi)|^2p^2 \prob(\hat{p} \leq p/2)\\
    & \stackrel{(ii)}{\leq} \frac{k^2}{(p/2)^{2k+2}} \frac{p(1-p)}{N} + k^2 N^{2k+2} p^2 e^{-Np/8}.
\end{align*}
Here the inequality (i) follows from the mean value theorem, and the last one (ii) uses the derivative calculation as well as the tail bound for binomial random variables; see e.g., Exercise 4.7 in \citet{mitzenmacher2017probability}. 
As a result, we conclude that 
\begin{alignat*}{2}
    \E \Big[ \frac{N^k}{(n \vee 1)^k}\Big] & \leq k N^k (1-p)^N + \frac{1}{p^k} + \sqrt{\E[(g_N(p) - g_N(\hat{p}))^2]}\\
    & \leq k N^k (1-p)^N + \frac{1}{p^k} +\frac{k}{(p/2)^{k+1}} \sqrt{\frac{p(1-p)}{N}} + k N^{k+1} p e^{-Np/16}\\
    & \leq k N^k (1-p)^N + \frac{1}{p^k} +\frac{k 2^{k+1}}{p^k} + k N^{k+1} p e^{-Np/16},
\end{alignat*}
where the last inequality holds since $p \geq 1 / N$. 
Consequently, we have 
\begin{alignat*}{2}
    \E \Big[ \frac{(Np)^k}{(n \vee 1)^k}\Big] & \leq  1 + k2^{k+1} + k (Np)^k (1-p)^N + k (Np)^{k+1} e^{-Np/16}.
\end{alignat*}
Note that the following two bounds hold: 
\begin{align*}
    \max_p k (Np)^k (1-p)^N &\leq k \Big(N \frac{k}{N+k}\Big)^k \Big(1-\frac{k}{k+N}\Big)^N
    \leq k^{k+1}, \\
    (Np)^k e^{-Np/16} &\leq \Big(\frac{16k}{e}\Big)^k.
\end{align*}
The proof is now completed. 
\end{proof}

The last lemma, due to Gilbert and Varshamov~\cite{gilbert1952comparison, varshamov1957estimate}, is useful for constructing hard instances 
in various minimax lower bounds. 
\begin{lemma}\label{lem:V_G}
There exists
a subset $\mathcal{V}$ of $\{-1,1\}^{S}$ such that (1)
$|\mathcal{V}|\geq\exp(S/8)$ and (2) for any $v_{i},v_{j}\in\mathcal{V}$,
$v_{i}\neq v_{j}$, one has $\|v_{i}-v_{j}\|_{1}\geq\frac{S}{2}$. 
\end{lemma}

\end{document}